\def\year{2021}\relax
\documentclass[letterpaper]{article} 
\usepackage{aaai21}  
\usepackage{times}  
\usepackage{helvet} 
\usepackage{courier}  
\usepackage[hyphens]{url}  
\usepackage{graphicx} 
\usepackage{natbib}  
\usepackage{caption} 
\usepackage[linesnumbered,ruled ]{algorithm2e}
\usepackage{xcolor}
\usepackage[]{algorithm2e}
\usepackage{algorithmic}
\usepackage{enumitem}
\usepackage{amsfonts}
\usepackage{amsmath}
\usepackage{bbm}
\usepackage{booktabs}
\usepackage{subcaption}
\usepackage[flushleft]{threeparttable}
\usepackage{booktabs,caption}
\title{An Efficient Algorithm for Deep Stochastic Contextual Bandits}
\author {

        Tan Zhu,
        Guannan Liang, 
        Chunjiang Zhu, 
        Haining Li, 
        Jinbo Bi \\
}
\affiliations {
    Department of Computer Science and Engineering, University of Connecticut, Storrs, CT, USA \\
    tan.zhu@uconn.edu, guannan.liang@uconn.edu, chunjiang.zhu@uconn.edu, haining.li@uconn.edu, jinbo.bi@uconn.edu
}

\begin{document}
\maketitle
\begin{abstract}
In stochastic contextual bandit (SCB) problems, an agent selects an action based on certain observed context to maximize the cumulative reward over iterations. Recently there have been a few studies using a deep neural network (DNN) to predict the expected reward for an action, and the DNN is trained by a stochastic gradient based method. However, convergence analysis has been greatly ignored to examine whether and where these methods converge. In this work, we formulate the SCB that uses a DNN reward function as a non-convex stochastic optimization problem, and design a stage-wise stochastic gradient descent algorithm to optimize the problem and determine the action policy. We prove that with high probability, the action sequence chosen by this algorithm converges to a greedy action policy respecting a local optimal reward function. Extensive experiments have been performed to demonstrate the effectiveness and efficiency of the proposed algorithm on multiple real-world datasets.
\end{abstract}
\newtheorem{definition}{Definition}
\newtheorem{assumption}{Assumption}
\newtheorem{theorem}{Theorem}
\newtheorem{fact}{Fact}
\newtheorem{lemma}{Lemma}
\newtheorem{remark}{Remark}
\section*{Introduction}
\noindent 
The multi-armed bandit is an interactive machine learning problem and has received tremendous research interest for decades, e.g., \cite{auer2002finite-time,filippi2010parametric,gopalan2014thompson}. The \emph{stochastic contextual bandit (SCB) problems} contain a sequence of rounds $t$, where the agent observes a context $d_t\in\mathbb{R}^{n_{d}}$ (for $n_d\in I$) which consists of features about all actions and can be regarded as a sample from an unknown distribution that is fixed when shifted in time, e.g., user information in news recommendation system \cite{li2010contextual}. 
Then the agent selects an action $a_t$ from $K$ choices (e.g., $K$-arms) to get the reward $r_{d_t,a_t}\in[0,1]$, which is i.i.d. drawn from an unknown but fixed distribution with the expectation modeled by the \emph{reward function} defined below.
\begin{definition}[Reward Function]
\label{rfd}
There exists a real-valued vector $x^{*}\in \mathbb{R}^{n_{x}}$ and a reward function  $f(d;x^{*})$ s.t. for all actions $a_t\in\{1,2,...,K\}$ and context $d_t\sim\mathcal{D}$,
\begin{equation}
\label{xstar_rf}
\textstyle \mathbb{E}_{r_{d_t,a_t}}[r_{d_t,a_t}\,|\,d_t,a_t]=f_{a_t}(d_t;x^{*}),
\end{equation}
where $f_{a_t}(d_t;x^{*})\in[0,1]$ is the $a_t^{th}$ output of $f(d_t;x^{*}).$
\end{definition} Note that the context $d$ can be continuous but there are discrete and a finite number ($K$) of actions.
The agent's goal is to minimize the cumulative regret
\begin{equation}
\label{cr}
\begin{aligned}
\textstyle \sum_{t=1}^T\mathbb{E}_{r_{d,a}}[r_{d_t,a^*_{d_t}}|d_t,a^*_{d_t}]-r_{d_t,a_t},
\end{aligned}
\end{equation}
where $T$ is the total number of rounds, and $a_{d_t}^*$ is the optimal action that maximizes the expected reward $ \mathbb{E}_{r_{d,a}}[r_{d,a}|d,a]$. 

To minimize the cumulative regret, the agent needs to balance between selecting potentially optimal action to learn its reward distribution (Exploration) and choosing the action with the highest estimated expected reward (Exploitation), a.k.a. the $Exploration\ vs.\ Exploitation$.
Earlier studies on SCB problems focus on SCB with a linear reward function, which has been well understood both theoretically and empirically \cite{auer2002using,dani2008stochastic,foster2018practical,liu2018customized}.
However, the assumption of a linear reward function may not be satisfied for many real-life problems. Later, many research efforts have been devoted to SCB with a more relaxed assumption on the reward function. In particular, SCB with a deep neural network (DNN) modeled reward function or called \emph{deep SCB} have attracted great interest, because of the DNN's powerful expressiveness for complicated real-world problems \cite{allesiardo2014neural,collier2018deep,zhou2020neural,kveton2020randomized}.

For deep SCB, one may naturally adopt \emph{adaptive action policies}, which have been successfully employed in SCB with a simple reward function, e.g., a linear or generalized linear reward function. In such policies, the exploration schedule can be dynamically updated based on the observed contexts, selected actions and received rewards in previous rounds. Then the computational complexity generally depends on 
the dimensionality of the learnable parameters and the contexts, which, however, are often very high for DNNs (because of its power in modeling high-dimensional data). Therefore, it is still not clear how to use adaptive action policies for DNN based on SCB in an efficient way to enable its scalability to large DNNs.
For example, recently \cite{zhou2020neural} designed a powerful adaptive action policy for deep SCB, but the complicated matrix calculations to update the action policy still make it applicable to small DNNs.


Moreover, although almost all deep SCB algorithms train the DNN-realized reward function with stochastic gradient based methods, rigorous analysis has seldom been given to examine whether and where the DNN and the action policy converge.
That is, there is no guarantee that these deep SCB algorithms can learn a stable reward function or effective action policy.  
Recently \cite{sokolov2016stochastic} showed that parameters in the DNN converge to a stationary point, but it remains open to examine the quality of the converged reward function and convergence of the action policy.

{\noindent \bf Contributions.}
We propose a stage-wise stochastic gradient descent (SGD) for SCB (\emph{SSGD-SCB}), which exploits adaptive action policies in an efficient manner and enjoys \emph{provable} convergence on the action policy $\pi$ and reward function $f$.
Our main results are summarized as follows.

\begin{itemize}[leftmargin=-0.03cm]
\item
We design an efficient adaptive action policy for deep SCB by dividing contexts into sub-regions, and counting and exploiting how many times an action has been chosen within each sub-region.
We then devise a stage-wise SGD method to train the DNN.
In this way, the adaptive action policy in SSGD-SCB can be updated incrementally rather than from full history records,
and thus does not require complicated matrix calculations, such as those in \cite{zhou2020neural,kveton2020randomized}, so it can be scalable to large DNNs.

\item
For ease of theoretical analysis, we relax SCB to a non-convex stochastic optimization problem, in which the adaptive action policy makes the variance of the stochastic gradients grow along the rounds. It thus motivates us to stage-wisely reduce the learning rate to account for the growing variance of stochastic gradients. We prove that SSGD-SCB can find a local minimizer under mild regularity conditions, corresponding to a local optimal reward function. The action policy converges to a greedy policy observing the local optimal reward function with high probability.
If a stronger condition is assumed for the objective function, we prove that the action policy converges to one which is within an $\epsilon$-neighborhood of the global optimal action policy.
\item
We have performed extensive experiments to confirm that SSGD-SCB achieves the best cumulative regret compared with state-of-the-art algorithms, and its fitted reward function has the best generalization performance.
\end{itemize}
{\noindent \bf Related Work.}
The agent in SCB can establish and operate on different forms of reward functions,
including linear functions and nonlinear functions represented by DNNs.
The upper confidence bound (UCB) algorithm \cite{auer2002finite-time} was first applied to the contextual bandit problem with a linear reward function in LINREL\cite{auer2002using}. Since then, many efficient ordinary linear algorithms have been proposed \cite{dani2008stochastic,chu2011contextual,agrawal2013thompson,agarwal2014taming,foster2018practical}, and the linear reward functions are also extended to \emph{generalized linear models} (GLMs)  \cite{li2017provably,jun2017scalable,abeille2017linear,kveton2019perturbed}. Some of these algorithms, called \emph{agnostic algorithms} \cite{agarwal2014taming,foster2018practical}, solved the SCB problem by converting it into an optimization problem that can be effectively solved by an existing optimization oracle (solver). However, for SCB problems that correspond to difficult-to-solve optimization problems, new optimization oracles are needed.

Compared with linear and generalized linear models, DNN is more powerful to fit the high dimensional context 
such as images or board configurations of board games.
DNNs have been used to find a low-dimension representation of the raw context for
SCB algorithms with a linear reward function
\cite{liu2018customized,riquelme2018deep}. DNNs have also been used to directly represent the expected reward function to result in deep SCB methods \cite{allesiardo2014neural,collier2018deep,zhou2020neural,kveton2020randomized}.
For deep SCB, there have been static action policies such as the $\epsilon$-greedy policy in NeuralBandit1 \cite{allesiardo2014neural} and noise-added greedy policy in DeepFPL \cite{kveton2020randomized}. In these policies, the calculation of action distribution (as a policy) is independent of the information collected in previous rounds.
Recently \cite{zhou2020neural} proposed an adaptive action policy for deep SCB, but it suffers high computational complexity. Based on full gradients, a convergence analysis was given by presenting an upper bound on the expected cumulative regret. In this paper, we analyze the convergence of the proposed method based on stochastic gradients, which is much more practical than full gradients in the training of DNNs.

{\noindent \bf Organization.}
The remainder of this paper is structured as follows. In the problem setting section, we formulate the SCB into an optimization problem and then propose SSGD-SCB to solve the optimization problem in the next section. After that, theoretical properties of SSGD-SCB are discussed in the section of theoretical analysis. Finally, we present the experimental results in the section of experiments, and conclude the paper with a brief discussion on the future work in the section of conclusions and future work.

\section*{The Problem Formulation}
\label{problemsetting}

In this section, we formulate an upper bound on the expected cumulative regret which can be optimized by stochastic gradient based methods. 
All proofs of lemmas and theorems are deferred to the Appendix A.

{\noindent \bf Notation.} 
We use $\|\cdot\|$ to denote the 2-norm of vectors. The notation $\mathcal{O}(\cdot)$, $\Omega(\cdot)$, $\Theta(\cdot)$ are used to hide constants which do not rely on the setup of the problem parameters and the notation $\tilde{\mathcal{O}}(\cdot)$, $\tilde{\Omega}(\cdot)$, $\tilde\Theta(\cdot)$ are used to additionally hide all logarithmic factors. The operator $\mathbb{E}_x[\cdot]$ represents taking the expectation over the random variable listed in the 
subscript, while $\mbox{VAR}_x[\cdot]$ calculates the corresponding variance, and $\nabla f( \cdot )$ is the full gradient.

We start from defining an objective function $F(x)$ when feedback for all actions can be observed.

\begin{definition}[The Objective Function with Full Feedback]
\label{ofwff}
Suppose that after observing the context $d$, the agent gets all candidate actions' rewards (i.e., the full feedback). 
The objective function is
\begin{equation*}
\begin{aligned}
    \textstyle F(x)=
    \mathbb{E}_{d}\sum_{a=1}^K\mathbb{E}_{r_{d,a}}\big[(f_{a}(d;x)-r_{d,a})^2|d,a\big].
\end{aligned}
\end{equation*} 
\end{definition}
It is calculated as the expected squared error between the estimated action reward and the actual reward $r_{d,a}$. A smaller objective $F(x)$ corresponds to a better estimation of $f(d;x)$. Furthermore, the global minimizer of $F(x)$ corresponds to the $x^*$ used in Definition \ref{rfd}, forming the optimal reward function $f(d;x^{*})$. Lemma \ref{min_star} characterizes this result.

\begin{lemma}   
	\label{min_star}
$x^*$ in Eq. (\ref{xstar_rf}) is a global minimizer of $F(x)$. With $x=x^*$, we have $ F(x^*)=\mathbb{E}_d\sum_{a=1}^K\mbox{VAR}_{r_{d,a}}[r_{d,a}|d,a]$.
\end{lemma}
We now establish the connection between the expected cumulative regret Eq.(\ref{cr}) and $F(x)$ in Lemma \ref{uecr}.
\begin{lemma}
\label{uecr}
Suppose in the $t^{th}$ round, the agent estimates the reward of actions with $f(d_t,x_t)$, where $x_t$ is inferred from the history $\mathcal{H}_{t-1}=\{d_\tau,a_\tau,r_{d_\tau,a_\tau},\xi_\tau|\tau\in\{1,2,...,t-2\}\}$,
where $\xi_\tau$ is the set of all other information the agent observes at round $\tau$, e.g., random variables used in the action policy. Let $a_t^*=\arg\max_a f_a(d_t,x_t)$. The expected cumulative regret in Eq. (\ref{cr}) can be upper bounded by
\begin{equation}
\label{lere}
\begin{aligned}
\textstyle\sum_{t=1}^T Pr(a_t\ne a_t^*)+2\sqrt{K}\sum_{t=1}^T\mathbb{E}_{\mathcal{H}_{t-1}}\sqrt{F(x_t)-F(x^*)}.
\end{aligned}
\end{equation}
\end{lemma}
In Eq. (\ref{lere}), the first term is the regret incurred due to the exploration of the agent's action policy and $r_{d,a^*}-r_{d,a_t}\le 1$ because $r \in [0,1]$. For example, when the $\epsilon$-greedy policy is used, the probability in this term becomes $\epsilon$. The second term represents the error of the agent's estimation on the expected action reward measured by $F(x)$ in Definition \ref{ofwff}. Thus, minimizing Eq. (\ref{cr}) can be relaxed to minimizing its upper bound, or equivalently $F(x)$. Since the agent can only observe the selected action's reward at each round, it does not have the full feedback to compute $F(x)$ during the iterative process and thus can only minimize $F(x)$ with $\mathcal{H}_{t-1}$.

The regret produced by the second term heavily relies on the agent's exploration measured by the first term. For example, if the agent does not conduct much exploration, $F(x_t)$ can not be efficiently decreased with the unbalanced actions and rewards in the history $\mathcal{H}_{t-1}$. To simplify the interaction between these two terms, we focus on a set of action policies in which actions will be selected with a probability that is bounded in terms of $t$, as shown in Assumption \ref{lowerbound}.
\begin{assumption}
\label{lowerbound}
    At $ t\in\{1,2,$ $...,T\}$, for any $d_t\sim\mathcal{D}$ and $\mathcal{H}_{t-1}$, $\pi(a_t\,|\,d_t,\mathcal{H}_{t-1})\ge\Omega(\frac{1}{t^{\mathcal{C}}})$, where $\mathcal{C}>0$. 
\end{assumption}
\begin{remark}
\label{lowerobund_pi}
Assumption 1 is a mild assumption because $\pi\ge 0$ is a probability, and this decreasing lower bound approaches $0$ as the number of rounds increases, and it can find a proper constant in $\Omega$ that does not limit the agent's exploration (the first term in Eq. (\ref{lere})).
\end{remark}

We can then use the \emph{Inverse Propensity Scoring} (IPS) to build 
a corrected stochastic gradient ${\tilde{g}}$, as shown in Lemma \ref{unbiased_g}. Thus, we can use SGD methods to decrease $F(x_t)$ by moving $x_{t-1}$ in the opposite direction of  ${\tilde{g}}$. 

\begin{lemma}[Inverse Propensity Scoring] 
\label{unbiased_g}
At round $ t\in\{1,2,$ $...,T\}$, for any $d_t\sim\mathcal{D}$, and $a_t \in\{1,2,..,,K\}$, the loss for selecting $a_t$ is defined as
$l_{d_t,a_t}(x{;r_{d_t,a_t}})=(f_{a_t}(x;d_t)-r_{d_t,a_t})^2$. Let 
\begin{equation}
\label{unbiased_gradient}
    \textstyle{\tilde{g}}(x;\mathcal{H}_{t})=\frac{1}{\pi(a_t|d_t,\mathcal{H}_{t-1})}\nabla l_{d_t,a_t}(x;r_{d_t,a_t}),
\end{equation}
where $\nabla l$ is the gradient of $l$ with respect to $x$. Under Assumption \ref{lowerbound}, we have that

(1) The stochastic gradient $\tilde{g}$ is an unbiased estimate of the full gradient $\nabla F(x)$, i.e., $\mathbb{E}_{\mathcal{H}_t}{\tilde{g}}(x;\mathcal{H}_{t})=\nabla F(x)$.

(2) 
The element-wise variance of $\tilde{g}$ is upper bounded by $\mathbb{E}_{d_t}\sum_{a=1}^K\mathbb{E}_{r_{d_t,a}}t^{\mathcal{C}}\|\nabla l_{d_t,a}(x;r_{d_t,a})\|^2$.

For notation convenience, we will use $\tilde{g}_{t}(x)$ to represent ${\tilde{g}}(x;\mathcal{H}_{t})$ in the sequel.
\end{lemma}
Lemma \ref{unbiased_g} shows that the IPS corrects the stochastic gradient $\nabla l$ by dividing it by $\pi(a_t\,|\,d_t,\mathcal{H}_{t-1})$ to reach an unbiased estimate of $\nabla F(x)$, but this correction step causes {\em the variance of stochastic gradients to approach infinity} when $\pi$ goes to $0$ at certain $a$.
Specifically, when the SCB shifts from the exploration to exploitation stage, the action policy $\pi$ transits from the initial stage of evenly choosing among actions to later stages of selecting a specific action where the probabilities $\pi$ for other actions approach 0, as shown in Fig. \ref{Fig.1}. However, for SGD to converge, the variance of stochastic gradient needs to be upper bounded. \cite{ge2015escaping,allen2018natasha}. 
When employing SGD methods to solve SCB, although one may use a fast decaying learning rate to resolve the increasing variance, this could result in a very slow convergence rate \cite{johnson2013accelerating}.  Therefore, it is unclear how to find a balance between the decaying speed of the learning rate and the drifting speed of action policy from exploration to being greedy.
\begin{figure}
\label{fig1}
\begin{subfigure}{.23\textwidth}
  \includegraphics[scale=0.5]{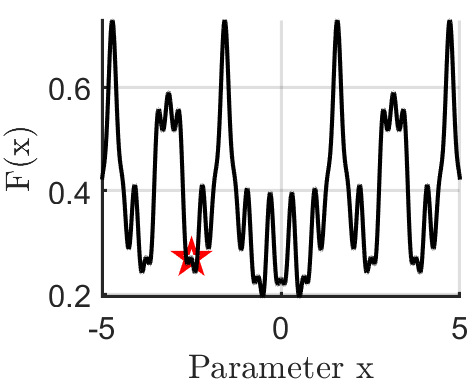}
  \caption{}
  \label{fig:sub-first}
\end{subfigure}
\begin{subfigure}{.23\textwidth}

  \includegraphics[scale=0.5]{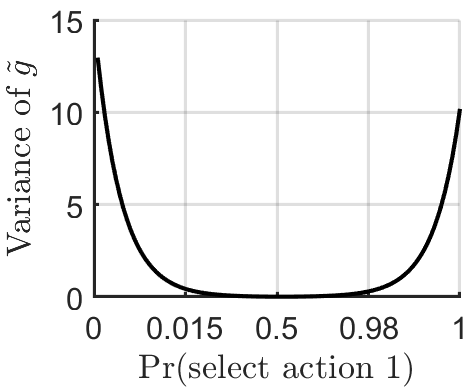}
  \caption{}
  \label{fig:sub-second}
\end{subfigure}
\caption{\label{Fig.1}A simple SCB problem showing the effect of action policy on the variance of IPS corrected stochastic gradient $\tilde{g}$. Here, $\mathcal{D}$ contains two contexts, $d_1=2$ and $d_2=5$, both are uniformly sampled. In each round, the agent can select either $a_1$ or $a_2$. The reward of $r_{d_1,a_1}$, $r_{d_1,a_2}$, $r_{d_2,a_1}$ and $r_{d_2,a_2}$ are randomly drawn from $\mbox{unif}(0.2,0.8)$, $\mbox{unif}(0.55,0.85)$, $\mbox{unif}(0.3,0.9)$ and $\mbox{unif}(0,0.2)$ respectively. The reward functions are $f_{a_1}(d;x)=0.2sin^2(dx)+0.8e^{-(dx)^2}$ and $f_{a_2}(d;x)=0.8sin^2(dx)+0.2e^{-(dx)^2}$, where $x\in\mathbb{R}$ is the learnable parameter. (a) $F(x)$ of this model, and the star marks the position used in Fig. 1(b). (b) The variance of $\tilde{g}$ under different $\pi$ (or Pr(selecting $a_1$)).}
\end{figure}
\section*{The Proposed Method}
\label{method}
In this section, we develop the stage-wise SGD for SCB (\emph{SSGD-SCB}).
Stage-wise SGD has been studied in non-convex optimization and widely used in training DNNs where the step size when moving along a negative stochastic gradient takes multiple values at different stages of the optimization process. 
The method uses a relatively large learning rate in the first stage, and gradually decays it in subsequent stages.
It has been shown that stage-wise SGD helps alleviate the learning-rate-dilemma: a fixed large learning rate is prone to jump out of local minimizer neighborhood while a small learning rate results in very slow convergence \cite{zhu2018anisotropic}.

{\noindent \bf Notations.}
Before describing the SSGD-SCB method, we define some notations.
Let $T_s=T_0s^{2\upsilon}$ be the number of rounds in the $s^{th}$ stage ($s\in\{1,...,\mathcal{S}\}$) 
where $T_0$ is the initial number of rounds and $\upsilon\in[1,\infty)$ is a hyper-parameter controlling the rate of increasing round numbers along stages. Note that we control $T_s$ according to our theoretical analysis in Theorem \ref{converagerate}. Let $I(s,n)$ be the total number of rounds that have been performed at the $n^{th}$ rounds of the $s^{th}$ stage.

In every round, SSGD-SCB performs two major tasks: \emph{action selection} and \emph{backpropagation}. 
\begin{algorithm}[tb]
   \caption{SSGD-SCB}
   \label{PCB-SGD}
\begin{algorithmic}[1]
   \STATE {\bfseries Input:} The initial point $x_1$, the number of rounds $T_0$ in the first stage, the noise scale $\mathcal{N}_0$, the initial learning rate $\eta_0$, and parameters $\omega$ and $C$. According to Remark \ref{hyperparameter_PCBSGD}, we set $\upsilon=1$, ,$\beta=\frac{11}{24}$ and $\kappa=0.5$.
   \FOR{$s=1$ {\bfseries to} $\mathcal{S}$}
   \STATE $T_s=T_0s^{2}$
   \FOR{$n=1$ {\bfseries to} $T_s$}
   \STATE\slash\slash \textbf{Action Selection}
	\STATE The context $d_{I(s,n)}$ is revealed
	\STATE $c_{s,n}=\operatorname*{argmax}_k f_k(d_{I(s,n)};x_{I(s,n)})$
	\FOR{$k=1$ {\bfseries to} $K$}
	\STATE	Calculate $\pi(k|d_{I(s,n)},\mathcal{H}_{I(s,n)-1})$
	\ENDFOR
	\STATE$a_{I(s,n)}\sim \pi(\cdot|d_{I(s,n)},\mathcal{H}_{I(s,n)-1})$\\
	\STATE\slash\slash \textbf{Backpropagation}\\
	\STATE$r_{d_{I(s,n)},a_{I(s,n)}}$ is revealed\\
    \STATE$N_{I(s,n)+1,a,c_{s,n}}=N_{I(s,n),a,c_{s,n}}+1$\\
    \STATE Update $x_{I(s,n)+1}$ using Eq. (\ref{eq:update_x})
   \ENDFOR
   \ENDFOR
\end{algorithmic}
\end{algorithm}

{\noindent \bf Action Selection.}  
Here by dividing observed contexts into clusters, 
we develop an efficient action selection policy which can adapt the exploitation scheme based on how often an action is visited in the cluster which the current context belongs to. 
As a warm-up, we first present a simple adaptive action policy below:
\begin{equation}
\label{p_part2_1}
\begin{aligned}
    &\textstyle W_{s,n,a}=\left\{
		\begin{aligned}
		& \textstyle U_{s,n,a},& a = \arg\max_{k}U_{s,n,k};\\
		& \textstyle \frac{U_{s,n,a}}{s^{\omega}},&else,
		\end{aligned}
		\right.\\
    &\textstyle \pi_0(a\,|\,d_{I(s,n)},\mathcal{H}_{I(s,n)-1})
    =\frac{0.05}{Ks^{0.5\kappa}}
    +\frac{(1-\frac{0.05}{s^{0.5\kappa}})W_{s,n,a}}{\sum_{k=1}^K W_{s,n,k}},
\end{aligned}
\end{equation}
where $U_{s,n,a}=f_a(d_{I(s,n)};x_{I(s,n)})$, and $\omega\in(\frac{\kappa}{2},\infty)$ and $\kappa\in(0,\upsilon)$ are tuning parameters. The policy $\pi_0$ is a convex combination of the exploration probability $\frac{1}{K}$ and the exploitation probability $W_{s,n,a}/(\sum_k W_{s,n,k})$ with a combination coefficient of $0.05s^{-0.5\kappa}$. By this definition, $\pi_0$ has a lower bound $\frac{0.05}{Ks^{0.5\kappa}}$ for any action and this lower bound decreases along stages as controlled by $\kappa$ which is a constant. It is easy to check that $\pi_0$ satisfies Assumption \ref{lowerbound}. 

When $s$ increases, $\pi_0$ will drift from a more even chance of choosing actions to a more greedy policy which selects the action with the largest estimated reward. To shift to choose the action with the largest estimated reward, the parameter $\omega$ controls the weight decay of other actions. Thus, 
%
$\pi_0$ can dynamically balance the exploration and exploitation. However, if $U_{s,n,a}=f_a(d_{I(s,n)};x_{I(s,n)})$, it does not utilize the history of action selection in previous rounds.

Inspired by the effective multi-armed bandit algorithm - UCB1 \cite{auer2002finite} that balances exploration and exploitation via the visit number of actions, we form a new exploitation term incorporating the history action visits.
We first group observed contexts according to the action that gives the largest reward estimated at the respective round:
\begin{equation}
\label{cluster}
    \textstyle c=\operatorname*{argmax}_k f_k(d_{t};x_{t}),
\end{equation}
so $d_t$ belongs to cluster $c$ (as one of the $K$ clusters). Thus, for the current context $d_{I(s,n)}$, which belongs to cluster $c_{s,n}$, we count the number of times, $N_{I(s,n),a,c_{s,n}}$, an action $a$ has been visited before the round $I(s,n)$.
We can revise $U_{s,n,a}$ to incorporate $N_{I(s,n),a,c_{s,n}}$ according to Eq. (\ref{UCB}). Note that the bookkeeping of $c_{s,n}$ and $N_{I(s,n),a,c_{s,n}}$ is performed iteratively.
\begin{equation}
\label{UCB}
   \textstyle U_{s,n,a}=f_a(d_{I(s,n)};x_{I(s,n)})+C\frac{(\sum_{k=1}^KN_{I(s,n),k,c_{s,n}})^{\beta}}{\sqrt{N_{I(s,n),a,c_{s,n}}}},
\end{equation}
where $\beta\in(0,0.5)$ and $C\in[0,\infty)$ are hyperparameters weighting the action-visit term. Eq. (\ref{UCB}) provides a mechanism to adapt the exploitation before becoming completely greedy based only on $f_a(d_{I(s,n)};x_{I(s,n)})$. The actions with a smaller visit number will have a greater chance to be explored. At last, we calculate the final action policy $\pi(\cdot|d_{I(s,n)},\mathcal{H}_{I(s,n)-1})$ with the new $U_{s,n,a}$ in Eq. (\ref{p_part2_1}). Algorithm \ref{PCB-SGD} summarizes all the steps. If there exists an action that has been visited less than a threshold, frequently visited actions will only be selected with low probability until the visit numbers of all actions pass the threshold.
Formal descriptions are given in Lemma \ref{lowerbound_pi} (see Appendix A for the proof). 


\begin{lemma} 
\label{lowerbound_pi}
For a cluster $c_{s,n}$, if there exists an action with its visit number smaller than $\frac{C^2(\sum_{k=1}^KN_{I(s,n),k,c_{s,n}})^{2\beta}}{ C^2K+2C\sqrt{K}+1}$, the action with the largest visit number will be selected by $\pi$ with probability $\mathcal{O}(s^{-0.5\kappa})$, while the action with the largest $U_{s,n,a}$ will be selected with probability greater than $\frac{(1-0.05s^{-0.5\kappa})s^\omega}{s^\omega+K-1}$. 
\end{lemma}

{\noindent \bf Backpropagation.}
In the $n^{th}$ round of $s^{th}$ stage, after taking the action $a_{I(s,n)}$, the agent observes the corresponding reward $r_{d_{I(s,n)},a_{I(s,n)}}$. SSGD-SCB then updates its parameters as follows. First, the visit count of $a_{I(s,n)}$ is increased by one.
Then 
the stochastic gradient $\tilde{g}_{I(s,n)}(x_{I(s,n)})$ is computed based on the context $d_{I(s,n)}$,  the reward $r_{d_{I(s,n)},a_{I(s,n)}}$, and $\pi(a|d_{I(s,n)}, \mathcal{H}_{I(s,n)-1})$ in Eq. (\ref{unbiased_gradient}).
Finally,the parameters $x_{I(s,n)}$ can be updated in a SGD step:
\begin{equation}\label{eq:update_x}
    \textstyle x_{I(s,n)+1}=x_{I(s,n)}-\eta_s(\tilde{g}_{I(s,n)}(x_{I(s,n)})+\mathcal{N}_s),
\end{equation}
where $\eta_s=\frac{\eta_0}{s^{\upsilon}}$ is the learning rate decaying from initial $\eta_0$, and $\mathcal{N}_{s}=\frac{s^{\frac{\kappa}{2}}\mathcal{N}_0w}{\|w\|}$ is a noise term 
where $w$ is a vector of noise drawn from standard Gaussian and $\mathcal{N}_0$ is a hyper-parameter controlling the magnitude of noise. Since $\kappa < \upsilon$ (see Eq.(\ref{p_part2_1})), $\eta$ decays faster than the growth of the variance of stochastic gradients over rounds.
Because our optimization problem is non-convex, there can be saddle points (i.e., $\nabla F(x)=0$, but $x$ is not a local extremum). Prior work \cite{ge2015escaping} has informed that adding proper noise to the stochastic gradient can help escape from strict saddles. The noise needs to be small enough without incurring a sharp increase in the variance of gradients, which we prove that the variance caused by this noise is smaller than that due to dividing by $\pi$. 
The noise $\mathcal{N}_{s}=\Theta(s^{\frac{\kappa}{2}})$ increases along stages because if a fixed small noise is used, it will be hard in later stages to escape strict saddles when $\eta_s$ is very small.

\section*{Theoretical Analysis}
\label{ta}
When the reward function $f$ is represented by a DNN, the objective $F(x)$ is non-convex. It is NP-hard to find a global minimizer of a non-convex problem or even to approximate a global minimizer \cite{jain2017non}, especially when we cannot assume properties such as PL-condition \cite{yuan2019stagewise} or $\sigma$-Nice condition \cite{hazan2016graduated} (for which an approximate optimal solution can be found in polynomial time). Under mild conditions, the SGD method can be proved to reach just a stationary point for non-convex optimization. 

Given that SCB is not only non-convex but also a stochastic optimization process, it becomes even more challenging to analyze convergence, so it has hardly seen any prior work discussing where  SCB algorithms converge. Under four assumptions (see below), we prove that SSGD-SCB can find a local minimizer of $F(x)$, corresponding to a suboptimal reward function $f$ and
a sub-linear expected local cumulative regret. Before introducing the concept of local cumulative regret, we first give our Assumption \ref{non_convex_PCB} on $F(x)$. Then we define the local optimal convergence used in the subsequent analysis in Definition \ref{local optimal convergence}.

\begin{definition} [$(\alpha,\gamma,\epsilon,\delta)$-Strict Saddle]
    \label{ssp}
    A twice differentiable function $\mathcal{F}(x)$ is $(\alpha,\gamma,\epsilon,\delta)$-strict saddle, if for any point $x$, at least one of the following conditions is true:
    \newline
    1. $\|\nabla \mathcal{F}(x)\|\ge \epsilon$.
    \newline
    2. The minimum eigenvalue of $\nabla^2\mathcal{F}(x)$ is smaller than a fixed negative real number $\gamma$.
    \newline
    3. There is a local minimizer $x^l$ such that $\|x-x^l\|\le\delta$, and the function $\mathcal{F}(x')$ restricted to $2\delta$ neighborhood of $x^l$ ($\|x'-x^l\|\le 2\delta$) is $\alpha$-strongly convex (see Appendix for definition).
\end{definition}
\begin{assumption}[\cite{ge2015escaping}]
\label{non_convex_PCB}
The objective function with full feedback, $F(x)$, is an L-smooth, bounded, $(\alpha,\gamma,\epsilon,\delta)$-strict saddle function, and has $\rho$-Lipschitz Hessian (see Appendix for definition).
\end{assumption}
\begin{definition}
\label{local optimal convergence}
    Let $t'$ be a fixed round number. The event of local optimal convergence can be defined as $\mathfrak{E}^{t'}=\{$With $\mathcal{L}>0$, for any round $t>t'$ and any history $\mathcal{H}_{t'}$,  there exists a local minimum of $F(x)$ marked as $x^{\mathcal{H}_{t'}}$ such that $\|x_{t}-x^{\mathcal{H}_{t'}}\|\le\tilde{\mathcal{O}}(\frac{1}{t^{\mathcal{L}}})\le\delta\}$.
\end{definition}

For a deterministic optimization problem, one may argue that a method converges to any local minimum. For SCB, Definition \ref{local optimal convergence} is more strict because it does not allow $x$ to jump among different local minimums over the stochastic process after a certain number of rounds is reached. Because of the stochasticity of stochastic gradients, this is  difficult to guarantee \cite{kleinberg2018alternative}. Conditioned on $\mathfrak{E}^{t'}$, $\sqrt{F(x_t)-F(x^*)}$ in Eq. (\ref{lere}) can be upper bounded by $\sqrt{|F(x_t)-F(x^{\mathcal{H}_{t'}})|}+\sqrt{F(x^{\mathcal{H}_{t'}})-F(x^*)}$. The second term of this upper bound measures the regret produced by the difference between the loss of local and global minimizers. This difference is inaccessible to the agent, and thus we focus on examining the local cumulative regret:
\begin{equation*}
\label{local regret}
\begin{aligned}
\textstyle \sum_{t=1}^T \big[Pr(a_t\ne a_t^*)+\mathbb{E}_{\mathcal{H}_{t-1}}2\sqrt{K|F(x_t)-F(x^{\mathcal{H}_{t'}})|}\big].
\end{aligned}
\end{equation*}
The local cumulative regret measures the transition speed of the action policy from exploration to exploitation (the first term) and the convergence rate of the current loss to the local optimal loss (the second term). With Assumptions \ref{non_convex_PCB}-\ref{boundedslope}, conditioned on the event $\mathfrak{E}^{t'}$, we can show that if the local cumulative regret is sub-linear, the action policy selected by SSGD-SCB approaches (in expectation) to a policy defined by the local optimal reward function  over rounds. This is obtained by discussing the expected mismatching rate defined in Definition \ref{def:local}. We show that the mismatching rate has a  zero-approaching upper bound
, as shown in Lemma \ref{lowerbound_local_optimal} (see Appendix B for the proof). As an extension, if $F(x)$ is convex, SSGD-SCB can reach a sub-linear expected cumulative regret by achieving a zero-approaching expected mismatching rate with the optimal policy, as discussed in Remark \ref{convex_condition}.

\begin{assumption}
\label{banditproof}
Denote the set of local minimizers of $F(x)$ as $X$. For any $ x^l\in X$, given any sampled context $d_t$, the largest output in $f(d_t;x^l)$ is unique. In other words, if $a^{x^l}_{d_t}=\arg\max_{k\in\{1,2,...,K\}}f_k(d_t;x^l)$, then $f_a(d_t;x^l)\ne f_{a^{x^l}_{d_t}}(d_t;x^l)$ for any $a \neq a_{d_t}^{x^l}$.
\end{assumption}

\begin{remark}
Assumption \ref{banditproof} is a variant of the Assumption 3 used in a Thompson Sampling based contextual bandit algorithm \cite{gopalan2014thompson} which assumes a convex loss function and the largest output in $f(d_t;x^*)$ (where $x^*$ is a global optimizer) is unique. We deal with a non-convex loss so we assume that the largest reward output for a local minimizer is unique. 
\end{remark}

\begin{assumption}[\cite{lu2010contextual}]
	\label{boundedslope}
	The gradient of $f$ is bounded. In other words, 
	$\exists M_f\in\mathbb{R}_*^+$ such that $\forall d\sim\mathcal{D}$, $\forall x\sim\mathcal{R}^{n_{x}}$ and $\forall a\in\{1,2,...,K\}$, $\|\nabla f_a(x;d)\|\le \sqrt{M_f}$.
\end{assumption}

\begin{definition}[Expected Mismatching Rate] \label{def:local}
Let $X$ be the set of local minimizer of a non-convex $F(x)$, the expected local optimal mismatching rate is
\begin{equation*}
\textstyle \mathbb{E}_{\mathcal{H}_T}\big[\min_{x^l\in X}(\sum_{t=1}^{T}\frac{\mathbf {1}(a_{d_t}\ne a_{d_t}^{x^l})}{T})\big],
\end{equation*}
where $T$ is the total number of rounds, and $a_{d_t}^{x^l}=\arg\max_k(f_k(d_t;x^l))$.
\end{definition}
\begin{lemma}
\label{lowerbound_local_optimal}
With Assumptions \ref{non_convex_PCB}-\ref{boundedslope}, conditioned on the event $\mathfrak{E}^{t'}$, if the expected local regret is sub-linear, the expected mismatching rate defined in Definition \ref{def:local} has a zero-approaching upper bound.
\end{lemma}
\begin{remark}
\label{convex_condition}
If $F(x)$ is a convex function, the expected mismatching rate can be defined for a global minimizer of $F(x)$. The mismatching rate is the upper bound of expected cumulative regret averaged over $T$.
\end{remark}

{\noindent \bf SSGD-SCB for Non-Convex Objectives.}
We analyze the convergence of SSGD-SCB in two steps. First, based on Assumptions \ref{non_convex_PCB} and \ref{Noise_SGD}, 
Theorem \ref{converagerate} characterizes that if the stage number is large enough, $x$ can reach and be captured at a local minimizer with high probability.
We then prove convergence of the action policy by examining the mismatching rate
in Theorem \ref{pcbsgdconclusion}.  All proofs are provided in Appendix B. 

\begin{assumption}
\label{Noise_SGD}
for any context d, action a and solution $x$, $\exists M\in\mathbb{R}_*^+$, $\|\nabla l_{d,a}(x)\|^2\le M^2$.
\end{assumption}

\begin{theorem}
	\label{converagerate}
	Under Assumptions \ref{non_convex_PCB} and \ref{Noise_SGD}, 
	for any $\zeta\in(0,1)$, 
	there exists a large enough stage $s'$
	such that with probability at least $1-\zeta$, 
	for $\forall s>s'$ and $\forall n\in \{1, 2,...,T_0 s^{2\upsilon}\}$, we have $\|x_{I(s,n)}-x^{s'}\|^{2}\le\mathcal{O}( s^{\kappa-\upsilon}\log \frac{1}{\eta_{s}\zeta})\le\delta$, where $x^{s'}$ is a local minimizer of $F(x)$.
\end{theorem}
With $\mathfrak{E}_{t'}$ defined in Definition \ref{local optimal convergence}, Theorem \ref{converagerate} shows that $Pr(\mathfrak{E}^{I(s',T_{s'})})\ge 1-\zeta$. With Assumptions 2-\ref{Noise_SGD}, 
conditioned on $\mathfrak{E}^{I(s',T_{s'})}$, we show that the mismatching rate between the actions taken by SSGD-SCB and the policy defined by a local minimizer $x^l$ will converge to 0 as $T\rightarrow\infty$.


\begin{theorem}
\label{pcbsgdconclusion}
    Under Assumptions 2-5, if $\mathcal{S}$ is large enough,
    with probability at least $1-\zeta$,
    the expected mismatching rate is bounded as follows:
	\begin{equation*}
	    \begin{aligned}
	    &\textstyle \mathbb{E}
	    \big[\min_{x^l\in X}(\sum_{s=1}^{\mathcal{S}}\sum_{n=1}^{T_0s^2}\frac{\mathbf {1}(a_{I(s,n)}\ne a^{x^l}_{d_{I(s,n)}})}{\sum_{s=1}^\mathcal{S}T_0s^2})\big]
        \le\tilde{\mathcal{O}}(\frac{1}{\mathcal{S}^\Lambda}),
	    \end{aligned}
	\end{equation*}
	where $\Lambda=\min(\frac{\kappa}{2},(1-2\beta)(1+2\upsilon),(2\beta-1)(1+2\upsilon)+\kappa-\upsilon)$.
\end{theorem}
\begin{remark}
\label{hyperparameter_PCBSGD}
By setting $\upsilon=1$ and assuming $\frac{\kappa}{2}=(1-2\beta)(1+2\upsilon)=(2\beta-1)(1+2\upsilon)+\kappa-\upsilon$, we can set $\kappa=0.5$, and then $\beta=\frac{11}{24}$ in an implementation of Algorithm 1. Then we have 
$\Lambda=0.25$. Note that when the exploration parameter $C$ in  Eq. (\ref{UCB}) equals 0, the upper bound is $\mathcal{O}(\frac{1}{\mathcal{S}^{0.25}})$. 
\end{remark}
\begin{remark}
With an additional structural assumption on $F(x)$, i.e. $\sigma$-nice, 
a variant of SSGD-SCB can achieve a zero approaching expected mismatching rate, in which $X$ is a set of $\epsilon$-optimal solution ($F(x^l)-F(x^*)\le\epsilon$). Besides, if $F(x)$ is a strongly convex function, SSGD-SCB can achieve a sub-linear expected cumulative regret. These two cases are discussed in Appendices E and F. 
\end{remark}
\section*{Experiments}
\label{experiments}
We have performed extensive experiments to confirm the effectiveness and computational efficiency of the proposed method, SSGD-SCB with a DNN reward function. Another set of experiments have also been conducted for SSGD-SCB with a linear reward function. 
\begin{figure*}[!ht]
   \begin{minipage}{1\textwidth}
     \centering
     \subcaptionbox{ACR on CIFAR-10\label{fig3:a}}{\includegraphics[width=2.05in]{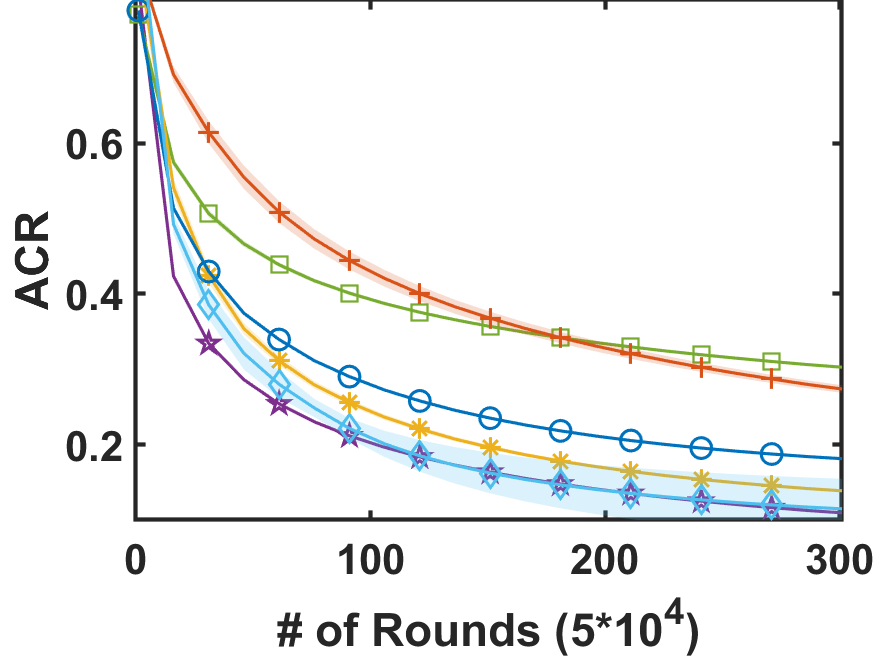}}
    \subcaptionbox{Out-of-sample MSE on CIFAR-10\label{fig3:b}}{\includegraphics[width=2.05in]{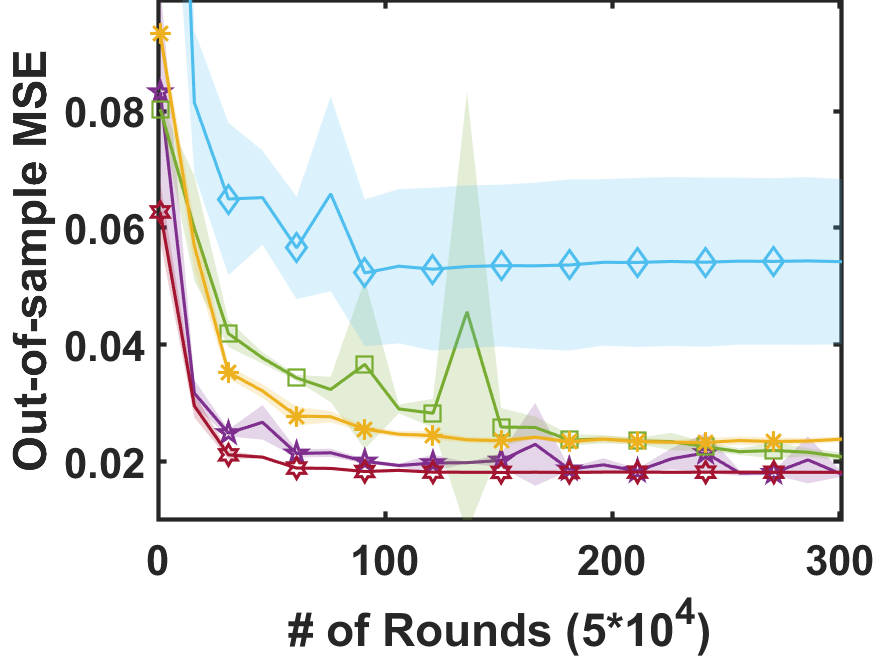}}
    \subcaptionbox{Top-1 test accuracy on CIFAR-10\label{fig3:c}}{\includegraphics[width=2.05in]{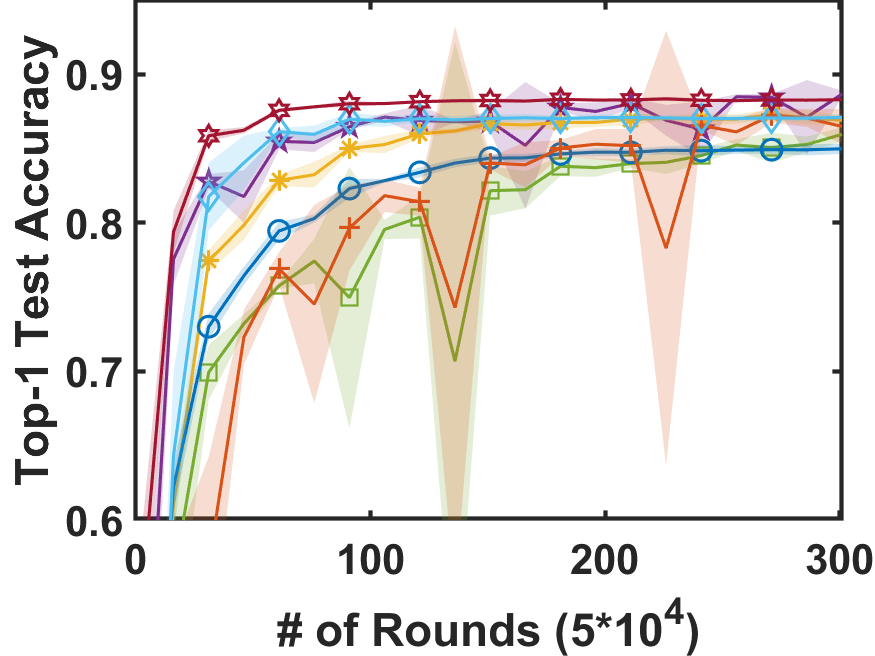}}%
   \end{minipage}
   \begin{minipage}{1\textwidth}
     \centering
     \subcaptionbox{ACR on CIFAR-10+N\label{fig3:d}}{\includegraphics[width=2.05in]{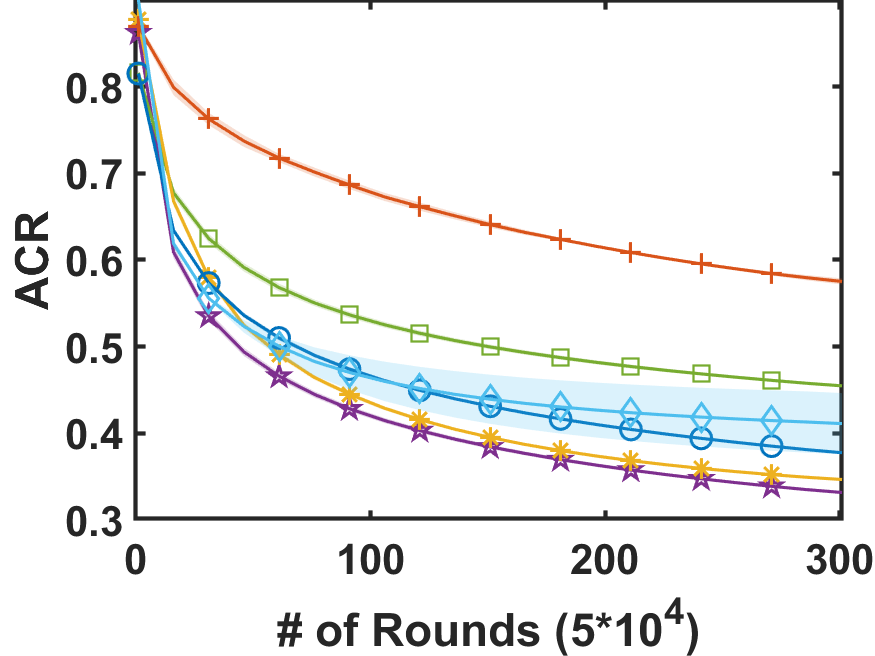}}
    \subcaptionbox{Out-of-sample MSE on CIFAR-10+N\label{fig3:e}}{\includegraphics[width=2.05in]{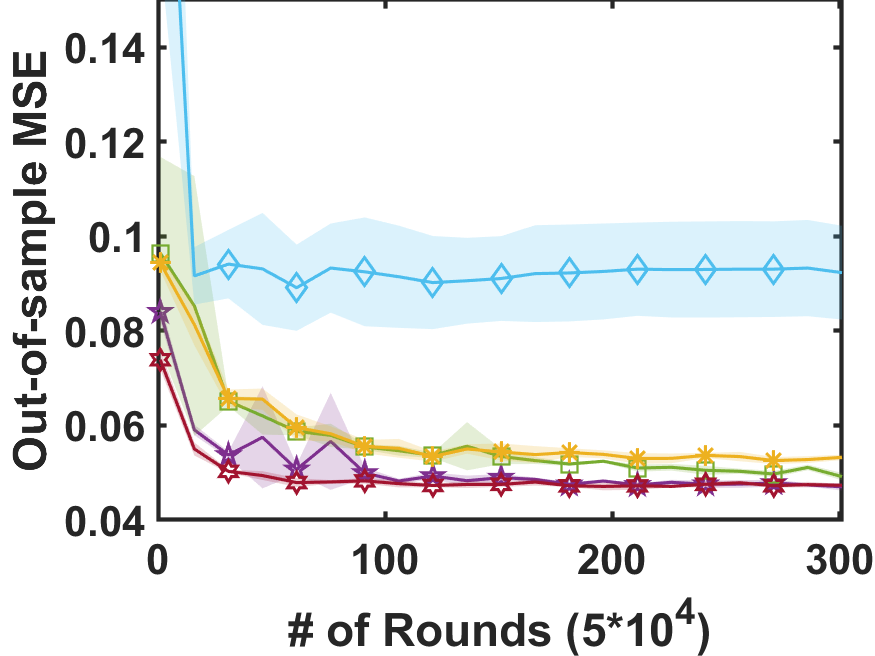}}
    \subcaptionbox{Top-1 test accuracy on CIFAR-10+N\label{fig3:f}}{\includegraphics[width=2.05in]{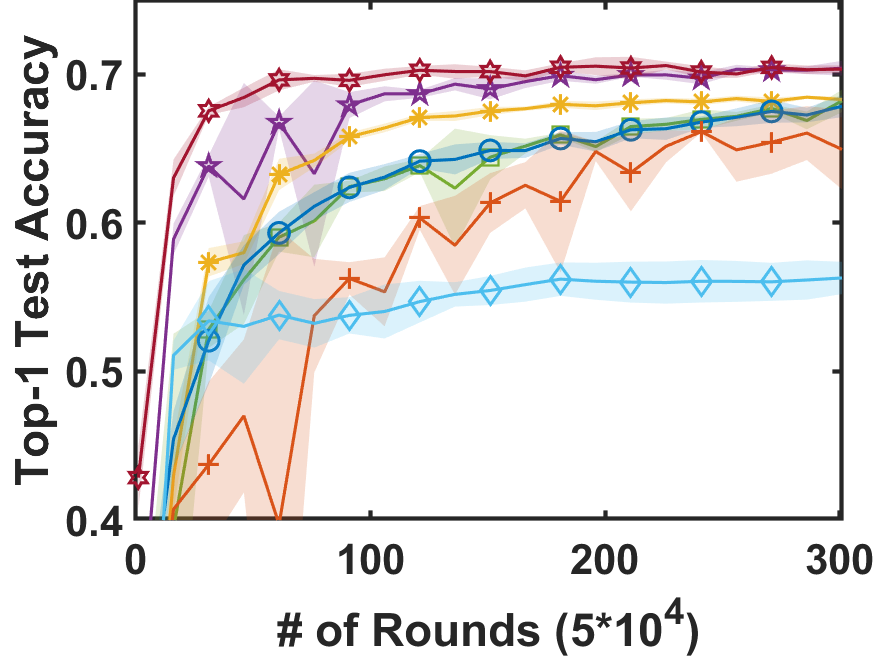}}
   \end{minipage}
   \begin{minipage}{1\textwidth}
     \centering
      \includegraphics[width=3.5in]{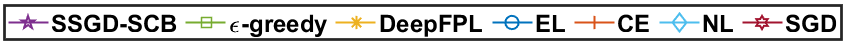}
    \caption{Average cumulative regret, out-of-sample MSE and top-1 test accuracy of SSGD-SCB on CIFAR-10 and CIFAR-10+N. EL and CE have no out-of-sample MSE because rather than fitting rewards of actions, they learn action policy directly. }
   \end{minipage}
   \vspace{-0.2in}
\end{figure*}

{\noindent \bf Datasets.} We use the CIFAR-10 dataset \cite{simonyan2014very}, which has been widely used for benchmarking non-convex optimization algorithms.
It consists of $60K$ $32\times 32$ RGB images from $10$ classes. Considering that it does not include noise in the image label,  we construct a noise-added version of CIFAR-10, called CIFAR-10+N, by correctly marking the labels of 80\% samples while randomly marking the labels of the rest 20\% samples.
For both CIFAR-10 and CIFAR-10+N data, $50K$ samples are selected as the training set $D_{train}$ while the rest $10K$ samples form the test set $D_{test}$. 

{\noindent \bf Experimental Methods.} At each round $t$, an image is randomly sampled from the training set followed by an augmentation with random crop (4$\times$4 region) and horizontal flip to increase its diversity, and then used as the context $d_t$. Next, the agent selects an index of class as the selected action $a_t$, and then the label corresponding to $a_t$ is revealed as the reward $r_{d_t,a_t}$. We denote the $i^{th}$ sample as $(d_i,l_i)$, where $d_i$ is the $i^{th}$ image and $l_i=[l_{i,1},l_{i,2},...,l_{i,K}]^T$ is its one-hot encoded label. For simplicity, we use the \emph{average cumulative regret} (ACR), $1-\frac{\sum_{t=1}^T r_{d_t,a_t}}{T}$, as an estimation of the expected cumulative regret. 
With the test set, we measure the generalization performance of the reward function on the reward of the optimal action and the rewards of all candidate actions, respectively by \emph{top-1 test accuracy}, $\sum_{(d_i,l_i)\in \mathcal{D}_{test}}\frac{l_{i,\arg\max_k f_k(d_i;x_t)}}{|D_{test}|}$ and the \emph{out-of-sample min-squared error (MSE)}, $\sum_{(d_i,l_i)\in D_{test}}\frac{\|f(d_i;x_t)-l_i\|^2}{K|D_{test}|}$.

{\noindent \bf Baseline Algorithms.}
We compared SSGD-SCB with a DNN realized reward function against five baseline algorithms: (1) $\epsilon$-greedy policy, a basic SCB algorithm where the agent selects actions, either greedily with probability $1-\epsilon$ or randomly with probability $\epsilon$; (2 and 3) Bandit structured prediction with Expected Loss (EL) and Cross-Entropy (CE) minimizations \cite{sokolov2016learning}, originally
used in the field of interactive Natural Language Processing; (4) Neural Linear (NL) \cite{riquelme2018deep}, which learns the action policy via the linear Thompson sampling algorithm using the representations of the last hidden layer of a DNN as input; and (5) DeepFPL \cite{kveton2020randomized}, a recent algorithm using multiple DNNs to model the reward function and determine action policy.
In addition, we compare the generalization performance of SSGD-SCB with a supervised learning method, in which all actions' rewards are available to the agent in each round.
The reward functions of the above algorithms are modeled by a variant of VGG-11 with batch normalization, which contains $9$ weight layers and $\sim9.2$ million learnable parameters (see Appendix C for the detailed structure and the hyper-parameter settings). We implement all the algorithms in PyTorch and test on a server equipped with Intel Xeon Gold 6150 2.7GHz CPU, 192GB RAM, and an NVIDIA Tesla V100 GPU.
 
{\noindent \bf Results.} In the simulation of $15$ million rounds, we record the average accumulate regret, the out-of-sample MSE and top-1 test accuracy in every $750K$ round, and compare the results averaged over six runs.
As shown in Fig. 2(a,d), SSGD-SCB consistently attains the lowest average cumulative regret amongst all tested algorithms, demonstrating the empirical effectiveness. 
Furthermore, the reward function of SSGD-SCB has the best generalization capability in terms of the out-of-sample MSE and top-1 test accuracy amongst all tested SCB algorithms, as shown in Fig. 2(b,c,e,f). 
Compared to the DNN under supervised learning, SSGD-SCB results in similar out-of-sample MSE and top-1 test accuracy, but uses only the selected action's reward instead of all actions' rewards.
We observe that NL has a good performance in predicting the reward of the optimal action (top-1 test accuracy) but fails to estimate other actions' rewards (the out-of-sample MSE) in Fig. 2(b,c). This might be due to it uses limited exploration in the large-scale DNN tested in our experiment. 


Finally, Table \ref{tab:test} summarizes the running time of the three top-performed algorithms, SSGD-SCB, DeepFPL and NL, in Fig. 2. SSGD-SCB consumes the smallest running time, consistent with our expectation on its computational efficiency. For comparison, both training multiple DNNs at the same time in DeepFPL and updating Bayesian estimates in NL are computationally expensive. Note that for DeepFPL, 3 DNNs are used in the CIFAR-10 data while 5 DNNs are used in the CIFAR-10-N data.
\begin{table}                                          
\center
\small
\setlength\tabcolsep{3pt}
 \begin{tabular}{lccc} 
  \toprule 
  Dataset& SSGD-SCB&NL&$\mbox{DeepFPL}$\\ 
  \midrule 
 CIFAR-10&$15.25\pm{0.08}$&$35.92\pm{0.11}$&$21.52\pm{0.05}$\\ 
 CIFAR-10+N&$15.25\pm{0.08}$&$36.02\pm{0.04}$&$33.89\pm{0.11}$\\
  \bottomrule 
 \end{tabular} 
\caption{\label{tab:test}Elapsed time (in seconds for $5\times10^4$ rounds) for SSGD-SCB, NL and DeepFPL.} 
\end{table}

{\noindent \bf SSGD-SCB with a Linear Reward Function.}
In this experiment, we use the CIFAR-10, Microsoft MSLR-WEB30k (MSLR) and Yahoo! Learning to Rank Challenge V2.0 (Yahoo) \cite{foster2018practical} as the datasets to evaluate SSGD-SCB and five baseline algorithms, including One Against All (OAA) (supervised learning), $\epsilon$-greedy, Greedy, Bagging, and Cover in the Vowpal Wabbit system \cite{langford2011vowpal}.
To estimate the cumulative regret, we use the \emph{progressive validation loss} (PVL), which has been widely adopted in the linear contextual bandits literature \cite{agarwal2014taming,bietti2018contextual}.
%
As shown in Fig. 2(b) in Appendix C, SSGD-SCB achieves much lower PVL than the baseline algorithms, and even comparable to the supervised learning in the Yahoo data. In the other datasets, we do not observe a clear trend on the curves of different algorithms.
More details can be found in Appendix C.
\section*{Conclusions and Future Work}
\label{futurework}
In this work, we propose a stage-wise SGD algorithm for deep neural network realized SCB problems and solve the high computational complexity incurred by using an adaptive action policy. Based on stochastic gradients, we provide the first theoretical analysis of the convergence on the action policy and reward function.
Extensive experiments have been performed to demonstrate that the proposed algorithm achieves both better generalization performance and a lower cumulative regret compared to the state-of-the-art baseline algorithms.
In the future, we will study how to improve the optimization oracles, e.g., by replacing SGD with more advanced optimizers. It will be also interesting to investigate how to add more action-specific layers in the deep neural network to improve estimations of the actions' rewards. Finally, rather than clustering contexts with the estimated rewards, more complicated clustering approaches may be employed to improve the performance of exploration.
\section*{Acknowledgments}
This work was funded by NIH grant 1K02DA043063-01 and NSF grant IIS-1718738 to Jinbo Bi.
\section*{Ethics Statement}

\begin{itemize}
    \item[a)] Who may benefit from this research? 
    \item This research proposes a stage-wise SGD algorithm for SCB problems with deep neural network (DNN) modeled reward function, a.k.a deep SCB problems. We proposed an efficient adaptive action policy for deep SCB problems to select action, which can dynamically schedule the exploration according to the observed history information and be incrementally updated. Besides, we first prove the convergence of action policies and reward functions of deep SCB problem when the DNN modeling the reward function is trained by a stochastic gradient based method. Because of the effectiveness and computational efficiency confirmed in our experiment, people can use our method to handle large scale real-world SCB problems. Therefore, we believe both academic researchers and industry would benefit from our work.
    
    \item[b)] Who may be put at a disadvantage from this research?
    \item We do not think that our work will put people at a disadvantage.
\end{itemize}
\begin{quote}
\bibliography{bible.bib}
\end{quote}

\end{document}


\noindent{\Large{Appendixes}}
	\tableofcontents
	\newpage
	\appendix
	\section{Proofs Cited in Sections 2 and 3}

	\begin{definition}[Reward function]
    \label{rfd}
    There exists a real-valued vector $x^{*}\in \mathbb{R}^{n_{x}}$ and a reward function  $f(d;x^{*})$ s.t. for all actions $a_t\in\{1,2,\dots,K\}$ and context $d_t$,
    \begin{equation}
    \label{xstar_rf}
    \textstyle \mathbb{E}_{r_{d_t,a_t}}[r_{d_t,a_t}\,|\,d_t,a_t]=f_{a_t}(d_t;x^{*}),
    \end{equation}
    where $f_{a_t}(d_t;x^{*})\in[0,1]$ is the $a_t^{th}$ output of $f(d_t;x^{*}).$
    \end{definition}
	
	\begin{definition}[The Objective Function with Full Feedback]
    \label{ofwff}
    Suppose that after observing the context $d$, the agent gets all candidate actions' rewards. 
    The objective function with full feedback is
    \begin{equation*}
    \begin{aligned}
        \textstyle F(x)=
        \mathbb{E}_{d}\sum_{a=1}^K\mathbb{E}_{r_{d,a}}\big[(f_{a}(d;x)-r_{d,a})^2\,|\,d,a\big].
    \end{aligned}
    \end{equation*} 
    \end{definition}
	\begin{lemma}   
	    \label{min_star}
        $x^*$ in Eq. (\ref{xstar_rf}) is a global minimizer of $F(x)$. With $x=x^*$, we have $ F(x^*)=\mathbb{E}_d\sum_{a=1}^K\mbox{VAR}_{r_{d,a}}[r_{d,a}\,|\,d,a]$.
		\begin{proof}
			\begin{equation}
			\label{perfect_loss}
			\begin{aligned}
			F(x)=&\mathbb{E}_{d}\sum_{a=1}^K\mathbb{E}_{r_{d,a}}\big[(f_{a}(d;x)-r_{d,a})^2\,|\,d,a\big]\\
			=&\mathbb{E}_{d}\sum_{a=1}^K\big[f_{a}^2(x;d)-2f_{a}(x;d)f_{a}(x^*;d)+\mathbb{E}_{r_{d,a}}[r_{d,a}^2\,|\,d,a]\big]\\
			=&\mathbb{E}_{d}\sum_{a=1}^K\big[f_{a}^2(x;d)-2f_{a}(x;d)f_{a}(x^*;d)+(\mathbb{E}_{r_{d,a}}[r_{d,a}|d,a])^2+\mbox{VAR}_{r_{d,a}}[r_{d,a}\,|\,d,a]\big]\\
			=&\mathbb{E}_{d}\sum_{a=1}^K\big[(f_{a}(x;d)-f_{a}(x^*;d))^2+\mbox{VAR}_{r_{d,a}}[r_{d,a}\,|\,d,a]\big]\big].
			\end{aligned}
			\end{equation}
			Eq. (\ref{perfect_loss}) shows that the global minimum of $F(x)$ can be found when $x=x^*$. 
		\end{proof}
	\end{lemma}
	\begin{lemma}
        \label{uecr}
        Suppose in the $t^{th}$ round, the agent estimates reward of actions with $f(d_t,x_t)$, where $x_t$ is inferred from the history $\mathcal{H}_{t-1}=\{d_\tau,a_\tau,r_{d_\tau,a_\tau},\xi_\tau\,|\,\tau\in\{1,2,\dots,t-2\}\}$,
        where $\xi_\tau$ is the set of all other information the agent observes at round $\tau$, e.g., random variables used in the action policy. Let $a_t^*=\arg\max_a f_a(d_t,x_t)$. The expected (over $\mathcal{H}_T$) cumulative regret can be upper bounded by
        \begin{equation}
        \label{lere}
        \begin{aligned}
        \textstyle\sum_{t=1}^T Pr(a_t\ne a_t^*)+2\sqrt{K}\sum_{t=1}^T\mathbb{E}_{\mathcal{H}_{t-1}}\sqrt{F(x_t)-F(x^*)}.
        \end{aligned}
        \end{equation}
		\begin{proof}
			The upper bound on the expected cumulative regret is 
			\begin{equation}
			\label{lere1}
			\begin{aligned}
			&\mathbb{E}_{\mathcal{H}_{T}}\big[\sum_{t=1}^T\mathbb{E}_{r_{d,a}}[r_{d_t,a^*_{d_t}}\,|\,d_t,a^*_{d_t}]-r_{d_t,a_t}\,|\,\mathcal{H}_{T}\big]\\
			=&\sum_{t=1}^T\mathbb{E}_{d_t,\mathcal{H}_{t-1},a_t,r_{d_t,a_t},r_{d_t,a^*_{d_t}}}[r_{d_t,a^*_{d_t}}-r_{d_t,a_t}]\\
			\le&\sum_{t=1}^T\mathbb{E}_{d_t,\mathcal{H}_{t-1},a_t,r_{d_t,a_t},r_{d_t,a^*_{d_t}}}[\mathbf {1}(a_t=\arg\max_a f_a(d_t;x_{t}))(r_{d_t,a^*_{d_t}}-r_{d_t,a_t})]\\
			&+\sum_{t=1}^T\mathbb{E}_{d_t,\mathcal{H}_{t-1},a_t,r_{d_t,a_t},r_{d_t,a^*_{d_t}}}[\mathbf {1}(a_t\ne\arg\max_a f_a(d_t;x_{t}))(r_{d_t,a^*_{d_t}}-r_{d_t,a_t})]\\
			\le&\underbrace{\sum_{t=1}^T Pr(a\ne\arg\max_a f_a(d_t;x_{t}))}_\text{Exploration}
			+\underbrace{\sum_{t=1}^T\mathbb{E}_{d_t,\mathcal{H}_{t-1},r_{d_t,a_t},r_{d_t,a^*_{d_t}}}[r_{d_t,a^*_{d_t}}-r_{d_t,a_t}\,|\,a_t=\arg\max_a f_a(d_t;x_{t})]}_\text{Exploitation}\\
			\end{aligned}
			\end{equation}
			In Eq. (\ref{lere1}), the first term calculates regrets caused by the exploration and the second one represents those caused by the exploitation. For the second term, we further have
			\begin{equation}
			\label{online_upperbound_1}
			\begin{aligned}
			&\sum_{t=1}^T\mathbb{E}_{d_t,\mathcal{H}_{t-1},r_{d_t,a_t},r_{d_t,a^*_{d_t}}}[r_{d_t,a^*_{d_t}}-r_{d_t,a_t}\,|\,a_t=\arg\max_a f_a(d_t;x_{t})]\\
			=&\sum_{t=1}^T\mathbb{E}_{d_t,\mathcal{H}_{t-1},r_{d_t,a_t},r_{d_t,a^*_{d_t}}}[r_{d_t,a^*_{d_t}}-f_{a_{d_t}^*}(d_t;x_{t})+f_{a_{d_t}^*}(d_t;x_{t})-f_{a_t}(d_t;x_{t})\\
			&+f_{a_t}(d_t;x_{t})-r_{d_t,a_t}\,|\,a_t=\arg\max_a f_a(d_t;x_{t})]\\
			\le&\sum_{t=1}^T\mathbb{E}_{d_t,\mathcal{H}_{t-1}}\mathbb{E}_{r_{d_t,a^*_{d_t}},r_{d_t,a_t}}[(r_{d_t,a^*_{d_t}}-f_{a_{d_t}^*}(d_t;x_{t}))+(f_{a_t}(d_t;x_{t})-r_{d_t,a_t})\,|\,d_t,\mathcal{H}_{t-1},a_t=\arg\max_a f_a(d_t;x_{t})]\\
			\le&2\sum_{t=1}^T\mathbb{E}_{d_t,\mathcal{H}_{t-1}}\sum_{a=1}^K\sqrt{(\mathbb{E}_{r_{d_t,a}}\big[f_k(d_t;x_{t})-r_{d_t,a}\,|\,d_t,a,\mathcal{H}_{t-1}\big])^2}\\
			\le&2\sum_{t=1}^T\mathbb{E}_{d_t,\mathcal{H}_{t-1}}\sum_{a=1}^K\sqrt{\mathbb{E}_{r_{d_t,a}}\big[(f_k(d_t;x_{t})-r_{d_t,a})^2\,|\,d_t,a,\mathcal{H}_{t-1}\big]-\mbox{VAR}_{r_{d_t,a}}[r_{d_t,a}\,|\,d_t,a]}\\
			\le&2\sqrt{K}\sum_{t=1}^T\mathbb{E}_{d_t,\mathcal{H}_{t-1}}\sqrt{\sum_{a=1}^K\mathbb{E}_{r_{d_t,a}}\big[(f_k(d_t;x_{t})-r_{d_t,a})^2\,|\,d_t,a,\mathcal{H}_{t-1}\big]-\sum_{a=1}^K\mbox{VAR}_{r_{d_t,a}}[r_{d_t,a}\,|\,d_t,a]}\\
			\le&2\sqrt{K}\sum_{t=1}^T\mathbb{E}_{\mathcal{H}_{t-1}}\sqrt{\mathbb{E}_{d_t}\sum_{a=1}^K\mathbb{E}_{r_{d_t,a}}\big[(f_k(d_t;x_{t})-r_{d_t,a})^2\,|\,d_t,a,\mathcal{H}_{t-1}\big]-\mathbb{E}_{d_t}\sum_{a=1}^K\mbox{VAR}_{r_{d_t,a}}[r_{d_t,a}\,|\,d_t,a]}\\
			\end{aligned}
			\end{equation}
			The last inequality in Eq. (\ref{online_upperbound_1}) holds because of the Jensen's inequality. Then we substitute Eq. (\ref{online_upperbound_1}) into Eq. (\ref{lere1}). With 
			\begin{equation*}
			\begin{aligned}
			&F(x)=\mathbb{E}_{d}\sum_{a=1}^K\mathbb{E}_{r_{d,a}}\big[(f_{a}(d;x)-r_{d,a})^2\,|\,d,a\big],\\
			&F(x^*)=\mathbb{E}_d\sum_{a=1}^K\mbox{VAR}_{r_{d,a}}[r_{d,a}\,|\,d,a],
			\end{aligned}
			\end{equation*}
			we have
			\begin{equation*}
			\begin{aligned}
			&\mathbb{E}_{\mathcal{H}_{T}}\big[\sum_{t=1}^T\mathbb{E}_{r_{d,a}}[r_{d_t,a^*_{d_t}}|d_t,a^*_{d_t}]-r_{d_t,a_t}\,|\,\mathcal{H}_{T}\big]\\
			\le&\sum_{t=1}^T Pr(a_t\ne a_t^*)+2\sqrt{K}\sum_{t=1}^T\mathbb{E}_{\mathcal{H}_{t-1}}\sqrt{F(x_t)-F(x^*)}\\
			\end{aligned}
			\end{equation*}
		\end{proof}
	\end{lemma}
	
	\begin{assumption}
    \label{lowerbound}
        At $ t\in\{1,2,\dots,T\}$, for any $d_t\sim\mathcal{D}$ and $\mathcal{H}_{t-1}$, $\pi(a_t\,|\,d_t,\mathcal{H}_{t-1})\ge\Omega(\frac{1}{t^{\mathcal{C}}})$, where $\mathcal{C}>0$. 
    \end{assumption}
    \begin{lemma}[Inverse Propensity Scoring] 
    \label{unbiased_g}
    At round $ t\in\{1,2,\dots,T\}$, for any $d_t\sim\mathcal{D}$, and $a_t \in\{1,2,\dots,K\}$, the loss for selecting $a_t$ is defined as
    $l_{d_t,a_t}(x{;r_{d_t,a_t}})=(f_{a_t}(x;d_t)-r_{d_t,a_t})^2$. Let 
    \begin{equation}
    \label{unbiased_gradient}
        \textstyle{\tilde{g}}(x;\mathcal{H}_{t})=\frac{1}{\pi(a_t|d_t,\mathcal{H}_{t-1})}\nabla l_{d_t,a_t}(x;r_{d_t,a_t}),
    \end{equation}
    where $\nabla l$ is the gradient of $l$ with respect to $x$. Under Assumption \ref{lowerbound}, we have that
    
    (1) The stochastic gradient $\tilde{g}$ is an unbiased estimate of the full gradient $\nabla F(x)$, i.e., $\mathbb{E}_{\mathcal{H}_t}{\tilde{g}}(x;\mathcal{H}_{t})=\nabla F(x)$.
    
    (2) 
    The element-wise variance of $\tilde{g}$ is upper bounded by $\mathbb{E}_{d_t}\sum_{a=1}^K\mathbb{E}_{r_{d_t,a}}t^{\mathcal{C}}\|\nabla l_{d_t,a}(x;r_{d_t,a})\|^2$.
    
    For notational convenience, we will use $\tilde{g}_{t}(x)$ to represent ${\tilde{g}}(x;\mathcal{H}_{t})$ in the sequel.
		\begin{proof} First, we have
			\begin{equation*}
			\begin{aligned}
			&\mathbb{E}_{\mathcal{H}_t}{\tilde{g}}(x;\mathcal{H}_{t})\\
			=&\mathbb{E}_{d_t,\mathcal{H}_{t-1},a_t,r_{d_t,a_t}}{\tilde{g}}(x;\mathcal{H}_{t})\\
			=&\mathbb{E}_{d_t,\mathcal{H}_{t-1},a_t,r_{d_t,a_t}}\frac{1}{\pi(a_t\,|\,d_t,\mathcal{H}_{t-1})}\pi(a_t\,|\,d_t,\mathcal{H}_{t-1}){\tilde{g}}(x;\mathcal{H}_{t})\\
			=&\mathbb{E}_{d_t}\sum_{a=1}^K\mathbb{E}_{r_{d_t,a}}\big[\nabla l_{d_t,a}(x;r_{d_t,a})\,|\,d_t,a\big]\\
			=&\nabla F(x).
			\end{aligned}
			\end{equation*}
			As for the element-wise variance, let ${\tilde{g}}(x_{t};\mathcal{H}_{t})[i]$ be the $i^{th}$ element of $\tilde{g}_{\mathcal{H}_t}(x)$, we have
			\begin{equation*}
			\begin{aligned}
			&\mbox{VAR}_{\mathcal{H}_t}{\tilde{g}}(x_{t};\mathcal{H}_{t})[i]\\
			\le&\mathbb{E}_{d_t,\mathcal{H}_{t-1},a_t,r_{d_t,a}}\big[\frac{\|\nabla l_{t,a}(x;r_{d_t,a})\|^2}{\pi^2(a\,|\,d_t,\mathcal{H}_{t-1})}\big]\\
			=&\mathbb{E}_{d_t,\mathcal{H}_{t-1},r_{d_t,a}}\big[\frac{\|\nabla l_{t,a}(x;r_{d_t,a})\|^2}{\pi(a\,|\,d_t,\mathcal{H}_{t-1})}\big].
			\end{aligned}
			\end{equation*}
		\end{proof}
	\end{lemma}
	\begin{lemma} 
        \label{lowerbound_pi}
        For a cluster $c_{s,n}$, if there exists an action with a visit number smaller than $\frac{C^2(\sum_{k=1}^KN_{I(s,n),k,c_{s,n}})^{2\beta}}{ C^2K+2C\sqrt{K}+1}$, the action with the largest visit number will be selected by $\pi$ with probability $\Theta(\frac{1}{s^{\frac{\kappa}{2}}})$, while the action with the largest $U_{s,n,a}$ will be selected with probability $\Omega(s^{\omega})$.
		\begin{proof}
			In cluster $c_{s,n}$, let $\bar{a}=\arg\max_{k}N_{I(s,n),k,c_{s,n}}$ and $\underline{A}$ be the set containing all actions whose visit numbers are less than $\frac{C^2(\sum_{k=1}^KN_{I(s,n),k,c_{s,n}})^{\beta}}{ C^2K+2C\sqrt{K}+1}$. 
			In the $n^{th}$ round of the $s^{th}$ stage, for any cluster $c_{s,n}$, we have 
			\begin{equation}
			\label{largest_visit}
			N_{I(s,n),\bar{a}_{I(s,n)},c_{s,n}}\ge\frac{1}{K}\sum_{k=1}^K N_{I(s,n),k,c_{s,n}}.
			\end{equation}
			
			Based on Eq. (\ref{largest_visit}), for $\forall a\in\{1,2,\dots,K\}\backslash\bar{a}$, when 
			\begin{equation*}
			N_{I(s,n),a,c_{s,n}}\le\frac{C^2(\sum_{k=1}^KN_{I(s,n),k,c_{s,n}})^{2\beta}}{ C^2K+2C\sqrt{K}+1}\le\frac{C^2(\sum_{k=1}^KN_{I(s,n),k,c_{s,n}})^{2\beta}}{\big[C\sqrt{K}(\sum_{k=1}^KN_{I(s,n),k,c_{s,n}})^{\beta-\frac{1}{2}}+1\big]^2},
			\end{equation*}
			we have 
			\begin{equation*}
			C\frac{(\sum_{k=1}^KN_{I(s,n),k,c_{s,n}})^\beta}{\sqrt{N_{I(s,n),a,c_{s,n}}}}\ge C\sqrt{K}(\sum_{k=1}^KN_{I(s,n),k,c_{s,n}})^{\beta-\frac{1}{2}}+1\overset{Eq. (\ref{largest_visit})}{\ge} C\frac{(\sum_{k=1}^KN_{I(s,n),k,c_{s,n}})^\beta}{\sqrt{N_{I(s,n),\bar{a},c_{s,n}}}}+1.
			\end{equation*}
			Therefore, if $\underline{A}\ne\varnothing$, based on the action selection policy of SSGD-SCB, there exists an $ a\in\{1,2,\dots,K\}\backslash\bar{a}$ s.t. $a=\arg\max_{k}U_{s,n,k}$ and 
			\begin{equation*}
			\begin{aligned}
			&Pr(a\ is\ selected\,|\,c_{s,n})\\
			=&\frac{0.05}{Ks^{\frac{\kappa}{2}}}+\frac{(1-\frac{0.05}{s^{\frac{\kappa}{2}}})W_{s,n,a}}{\sum_{k=1}^K W_{s,n,k}}\\
			\ge&\frac{(1-\frac{0.05}{s^{\frac{\kappa}{2}}})U_{s,n,a}s^\omega}{U_{s,n,a}s^\omega+(K-1)U_{s,n,a}}\\
			\ge&\frac{(1-\frac{0.05}{s^{\frac{\kappa}{2}}})s^\omega}{s^\omega+K-1};\\
			&Pr(\bar{a}\ is\ selected|c_{s,n})\\
			=&\frac{0.05}{Ks^{\frac{\kappa}{2}}}+\frac{(1-\frac{0.05}{s^{\frac{\kappa}{2}}})W_{s,n,\overline{a}}}{\sum_{k=1}^K W_{s,n,k}}\\
			\le&\frac{0.05}{Ks^{\frac{\kappa}{2}}}+\frac{1}{\frac{U_{s,n,a}s^\omega}{U_{s,n,\overline{a}}}}\\
			\le&\frac{0.05}{Ks^{\frac{\kappa}{2}}}+\frac{1}{s^\omega}.\\
			\end{aligned}
			\end{equation*}
		\end{proof}
	\end{lemma}
	
	\newpage
	\section{Definitions of \texorpdfstring{$\alpha$}{Lg}-strongly Convex, \texorpdfstring{$L$}{Lg}-smooth and \texorpdfstring{$\rho$}{Lg}-Lipschitz Hessian}
	\begin{definition}[$\alpha$-strongly Convex]
    	\label{storngly_convex_definition}
    	A differentiable function $\mathcal{F}(x)$ is $\alpha$-strongly convex for $\alpha>0$ if for any points $x$ and $y$, $\mathcal{F}(y)\ge \mathcal{F}(x)+\nabla\mathcal{F}(x)^T(y-x)+\frac{\alpha}{2}\|y-x\|^2$.
	\end{definition}
	\begin{definition}[$L$-smooth]
		\label{smooth_definition}
		A differentiable function $\mathcal{F}(x)$ is $L$-smooth if there exists an $L$ such that for any $x$ and $y$, $\mathcal{F}(y)\le \mathcal{F}(x)+\nabla\mathcal{F}(x)^T(y-x)+\frac{L}{2}\|y-x\|^2$.
	\end{definition}
	\begin{definition}[$\rho$-Lipschitz Hessian]
		\label{smooth2_definition}
		A twice differentiable function $\mathcal{F}(x)$ has $\rho$-Lipschitz Hessian if there exists a $\rho >0$ such that for any points $x$ and $y$, $\|\nabla^2\mathcal{F}(y)- \nabla^2\mathcal{F}(x)\|\le\rho\|y-x\|$.
	\end{definition}
	
	\newpage
	\section{Proofs Cited in Section 4}
    \begin{definition} [$(\alpha,\gamma,\epsilon,\delta)$-strict saddle]
        \label{ssp}
        A twice differentiable function $\mathcal{F}(x)$ is $(\alpha,\gamma,\epsilon,\delta)$-strict saddle if for any point $x$, at least one of the following conditions is true:
        \newline
        1. $\|\nabla \mathcal{F}(x)\|\ge \epsilon$.
        \newline
        2. The minimum eigenvalue of $\nabla^2\mathcal{F}(x)$ is smaller than a fixed negative real number $\gamma$.
        \newline
        3. There is a local minimizer $x^l$ such that $\|x-x^l\|\le\delta$, and the function $\mathcal{F}(x')$ restricted to $2\delta$ neighborhood of $x^l$ ($\|x'-x^l\|\le 2\delta$) is $\alpha$-strongly convex.
    \end{definition}
    \begin{assumption}[\cite{ge2015escaping}]
    \label{non_convex_PCB}
    The objective function with full feedback $F(x)$ is an L-smooth, bounded, $(\alpha,\gamma,\epsilon,\delta)$-strict saddle function, and has $\rho$-Lipschitz Hessian.
    \end{assumption}
    \begin{definition}
    \label{local optimal convergence}
        Let $t'$ be a fixed round number. The event of local optimal convergence can be defined as $\mathfrak{E}^{t'}=\{$With $\mathcal{L}>0$, for any round $t>t'$ and any history $\mathcal{H}_{t'}$,  there exists a local minimum of $F(x)$ marked as $x^{\mathcal{H}_{t'}}$ such that $\|x_{t}-x^{\mathcal{H}_{t'}}\|\le\tilde{\mathcal{O}}(\frac{1}{t^{\mathcal{L}}})\le\delta\}$.
    \end{definition}
    \begin{assumption}
    \label{banditproof}
    Denote the set of local minimizers of $F(x)$ as $X$. For any $ x^l\in X$, given any sampled context $d_t$, the largest output in $f(d_t;x^l)$ is unique. In other words, if $a^{x^l}_{d_t}=\arg\max_{k\in\{1,2,\dots,K\}}f_k(d_t;x^l)$, then $f_a(d_t;x^l)\ne f_{a^{x^l}_{d_t}}(d_t;x^l)$ for any $a \neq a_{d_t}^{x^l}$.
    \end{assumption}
    \begin{assumption}
	\label{boundedslope}
	The gradient of $f$ is bounded. In other words, 
	There exists a $M_f\in\mathbb{R}_*^+$ such that $\forall d\sim\mathcal{D}$, $\forall x\sim\mathcal{R}^{n_{x}}$ and $\forall a\in\{1,2,\dots,K\}$, $\|\nabla f_a(x;d)\|\le \sqrt{M_f}$.
    \end{assumption}
    \begin{definition}[Expected mismatching rate] \label{def:local}
    Let $X$ be the set of local minimizer of a non-convex $F(x)$, the expected local optimal mismatching rate is
    \begin{equation*}
    \textstyle \mathbb{E}_{\mathcal{H}_T}\big[\min_{x^l\in X}(\sum_{t=1}^{T}\frac{\mathbf {1}(a_{d_t}\ne a_{d_t}^{x^l})}{T})\big],
    \end{equation*}
    where $T$ is the total number of rounds and $a_{d_t}^{x^l}=\arg\max_k(f_k(d_t;x^l))$.
    \end{definition}
    \begin{lemma}
    \label{lowerbound_local_optimal}
    With Assumptions \ref{non_convex_PCB}-\ref{boundedslope}, conditioned on the event $\mathfrak{E}^{t'}$, if the expected local regret is sub-linear, the expected mismatching rate has a zero-approaching upper bound.
		\begin{proof}
			Conditioned on $\mathfrak{E}^{t'}$, we have
			\begin{equation}
			\label{non_convex_upperbound}
			\begin{aligned}
			&\mathbb{E}_{\mathcal{H}_T}\big[\min_{x^l\in X}(\sum_{t=1}^{T}\mathbf {1}(a_{t}\ne a_{d_t}^{x^{l}}))\big]\\
			\le&\mathbb{E}_{\mathcal{H}_{T}}\big[\sum_{t=1}^{T}\mathbf {1}(a_t\ne a_{d_t}^{x^{\mathcal{H}_{t'}}})\big]\\
			\le&t'+\sum_{t=t'+1}^{T}\mathbb{E}_{\mathcal{H}_{t}}\big[\mathbf {1}(a_t\ne a_{d_t}^{x^{\mathcal{H}_{t'}}})\big]\\
			\le&t'+\sum_{t=t'+1}^{T}\mathbb{E}_{d_t,\mathcal{H}_{t-1},a_t}\big[\mathbf {1}(a_t\ne a_{d_t}^{x^{\mathcal{H}_{t'}}})\big]\\
			\le&t'+\sum_{t=t'+1}^{T}\mathbb{E}_{d_t,\mathcal{H}_{t-1}}\big[Pr(a_t\ne a_{d_t}^{x^{\mathcal{H}_{t'}}}\,|\,\mathcal{H}_{t-1},d_t)\big].
			\end{aligned}
			\end{equation}
			Then we have
			\begin{equation}
			\label{non_convex_upperbound_p2}
			\begin{aligned}
			&\sum_{t=t'+1}^{T}\mathbb{E}_{d_t,\mathcal{H}_{t-1}}\big[Pr(a_t\ne a_{d_t}^{x^{\mathcal{H}_{t'}}}\,|\,\mathcal{H}_{t-1},d_t)\big]\\
			=&\sum_{t=t'+1}^{T}\mathbb{E}_{d_t,\mathcal{H}_{t-1}}\big[Pr(a_t\ne a_{d_t}^{x^{\mathcal{H}_{t'}}},a_t=\arg\max_{k}f_k(d_t;x_t)\,|\,\mathcal{H}_{t-1},d_t)\\
			&+Pr(a_t\ne a_{d_t}^{x^{\mathcal{H}_{t'}}},a_t\ne\arg\max_{k}f_k(d_t;x_t)\,|\,\mathcal{H}_{t-1},d_t)\big]\\
			\le&\sum_{t=t'+1}^{T}\mathbb{E}_{d_t,\mathcal{H}_{t-1}}Pr(a_{d_t}^{x^{\mathcal{H}_{t'}}}\ne\arg\max_{k}f_k(d_t;x_t)\,|\,\mathcal{H}_{t-1},d_t)\\
			&+\sum_{t=t'+1}^{T}\mathbb{E}_{d_t,\mathcal{H}_{t-1}}Pr(a_t\ne\arg\max_{k}f_k(d_t;x_t)\,|\,\mathcal{H}_{t-1},d_t)\big].
			\end{aligned}
			\end{equation}
			Then we give the upper bound of $\sum_{t=t'+1}^{T}\mathbb{E}_{d_t,\mathcal{H}_{t-1}}Pr(a_{d_t}^{x^{\mathcal{H}_{t'}}}\ne\arg\max_{k}f_k(d_t;x_t)\,|\,\mathcal{H}_{t-1},d_t)$.
			\begin{equation}
			\label{non_convex_upperbound_p31}
			\begin{aligned}
			&\sum_{t=t'+1}^{T}\mathbb{E}_{d_t,\mathcal{H}_{t-1}}Pr(a_{d_t}^{x^{\mathcal{H}_{t'}}}\ne\arg\max_{k}f_k(d_t;x_t)\,|\,\mathcal{H}_{t-1},d_t)\\
			\le&\sum_{t=t'+1}^{T}\mathbb{E}_{d_t}\sum_{k\in\{1,2,\dots,K\}\backslash a_{d_t}^{x^{\mathcal{H}_{t'}}}}Pr(f_k(d_t;x_t)>f_{a_{d_t}^{x^{\mathcal{H}_{t'}}}}(d_t;x_t)\,|\,d_t)
			\end{aligned}
			\end{equation}
			We further have
			\begin{equation}
			\label{l5e1}
			\begin{aligned}
			&Pr(f_k(d_t;x_t)>f_{a_{d_t}^{x^{\mathcal{H}_{t'}}}}(d_t;x_t)\,|\,d_t)\\
			=&Pr(f_k(d_t;x_t)-\mathbb{E}f_k(d_t;x_t)+\mathbb{E}f_k(d_t;x_t)>f_{a_{d_t}^{x^{\mathcal{H}_{t'}}}}(d_t;x_t)-\mathbb{E}f_{a_{d_t}^{x^{\mathcal{H}_{t'}}}}(d_t;x_t)+\mathbb{E}f_{a_{d_t}^{x^{\mathcal{H}_{t'}}}}(d_t;x_t)\,|\,d_t)\\
			\le&Pr(f_k(d_t;x_t)-\mathbb{E}f_k(d_t;x_t)>\frac{\mathbb{E}f_{a_{d_t}^{x^{\mathcal{H}_{t'}}}}(d_t;x_t)-\mathbb{E}f_k(d_t;x_t)}{2}\,|\,d_t)\\
			&+Pr(f_{a_{d_t}^{x^{\mathcal{H}_{t'}}}}(d_t;x_t)-\mathbb{E}f_{a_{d_t}^{x^{\mathcal{H}_{t'}}}}(d_t;x_t)<\frac{\mathbb{E}f_k(d_t;x_t)-\mathbb{E}f_{a_{d_t}^{x^{\mathcal{H}_{t'}}}}(d_t;x_t)}{2}\,|\,d_t).
			\end{aligned}
			\end{equation}
			Following the idea of Lemma \ref{boundoutput} and the definition of the event $\mathfrak{E}^{t'}$, for any action $k\in\{1,2,\dots,K\}$, we have 
			\begin{align}
			\label{eq.epect}
			&|\mathbb{E}\left[f_k(x_t;d_t)-f_k(x^{\mathcal{H}_{t'}};d_t)\,|\,d_t\right]|
			\le \tilde{\mathcal{O}}(\|x_{t}-x^{\mathcal{H}^{t'}}\|),\\
			\label{eq.var}
			&\mbox{VAR}(f_k(x_t;d_t)|d_t)\le\tilde{\mathcal{O}}(\|x_{t}-x^{\mathcal{H}^{t'}}\|^2).
			\end{align}  
			Then define $\Delta$ such that 
			\begin{equation*}
			\Delta=min_{t,a\ne a_{d_t}^{x^{\mathcal{H}_{t'}}}}(\mathbb{E}\big[f_{a_{d_t}^{\mathcal{H}_{t'}}}(d_{t};x^{\mathcal{H}_{t'}})\big]-\mathbb{E}\big[f_a(d_{t};x^{\mathcal{H}_{t'}})\big]).
			\end{equation*}
			According to Assumption \ref{banditproof}, we have $\Delta>0$. When $t$ is large enough, based on Eq. (\ref{eq.epect}), we have $\frac{\mathbb{E}f_{a_{d_t}^{x^{\mathcal{H}_{t'}}}}(d_t;x_t)-\mathbb{E}f_k(d_t;x_t)}{2}\ge\frac{\Delta}{2}$. Therefore, according to Chebyshev's inequality and Eq. (\ref{eq.var}), we can bound the first term of Eq. (\ref{l5e1}) by 
			\begin{equation}
			\label{c0bound}
			\begin{aligned}
			&Pr(f_k(d_t;x_t)-\mathbb{E}f_k(d_t;x_t)>\frac{\mathbb{E}f_{a_{d_t}^{x^{\mathcal{H}_{t'}}}}(d_t;x_t)-\mathbb{E}f_k(d_t;x_t)}{2}\,|\,d_t)\le\tilde{\mathcal{O}}(\|x_{t}-x^{\mathcal{H}^{t'}}\|^2).
			\end{aligned}
			\end{equation}
			If $F(x)$ is a $(\alpha,\gamma,\epsilon,\delta)$-strict saddle function, we know that for large enough t, $x_t$ is in the strongly convex region of $F(x)$ and we further have 
			\begin{equation}
			\label{lr_upperbound_4}
			\tilde{\mathcal{O}}(\|x_{t}-x^{\mathcal{H}^{t'}}\|^2)\le \tilde{\mathcal{O}}(|F(x_{t})-F(x^{\mathcal{H}^{t'}})|)=
			\tilde{\mathcal{O}}(\frac{K|F(x_{t})-F(x^{\mathcal{H}^{t'}})|}{K})\le
			\tilde{\mathcal{O}}(\sqrt{K|F(x_{t})-F(x^{\mathcal{H}^{t'}})|}).
			\end{equation}
			The last inequality holds because $|F(x_{t})-F(x^{\mathcal{H}^{t'}})|\le K$. Following this proof, the second term of Eq. (\ref{l5e1}) has the same upper bound. Combine Eqs. (\ref{non_convex_upperbound}), (\ref{non_convex_upperbound_p2}) and (\ref{lr_upperbound_4}) together, we have 
			\begin{equation}
			\label{upperbound_final}
			\begin{aligned}
			&\mathbb{E}_{\mathcal{H}_T}\big[\min_{x^l\in X^l}(\sum_{t=1}^{T}\mathbf {1}(a_t\ne a_{d_t}^{x^{l}}))\big]\\
			\le&\mathcal{O}(1)+\sum_{t=t'+1}^{T}Pr(a_{d_t}^{x^{\mathcal{H}_{t'}}}\ne\arg\max_{k}f_k(d_t;x_t))\\
			&+\tilde{\mathcal{O}}(\mathbb{E}_{\mathcal{H}_{T}}\sum_{t=t'+1}^{T}\sqrt{K(F(x_{t})-F(x^{\mathcal{H}^{t'}}))}).
			\end{aligned}
			\end{equation}
			It means that if the expected local regret is sub-linear, the expected local optimal mismatching rate has a  zero-approaching upper bound.
		\end{proof}
	\end{lemma}
	\begin{remark}
        \label{convex_condition}
        If $F(x)$ is a convex function, the expected local optimal mismatching rate is the upper bound of expected cumulative regret averaged over $T$.
		\begin{proof}
			If $F(x)$ is a convex function, we have $X^l=\{x^*\}$. Then we get
			\begin{align*}
			&\mathbb{E}_{\mathcal{H}_T}\big[\min_{x^l\in X^l}(\sum_{t=1}^{T}\frac{\mathbf {1}(a_{d_t}\ne a_{d_t}^{x^l})}{T})\big]\\
			=&\mathbb{E}_{\mathcal{H}_{T}}\big[\sum_{t=1}^{T}\frac{\mathbf {1}(a_t\ne a_{d_t}^{x^*})}{T}\big]\\
			=&\sum_{t=1}^{T}\mathbb{E}_{\mathcal{H}_{t}}\big[\frac{\mathbf {1}(a_t\ne a_{d_t}^{x^*})}{T}\big]\\
			=&\sum_{t=1}^{T}\mathbb{E}_{d_t,\mathcal{H}_{t-1},a_t}\big[\frac{\mathbf {1}(a_t\ne a_{d_t}^{x^*})}{T}\big]\\
			\ge&\sum_{t=1}^{T}\mathbb{E}_{d_t,\mathcal{H}_{t-1},a_t,r_{d_t,a_t},r_{d_t,a^*_{d_t}}}\big[\frac{(r_{d_t,a^*_{d_t}}-r_{d_t,a_t})}{T}\big]\\
			=&\sum_{t=1}^{T}\mathbb{E}_{\mathcal{H}_t}\big[\frac{\mathbb{E}_{r_{d_t,a^*_{d_t}}}r_{d_t,a^*_{d_t}}-r_{d_t,a_t}}{T}\big]\\
			=&\frac{1}{T}\mathbb{E}_{\mathcal{H}_T}\sum_{t=1}^{T}\big[\mathbb{E}_{r_{d_t,a^*_{d_t}}}r_{d_t,a^*_{d_t}}-r_{d_t,a_t}\big]
			\end{align*}
		\end{proof}
	\end{remark}
	\begin{assumption}
	\label{Noise_SGD}
	$\forall d\sim\mathcal{D}$, $\forall a\in\{1,2,\dots,K\}$ and $\forall x\in\mathbb{R}^{n_{x}}$, $\exists M\in\mathbb{R}_*^+$, $\|\nabla l_{d,a}(x)\|^2\le M^2$.
	\end{assumption}     
	\begin{lemma}   
		\label{boundg}
		Based on Assumption \ref{Noise_SGD} and Lemma \ref{unbiased_g}, in the $n^{th}$ round of the $s^{th}$ stage of SSGD-SCB, we have
		\begin{equation*}
		\begin{aligned}
		&\mathbb{E}\big[\tilde{g}_{I(s,n)}(x_{I(s,n)})+\mathcal{N}_s]=\nabla F(x_{I(s,n)}),\\
		&\|\tilde{g}_{I(s,n)}(x_{I(s,n)})-\nabla F(x_{I(s,n)})+\mathcal{N}_s\|^2=\mathcal{O}( s^{\kappa}),\\
		&\mathbb{E}\|\tilde{g}_{I(s,n)}(x_{I(s,n)})-\nabla F(x_{I(s,n)})+\mathcal{N}_s\|^2=\Theta( s^{\kappa}).
		\end{aligned}
		\end{equation*} 
		\begin{proof}
			In the $n^{th}$ round of the $s^{th}$ stage, with $a_{I(s,n)}\in\{1,2,\dots,K\}$, Lemma \ref{unbiased_g} shows that
			\begin{equation*}
			\begin{aligned}
			&\mathbb{E}{\tilde{g}}_{I(s,n)}(x_{I(s,n)})=\nabla F(x_{I(s,n)}).
			\end{aligned}
			\end{equation*}
			
			Then according to the definition of $\mathcal{N}_s$, we have
			\begin{equation*}
			\begin{aligned}
			&\mathbb{E}[\mathcal{N}_s]=0,\\
			&\|\mathcal{N}_s\|=\mathcal{N}_0s^{\frac{\kappa}{2}}.
			\end{aligned}
			\end{equation*}
			
			With $\pi_S(a_{I(s,n)}|d_{I(s,n)},\mathcal{H}_{I(s,n)-1})\geq 0.05s^{-\frac{\kappa}{2}} K^{-1}$ and $\|\nabla l_{d_{I(s,n)},a_{I(s,n)}}(x_{I(s,n)})\|^2\le M^2$	
			, we have
			\begin{equation*}
			\begin{aligned}
			\|\tilde{g}_{I(s,n)}(x_{I(s,n)})\|\le\frac{Ks^{\frac{\kappa}{2}}}{0.05}\|\nabla l_{I(s,n),a_{I(s,n)}}(x_{I(s,n)})\|=\mathcal{O}(s^{\frac{\kappa}{2}}).
			\end{aligned}
			\end{equation*}
			Therefore, we have 
			\begin{equation}
			\label{variance_gradient}
			\begin{aligned}
			&\|\tilde{g}_{I(s,n)}(x_{I(s,n)})-\nabla F(x_{I(s,n)})+\mathcal{N}_s\|\\
			\le&\|\tilde{g}_{I(s,n)}(x_{I(s,n)})\|+\|\nabla F(x_{I(s,n)})\|+\|\mathcal{N}_s\|\\
			\le&\|\tilde{g}_{I(s,n)}(x_{I(s,n)})\|+\|\sum_{a=1}^K\mathbb{E}_{d_{I(s,n)},r_{d_{I(s,n)},a}}\nabla l_{I(s,n),a_{I(s,n)}}(x_{I(s,n)})\|+\|\mathcal{N}_s\|\\
			=&\mathcal{O}( s^{\frac{\kappa}{2}}).
			\end{aligned}
			\end{equation}
			
			Considering that $\|\mathcal{N}_s\|=\mathcal{N}_0s^{\frac{\kappa}{2}}$ and $\mathbb{E}[\mathcal{N}_s]=0$, we have 
			
			\begin{equation*}
			\begin{aligned}
			&\mathbb{E}\|\tilde{g}_{I(s,n)}(x_{I(s,n)})-\nabla F(x_{I(s,n)})+\mathcal{N}_s\|^2\\
			=&\mathbb{E}\|\tilde{g}_{I(s,n)}(x_{I(s,n)})-\nabla F(x_{I(s,n)})\|^2+\mathbb{E}\|\mathcal{N}_s\|^2\\
			=&\mathbb{E}\|\tilde{g}_{I(s,n)}(x_{I(s,n)})\|^2-\mathbb{E}\|\nabla F(x_{I(s,n)})\|^2+\mathbb{E}\|\mathcal{N}_s\|^2\\
			=&\mathbb{E}_{d_{I(s,n)}}\mathbb{E}_{\mathcal{H}_{I(s,n)-1}}\sum_{a=1}^K\mathbb{E}_{r_{d_{I(s,n)},a}}\frac{\|\nabla l_{I(s,n),a}(x_{I(s,n)};r_{d_{I(s,n)},a})\|^2}{\pi_S(a|d_{I(s,n)},\mathcal{H}_{I(s,n)-1})}-\mathbb{E}\|\nabla F(x_{I(s,n)})\|^2+\mathbb{E}\|\mathcal{N}_s\|^2\\
			=&\Theta( s^{\kappa}).
			\end{aligned}
			\end{equation*} 
		\end{proof}
	\end{lemma}
    \begin{theorem}
    	\label{converagerate}
    	Under Assumptions \ref{non_convex_PCB} and \ref{Noise_SGD}, 
    	for any $\zeta\in(0,1)$, 
    	there exists a large enough stage $s'$
    	such that with probability at least $1-\zeta$, 
    	for $\forall s>s'$ and $\forall n\in \{1, 2,\dots,T_0 s^{2\upsilon}\}$, we have $\|x_{I(s,n)}-x^{s'}\|^{2}\le\mathcal{O}( s^{\kappa-\upsilon}\log \frac{1}{\eta_{s}\zeta})\le\delta$, where $x^{s'}$ is a local minimizer of $F(x)$.
		
		\begin{proof}
			The proof of Theorem \ref{converagerate} can be found in Appendix \ref{proof_lemma_nsgd}.
			
		\end{proof}
	\end{theorem}
	\begin{lemma} 
		\label{boundoutput}
		In SSGD-SCB, if $\kappa-\upsilon<0$ and 
		\begin{equation*}
		\|x_{I(s,n)}-x^{s'}\|^2\leq\mathcal{O}(s^{\kappa-\upsilon}\log \frac{1}{\eta_{s}\zeta}),
		\end{equation*}
		under Assumption \ref{boundedslope}, for any action a and context d, the expectation and the variance of $f_a(x_{I(s,n)};d)$ can be bounded by
		\begin{equation*}
		\begin{aligned}
		&|\mathbb{E}\left[f_a(x_{I(s,n)};d)-f_a(x^{s'};d)|d\right]|
		\le \mathcal{O}(s^{\frac{\kappa-\upsilon}{2}}\log^{0.5} \frac{1}{\eta_{s}\zeta}),\\
		&\mbox{VAR}(f_a(x_{I(s,n)};d)|d)\le \mathcal{O}(s^{\kappa-\upsilon}\log \frac{1}{\eta_{s}\zeta}).
		\end{aligned}
		\end{equation*}       
		\begin{proof}
			In round $t\in\{1,2,\dots,T\}$, for any action $a$,
			according to Mean Value Inequality, we have 
			\begin{align*}
			&\|f_a(x_{I(s,n)};d)-f_a(x^{s'};d)\|^2\\
			\le&\|\int _{0}^{1}\nabla f_a(x_{I(s,n)}+h(x_{I(s,n)}-x^{s'});d)\,dh\|^2\| x_{I(s,n)}-x^{s'}\|^2\\
			\le&(\int _{0}^{1}\|\nabla f_a(x_{I(s,n)}+h(x_{I(s,n)}-x^{s'});d)\,\|dh)^2\| x_{I(s,n)}-x^{s'}\|^2,
			\end{align*}
			where $h\in[0,1]$.
			
			Then we have:
			\begin{equation}
			\label{Distance_t_star_2}
			\begin{aligned}
			&\mathbb{E}\big[\|f_a(x_{I(s,n)};d)-f_a(x^{s'};d)\|^2\,|\,d\big]\\
			\le&\mathbb{E}[(\int _{0}^{1}\|\nabla f_a(x_{I(s,n)}+h(x_{I(s,n)}-x^{s'});d)\,\|dh)^2\| x_{I(s,n)}-x^{s'}\|^2\,|\,d]\\
			\le&M_f\mathcal{O}(s^{\kappa-\upsilon}\log \frac{1}{\eta_{s}\zeta})\\
			=&\mathcal{O}(s^{\kappa-\upsilon}\log \frac{1}{\eta_{s}\zeta}).
			\end{aligned}  
			\end{equation}
			
			Eq. (\ref{Distance_t_star_2}) means that
			\begin{equation*}
			\begin{aligned}
			&\mathbb{E}\big[\|f_a(x_{I(s,n)};d)-f_a(x^{s'};d)\|^2\,|\,d\big]\\
			=&|\mathbb{E}\left[f_a(x_{I(s,n)};d)-f_a(x^{s'};d)\,|\,d\right]|^2+\mbox{VAR}\big[f_a(x_{I(s,n)};d)-f_a(x^{s'};d)\,|\,d\big]\\
			\le&\mathcal{O}(s^{\kappa-\upsilon}\log \frac{1}{\eta_{s}\zeta}),
			\end{aligned}
			\end{equation*}
			which yields 
			\begin{equation}
			\label{output_c1}
			|\mathbb{E}\left[f_a(x_{I(s,n)};d)-f_a(x^{s'};d)\,|\,d\right]|
			\le \mathcal{O}(s^{\frac{\kappa-\upsilon}{2}}\log^{0.5} \frac{1}{\eta_{s}\zeta}).
			\end{equation}
			
			Besides, Eq. (\ref{Distance_t_star_2}) also shows that 
			\begin{equation}
			\label{output_c2}
			\begin{aligned}
			&\mathbb{E}\big[\|f_a(x_{I(s,n)};d)-f_a(x^{s'};d)\|^2\,|\,d\big]\\
			=&\mathbb{E}[f_a(x_{I(s,n)};d)^2\,|\,d]+\mathbb{E}[f_a(x^{s'};d)^2\,|\,d]-2\mathbb{E}[f_a(x_{I(s,n)};d)f_a(x^{s'};d)\,|\,d]\\
			=&\mathbb{E}[f_a(x_{I(s,n)};d)^2\,|\,d]+f_a(x^{s'};d)^2-2\mathbb{E}[f_a(x_{I(s,n)};d)\,|\,d]f_a(x^{s'};d)\\
			\le& \mathcal{O}(s^{\kappa-\upsilon}\log \frac{1}{\eta_{s}\zeta}).
			\end{aligned}
			\end{equation}
			
			Based on Eq. (\ref{output_c2}), the variance of $f_a(x_{I(s,n)};d)$ can be bounded by
			\begin{equation}
			\label{output_c3}
			\begin{aligned}
			&\mbox{VAR}\left[f_a(x_{I(s,n)};d)\,|\,d\right]\\
			=&\mathbb{E}[f_a(x_{I(s,n)};d)^2\,|\,d]-\mathbb{E}[f_a(x_{I(s,n)};d)\,|\,d]^2\\
			\le&2\mathbb{E}[f_a(x_{I(s,n)};d)\,|\,d]f_a(x^{s'};d)-f_a(x^{s'};d)^2-\mathbb{E}[f_a(x_{I(s,n)};d)\,|\,d]^2+\mathcal{O}(s^{\kappa-\upsilon}\log \frac{1}{\eta_{s}\zeta})\\
			\le& \mathcal{O}(s^{\kappa-\upsilon}\log \frac{1}{\eta_{s}\zeta}).
			\end{aligned}
			\end{equation}
			With Eqs. (\ref{output_c2}) and (\ref{output_c3}), we complete the proof of Lemma \ref{boundoutput}. 
		\end{proof}
	\end{lemma}
    \begin{theorem}
    \label{pcbsgdconclusion}
        Under Assumptions 2-5, if $\mathcal{S}$ is large enough,
        with probability at least $1-\zeta$,
        the expected local optimal mismatching rate is bounded by
    	\begin{equation*}
    	    \begin{aligned}
    	    &\textstyle \mathbb{E}
    	    \big[\min_{x^l\in X}(\sum_{s=1}^{\mathcal{S}}\sum_{n=1}^{T_0s^2}\frac{\mathbf {1}(a_{I(s,n)}\ne a^{x^l}_{d_{I(s,n)}})}{\sum_{s=1}^\mathcal{S}T_0s^2})\big]
            \le\tilde{\mathcal{O}}(\frac{1}{\mathcal{S}^\Lambda}),\\
            &\textstyle \Lambda=\min(\frac{\kappa}{2},\omega,(1-2\beta)(1+2\upsilon),(2\beta-1)(1+2\upsilon)+\kappa-\upsilon).
    	    \end{aligned}
    	\end{equation*}
    	where $\beta\in(0,0.5)$, $\upsilon\ge 1$, $\kappa\in(0,\upsilon)$, and $\omega>0$.
		\begin{proof}
			\label{incorrectB}
			When $C\ne0$, although we can get the upper bound of the expected mismatching rate by calculating the upper bound of the local cumulative regret, as shown in Lemma \ref{lowerbound_local_optimal}, we directly calculate the upper bound of the expected mismatching rate.
			According to Theorem \ref{converagerate}, we have $Pr(\mathfrak{E}^{I(s',T_{s'})})\ge 1-\zeta$. Conditioned on $\mathfrak{E}^{I(s',T_{s'})}$, we have
			\begin{equation*}
			\begin{aligned}
			&\mathbb{E}_{\mathcal{H}_{I(\mathcal{S},T_{\mathcal{S}})}}\big[\min_{x^l\in X}(\sum_{s=1}^{\mathcal{S}}\sum_{n=1}^{T_0s^{2\upsilon}}\frac{\mathbf {1}(a_{I(s,n)}\ne a^{x^l}_{d_{I(s,n)}})}{\sum_{s=1}^\mathcal{S}T_0s^{2\upsilon}})\big]\\
			\le&\mathbb{E}_{\mathcal{H}_{I(\mathcal{S},T_{\mathcal{S}})}}\big[\sum_{s=1}^{\mathcal{S}}\sum_{n=1}^{T_0s^{2\upsilon}}\frac{\mathbf {1}(a_{I(s,n)}\ne a^{x^{\mathcal{H}_{I(s',T_{s'})}}}_{d_{I(s,n)}})}{\sum_{s=1}^\mathcal{S}T_0s^{2\upsilon}}\big]\\
			\end{aligned}
			\end{equation*}
			For the sake of simplicity, we simplify $a^{x^{\mathcal{H}_{I(s',T_{s'})}}}_{d_{I(s,n)}}$ as $a_{I(s,n)}^l$. Then we have 
			\begin{equation*}
			\begin{aligned}
			&\sum_{s=1}^{\mathcal{S}}\sum_{n=1}^{T_0s^{2\upsilon}}\mathbb{E}\big[\frac{\mathbf {1}(a_{I(s,n)}\ne a_{I(s,n)}^l)}{\sum_{s=1}^\mathcal{S}T_0s^{2\upsilon}}\big]\\
			\le&\frac{\sum_{s=1}^{s'}T_0s^{2\upsilon}}{\sum_{s=1}^\mathcal{S}T_0s^{2\upsilon}}+\frac{\sum_{s=s'+1}^{\mathcal{S}}\sum_{n=1}^{T_0s^{2\upsilon}}(Pr(a_{I(s,n)}\ne a_{I(s,n)}^l)}{\sum_{s=1}^\mathcal{S}T_0s^{2\upsilon}}
			\end{aligned}
			\end{equation*}
			
			In the above equation, $Pr(a_{I(s,n)}\ne a_{I(s,n)}^l)$ can be bounded by 
			\begin{equation*}
			\begin{aligned}
			&Pr(a_{I(s,n)}\ne a_{I(s,n)}^l)\\
			\le&Pr(a_{I(s,n)}\ne a_{I(s,n)}^l,a_{I(s,n)}\ne \arg\max_i U_{s,n,i})+Pr(a_{I(s,n)}\ne a_{I(s,n)}^l,a_{I(s,n)}= \arg\max_i U_{s,n,i})\\
			\le &Pr(a_{I(s,n)}\ne a_{I(s,n)}^l,a_{I(s,n)}\ne \arg\max_i U_{s,n,i})+Pr(a_{I(s,n)}\ne a_{I(s,n)}^l\,|\,a_{I(s,n)}=\arg\max_i U_{s,n,i}).
			\end{aligned}
			\end{equation*}
			
			
			
			First, we give the upper bound of $Pr(a_{I(s,n)}\ne a_{I(s,n)}^l,a_{I(s,n)}\ne \arg\max_i U_{s,n,i})$:
			\begin{equation}
			\label{final_second_part}
			\begin{aligned}
			&Pr(a_{I(s,n)}\ne a_{I(s,n)}^l,a_{I(s,n)}\ne \arg\max_i U_{s,n,i})\\
			\le&\frac{0.05s^{\frac{-\kappa}{2}}}{K}+\frac{1}{\frac{max_i(U_{s,n,i})s^\omega}{max_{i\in\{1,2,\dots,K\}\backslash\arg\max_i(U_{s,n,i})}(U_{s,n,i})}}\\
			\le&\frac{0.05s^{\frac{-\kappa}{2}}}{K}+\frac{1}{s^\omega}=\mathcal{O}(s^{max(\frac{-\kappa}{2},-\omega)})\\
			\end{aligned}
			\end{equation}
			
			We then bound $Pr(a_{I(s,n)}\ne a_{I(s,n)}^l\,|\,a_{I(s,n)}=\arg\max_i U_{s,n,i})$.
			
			Based on the conclusion of Lemma \ref{boundoutput}, conditioned on $\mathfrak{E}^{s'}$,
			for $s> s'$ and $\forall a\in\{1,2,\dots,K\}$, we have
			\begin{equation*}
			\begin{aligned}
			&|\mathbb{E}\big[f_a(x_{I(s,n)};d)-f_a(x^{s'};d)\,|\,d\big]|\le \tilde{\mathcal{O}}(s^{\frac{\kappa-\upsilon}{2}});\\
			&\mbox{VAR}\left[f_a(x_{I(s,n)};d)\,|\,d\right]\le\tilde{\mathcal{O}}(s^{\kappa-\upsilon}),
			\end{aligned}
			\end{equation*}
			We define $\Delta$ as 
			\begin{equation}
			\label{Delta}
			\Delta=min_{s,n,a\ne a_{d_{I(s,n)}}^l}(\mathbb{E}\big[f_{a_{I(s,n)}^l}(d_{I(s,n)};x^{s'})\big]-\mathbb{E}\big[f_a(d_{I(s,n)};x^{s'})\big]).
			\end{equation}
			According to Assumption \ref{banditproof}, $\Delta>0$.
			
			For the sake of simplicity, let
			\begin{equation*}
			\begin{aligned}
			e_{s,n,a}=C\frac{(\sum_{k=1}^KN_{I(s,n),k,c_{s,n}})^{\beta}}{\sqrt{N_{I(s,n),a,c_{s,n}}}},
			\end{aligned}
			\end{equation*}
			Then we have 
			\begin{equation}
			\label{final_most_important}
			\begin{aligned}
			&Pr(a_{I(s,n)}\ne a_{s,n}^l\,|\,a_{I(s,n)}=\arg\max_i U_{s,n,i})\\
			\le&\sum_{k\in\{1,\dots,K\}/\{a^l_{I(s,n)}\}}Pr(f_k(d_{I(s,n)};x_{I(s,n)})+e_{s,n,k}> f_{a^l_{I(s,n)}}(d_{I(s,n)};x_{I(s,n)})+e_{p,n,a^l_{I(s,n)}})\\
			=&\sum_{k\in\{1,\dots,K\}/\{a^l_{I(s,n)}\}}Pr(f_k(d_{I(s,n)};x_{I(s,n)})-\mathbb{E}f_k(d_{I(s,n)};x_{I(s,n)})+\mathbb{E}f_k(d_{I(s,n)};x_{I(s,n)})-e_{s,n,k}\\
			&+2e_{s,n,k}> f_{a^l_{I(s,n)}}(d_{I(s,n)};x_{I(s,n)})-\mathbb{E}f_{a^l_{I(s,n)}}(d_{I(s,n)};x_{I(s,n)})+\mathbb{E}f_{a^l_{I(s,n)}}(d_{I(s,n)};x_{I(s,n)})+e_{s,n,a^l_{I(s,n)}})\\
			\le&\sum_{k\in\{1,\dots,K\}/\{a^l_{I(s,n)}\}}\big[Pr(f_k(d_{I(s,n)};x_{I(s,n)})-\mathbb{E}f_k(d_{I(s,n)};x_{I(s,n)})> e_{s,n,k})\\
			&+Pr(2e_{s,n,k}> \mathbb{E}f_{a^l_{I(s,n)}}(d_{I(s,n)};x_{I(s,n)})-\mathbb{E}f_k(d_{I(s,n)};x_{I(s,n)}))\\
			&+Pr(f_{a^l_{I(s,n)}}(d_{I(s,n)};x_{I(s,n)})-\mathbb{E}f_{a^l_{I(s,n)}}(d_{I(s,n)};x_{I(s,n)})< -e_{s,n,a^l_{I(s,n)}})\big].
			\end{aligned}
			\end{equation}
			
			In order to bound Eq. (\ref{final_most_important}) , for $k\in\{1,2,\dots,K\}$, we first bound $Pr(f_k(d_{I(s,n)};x_{I(s,n)})-\mathbb{E}f_k(d_{I(s,n)};x_{I(s,n)})\ge e_{s,n,k})$. Following the conclusion of Lemma \ref{boundoutput} and Chebyshev's Inequality, we can get an upper bound:
			\begin{equation}
			\label{Chebyshev_Inequality}
			\begin{aligned}
			&Pr(f_k(d_{I(s,n)};x_{I(s,n)})-\mathbb{E}f_k(d_{I(s,n)};x_{I(s,n)})> e_{s,n,k})\\
			=&\mathbb{E}_{d_{I(s,n)}}Pr(f_k(d_{I(s,n)};x_{I(s,n)})-\mathbb{E}f_k(d_{I(s,n)};x_{I(s,n)})\ge e_{s,n,k}\,|\,d_{I(s,n)})\\
			\le&\frac{\tilde{\mathcal{O}}(s^{\kappa-\upsilon})N_{I(s,n),k,c_{s,n}}}{(\sum_{k'=1}^KN_{I(s,n),k',c_{s,n}})^{2\beta}}\le\tilde{\mathcal{O}}(s^{\kappa-\upsilon})(\sum_{s=1}^{\mathcal{S}}T_0s^{2})^{1-2\beta}\\
			\le&\tilde{\mathcal{O}}((\mathcal{S}^{2\upsilon+1})^{1-2\beta}s^{\kappa-\upsilon})\\
			\le&\tilde{\mathcal{O}}(\mathcal{S}^{(2\upsilon+1)(1-2\beta)}s^{\kappa-\upsilon}).
			\end{aligned}    
			\end{equation}
			Similar to Eq. (\ref{Chebyshev_Inequality}), we have $Pr(f_{a^l_{I(s,n)}}(d_{I(s,n)};x_{I(s,n)})-\mathbb{E}f_{a^l_{I(s,n)}}(d_{I(s,n)};x_{I(s,n)})< -e_{s,n,a^l_{I(s,n)}})\le\tilde{\mathcal{O}}(\mathcal{S}^{(2\upsilon+1)(1-2\beta)}s^{\kappa-\upsilon})$.
			As for $Pr(2e_{s,n,k}> \mathbb{E}f_{a^l_{I(s,n)}}(d_{I(s,n)};x_{I(s,n)})-\mathbb{E}f_k(d_{I(s,n)};x_{I(s,n)}))$, with $k\ne a^l_{I(s,n)}$, $s>s'$ and $n\in[T_0 s^{2\upsilon}]$, we have 
			\begin{equation*}
			\begin{aligned}
			&Pr(2e_{s,n,k}> \mathbb{E}f_{a^l_{I(s,n)}}(d_{I(s,n)};x_{I(s,n)})-\mathbb{E}f_k(d_{I(s,n)};x_{I(s,n)}))\\
			\le&\mathbb{E}_{d_{I(s,n)}\sim\mathcal{D}}Pr(2e_{s,n,k}> \mathbb{E}f_{a^l_{I(s,n)}}(d_{I(s,n)};x_{I(s,n)})-\mathbb{E}f_k(d_{I(s,n)};x_{I(s,n)})\,|\,d_{I(s,n)})\\
			\le&\mathbb{E}_{d_{I(s,n)}\sim\mathcal{D}}Pr(2e_{s,n,k}> \mathbb{E}f_{a^l_{I(s,n)}}(d_{I(s,n)};x^{s'})-\mathbb{E}f_k(d_{I(s,n)};x^{s'})-\tilde{\mathcal{O}}(s^{\frac{\kappa-\upsilon}{2}})\,|\,d_{I(s,n)})\\
			\le&\mathbb{E}_{d_{I(s,n)}\sim\mathcal{D}}Pr(2C\frac{(\sum_{k'=1}^KN_{I(s,n),k',c_{s,n}})^{\beta}}{\sqrt{N_{I(s,n),k,c_{s,n}}}}> \Delta-\tilde{\mathcal{O}}(s^{\frac{\kappa-\upsilon}{2}})\,|\,d_{I(s,n)})\\
			\le&\mathbb{E}_{d_{I(s,n)}\sim\mathcal{D}}Pr(N_{I(s,n),k,c_{s,n}}<\frac{4C^2(\sum_{k'=1}^KN_{I(s,n),k',c_{s,n}})^{\beta}}{(\Delta-\tilde{\mathcal{O}}(s^{\frac{\kappa-\upsilon}{2}}))^2}\,|\,d_{I(s,n)})\\
			\le&\mathbb{E}_{d_{I(s,n)}\sim\mathcal{D}}Pr(N_{I(s,n),a,c_{s,n}}<\frac{\tilde{\mathcal{O}}(\mathcal{S}^{\beta (2\upsilon+1)})}{(\Delta-\tilde{\mathcal{O}}(s^{\frac{\kappa-\upsilon}{2}}))^2}\,|\,d_{I(s,n)}).
			\end{aligned}
			\end{equation*}
			There exists a stage $s_1$ such that when $s\ge s_1$, $(\Delta-\tilde{\mathcal{O}}(s^{\frac{\kappa-\upsilon}{2}}))^2\ge\frac{\Delta^2}{4}$. It means that when \begin{equation}
			\label{lowerbound_n_final}
			N_{I(s,n),a,c_{s,n}}>\tilde{\mathcal{O}}(\mathcal{S}^{\beta (2\upsilon+1)}),
			\end{equation}
			we have 
			\begin{equation*}
			Pr(2e_{s,n,k}> \mathbb{E}f_{a^l_{I(s,n)}}(d_{I(s,n)};x_{I(s,n)})-\mathbb{E}f_k(d_{I(s,n)};x_{I(s,n)}))=0,
			\end{equation*}
			which yields
			\begin{equation}
			\label{final_first_part}
			Pr(a_{I(s,n)}\ne a_{I(s,n)}^l\,|\,a_{I(s,n)}=\arg\max_i U_{s,n,i})\le\tilde{\mathcal{O}}(\mathcal{S}^{(2\upsilon+1)(1-2\beta)}s^{\kappa-\upsilon}).
			\end{equation}
			
			Let $s>max(s_1,s')$. With Eqs. (\ref{final_second_part}), (\ref{lowerbound_n_final}) and (\ref{final_first_part}), we have:
			\begin{equation}
			\label{final_conclusion}
			\begin{aligned}
			&\sum_{s=1}^{\mathcal{S}}\sum_{n=1}^{T_0s^{2\upsilon}}\mathbb{E}\big[\frac{\mathbf {1}(a_{I(s,n)}\ne a_{I(s,n)}^l)}{\sum_{s=1}^\mathcal{S}T_0s^{2\upsilon}}\big]\\
			\le&\frac{\sum_{s=1}^{max(s_1,s')}T_0s^{2\upsilon}}{\sum_{s=1}^\mathcal{S}T_0s^{2\upsilon}}+\sum_{s=max(s_1,s')+1}^{\mathcal{S}}\sum_{n=1}^{T_0s^{2\upsilon}}\frac{Pr(a_{I(s,n)}\ne a_{I(s,n)}^l)}{\sum_{s=1}^\mathcal{S}T_0s^{2\upsilon}}\\
			\le&\mathcal{O}(\mathcal{S}^{max(\frac{-\kappa}{2},-\omega)})+\frac{\tilde{\mathcal{O}}(\mathcal{S}^{\beta (2\upsilon+1)})}{\sum_{s=1}^\mathcal{S}T_0s^{2\upsilon}}+\sum_{s=max(s_1,s')+1}^{\mathcal{S}}\sum_{n=1}^{T_0s^{2\upsilon}}\frac{\tilde{\mathcal{O}}(\mathcal{S}^{(2\upsilon+1)(1-2\beta)}s^{\kappa-\upsilon})}{\sum_{s=1}^\mathcal{S}T_0s^{2\upsilon}}\\
			\le&\tilde{\mathcal{O}}(\frac{1}{\mathcal{S}^{\min(\frac{\kappa}{2},(1-2\beta)(1+2\upsilon),(2\beta-1)(1+2\upsilon)-\kappa+\upsilon)}}).\\
			\end{aligned}
			\end{equation}
			When $C=0$, $Pr(a_{I(s,n)}\ne a_{s,n}^l\,|\,a_{I(s,n)}=\arg\max_i U_{s,n,i})$ can be bounded following the proof of Eqs. (\ref{non_convex_upperbound_p31}), (\ref{l5e1}) and (\ref{c0bound}), then we can get the upper bound $\mathcal{O}(\frac{1}{\mathcal{S}^{\min(\frac{\kappa}{2},\upsilon-\kappa)}})$.
		\end{proof}
	\end{theorem}
	\newpage
	\section{Experiment Settings of SSGD-SCB and Baseline Algorithms}
	As shown in Fig. \ref{cnns} of this appendix, the reward function of SSGD-SCB is modeled by a variant of VGG-11 with batch normalization. Compared to the original version of VGG-11 with batch normalization, fully connected layers are removed except one with ten outputs, and the softmax layer is replaced by a sigmoid activation function.
	
    In Fig. \ref{cnns}, Conv3-$i$ means the kernel size of the convolutional layer is three and the number of outputs is $i$. Besides, padding and stride of all convolutional layers are 1. The kernel size, stride, and padding of max-pool layers are 2, 2, and 0, respectively. Finally, Fully Connected Layer-10 means that the number of the fully connected layer outputs is 10.
	
	\begin{figure}[ht]
		\centering
		\includegraphics[width=4cm]{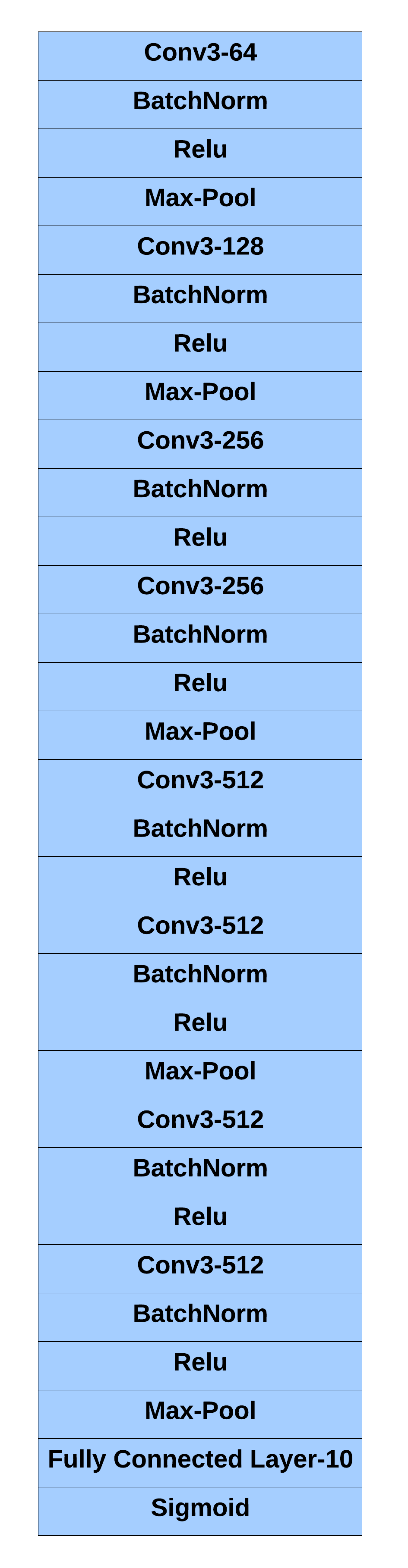}
		\caption{The DNN used in SSGD-SCB.}
		\label{cnns}
	\end{figure}
	
	The hyper-parameters of SSGD-SCB and four baseline algorithms are given as follows.
	\begin{enumerate}
		\item SSGD-SCB. According to Remark 2, we let $\kappa$ be 0.5, and further set $\upsilon$ and $\beta$ as 1 and 0.4583. For $T_0$, $C$, $\mathcal{N}_0$, and $\omega$, a grid search is conducted over $\{1000,5000\}$,  $\{0.01,0.05,0.1\}$, $\{$0.1, 0.001, 0.0001, 0.00001$\}$ and $\{0.5,1\}$ respectively. Besides, we test an additional hyper-parameter configuration, in which $C$ is halved for each of the first ten stages.
		\item $\epsilon$-greedy policy. In our experiment, a grid search is conducted on $\epsilon$ over $\{0.05,0.1,0.15\}$, and the reward function is modeled by the DNN used in SSGD-SCB. Following \citeauthor{foster2018practical}'s work (\cite{foster2018practical}), \emph{weighted regression onto all actions} is used in the $\epsilon$-greedy policy.
		\item DeepFPL. According to the paper proposing DeepFPL (\cite{kveton2019randomized}), we set the variance of noise as 1. The DNN used in DeepFPL is identical to that used in SSGD-SCB. For the number of neural networks $M$, a grid search is conducted over $\{3,5,10\}$, and the best $M$ will be selected.  
		\item EL and CE. The lower bounds of probability to select actions in EL and CE are set as 0.005. Because the outputs of EL and CE are a probability distribution rather than predicted rewards, the DNNs used in EL and CE are identical with that of SSGD-SCB except that the last layers of DNN in EL and CE are a softmax layer and the out-of-sample MSE of these two algorithms are not presented in the following result.
		\item Neural Linear (NL). Considering that the original version of Neural Linear is hard to be deployed to complex deep learning SCB problems. In our experiments, we implement an online version of Neural Linear: the DNN model generating the input of the Bayesian linear model will be updated every 50K rounds. After each update, instead of all observed history, a forward pass will be performed on history information collected within the latest 50K rounds to generate the training dataset for the Bayesian linear model. The hyper-parameters of NL follow (\cite{riquelme2018deep}), i.e., $\alpha=3$, $\beta=3$ and $\lambda=0.25$.
	\end{enumerate}
	
	For all algorithms mentioned above, the momentum is not used, and a grid search is performed over the learning rates $\{$0.05, 0.01, 0.005, 0.001$\}$. Except SSGD-SCB, we also test the performance of these algorithms with a decaying learning rate (the learning rate is halved every 1,500K rounds). Besides, in order to reduce the variance of the gradient, the tech of Mini-Batch Gradient Descent is applied to these five algorithms. Specifically, instead of updating parameters in every round, gradients are summed up and then the average is used to update parameters every 64 rounds.
	
	In addition, we compare  generalization performance of SSGD-SCB with DNN trained by CIFAR-like datasets under supervised learning.  
	The structure of DNN is identical to that used in SSGD-SCB. The model is trained and tested with the dataset of CIFAR-10 and CIFAR-10+N. The size of minibatch is 64, and SGD is used to fit the model. A grid search over $\{0.05,0.01,0.005,0.001\}$ is conducted on the learning rate and the learning rate is halved every 1,500K rounds (30 epochs). Note that the test set is used to tune parameters in this model since we would like to find the best top-1 test accuracy the DNN can get under supervised learning. 
	
	For SSGD-SCB with a linear reward function, the algorithm is implemented as will be shown in Appendix F. Note that we add the $l^2$-regularization to $F(x)$, which results in a strongly convex objective function. The hyper-parameters $\kappa$, $\upsilon$ and $\beta$ are set to 0.2, 0.4 and 0.45, respectively. According to the definition of $\Delta s$, we have $\Delta s=5$. Considering that $\Delta s$ is small, in this experiment we simply run SSGD-SCB from $s=1$ instead of $s=6$. For the other hyper-parameters $\omega$, $C$, the weight of $l^2$-regularization and the learning rate $\eta_0$, a systematic grid search is conducted over the values, resp., $\{0.2,0.4\}$, $\{0.01,0.05,0.1\}$,  $\{0.0005,0.005,0.05\}$ and   $\{0.1,0.05,0.01,0.005,0.001,0.0005,0.0001\}$. The hyper-parameters of the baseline algorithms follow \citeauthor{bietti2018contextual}'s work (\cite{bietti2018contextual}). The results of this simulation are shown in Fig. \ref{cdsgdresult}.
	\begin{figure}[htp]
    \centering
    \begin{minipage}{0.33\textwidth}
    \centering
    \includegraphics[width=\textwidth]{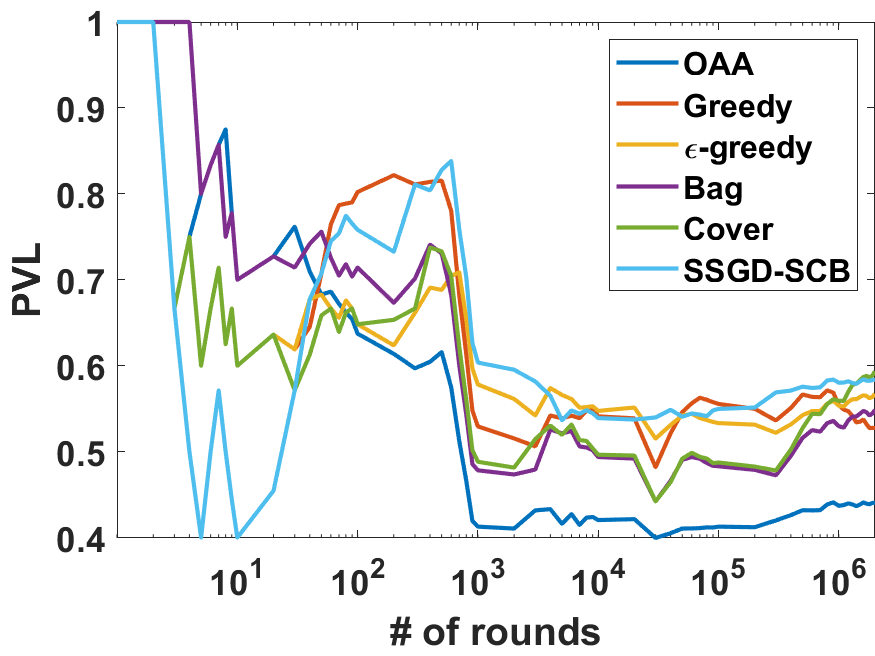}\\
    (a) MSLR
    \end{minipage}\hfill
    \begin{minipage}{0.33\textwidth}
    \centering
    \includegraphics[width=\textwidth]{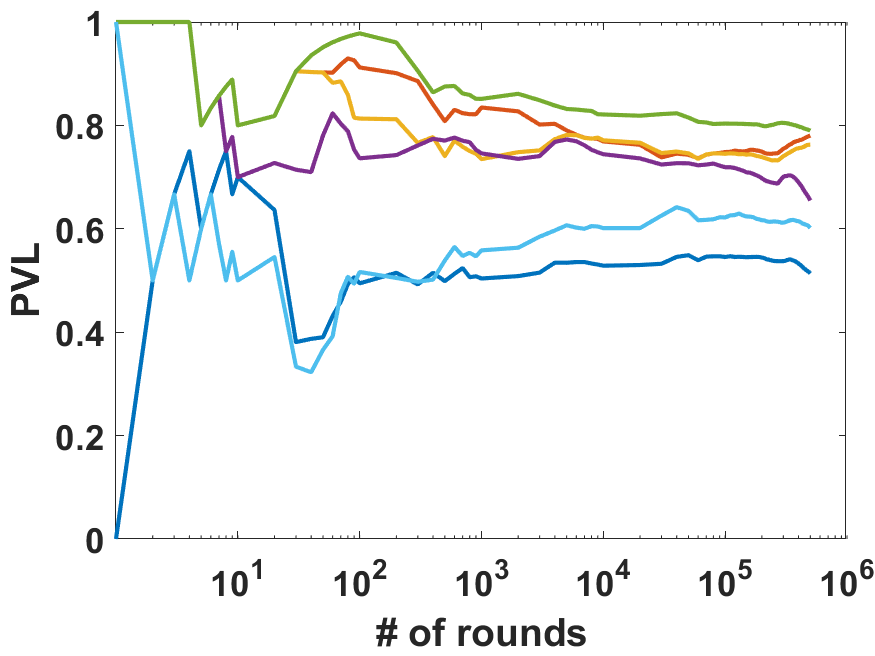}\\
    (b) Yahoo
    \end{minipage}\hfill
    \begin{minipage}{0.33\textwidth}
    \centering
    \includegraphics[width=\textwidth]{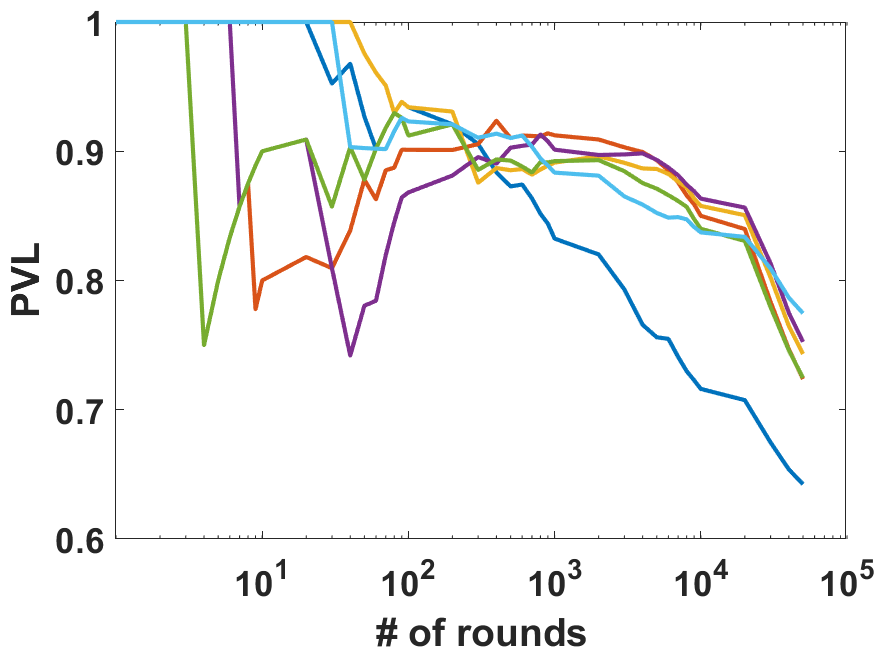}\\
    (c) CIFAR-10
    \end{minipage}\par
    \caption{\label{cdsgdresult}Progressive validation loss of SSGD-SCB with a linear reward function.}
    \end{figure}
	\newpage
	\section{Detailed Proof of Theorem \ref{converagerate}}
	\label{proof_lemma_nsgd}
	In this section, we give the proof of Theorem \ref{converagerate}.
	
	In the following proof, $\|\cdot\|$ denotes the l2 norm of vectors and spectral norm of matrices. For the sake of simplicity, let $x_{s,n}=x_{I(s,n)}$ and $\xi_{s,n}=\tilde{g}_{I(s,n)}+\mathcal{N}_s-\nabla F(x_{s,n})$. Then for better understanding, the conclusion of Lemma \ref{boundg} is simplified as:
	\begin{equation*}
	\begin{aligned}
	&\mathbb{E}\big[\xi_{s,n}+\nabla F(x_{s,n})\big]=\nabla F(x_{s,n}),\\
	&\|\xi_{s,n}\|^2\le \sigma_{s}^2,\\
	&r_{r}^2\sigma_{s}^2\le\mathbb{E}\|\xi_{s,n}\|^2\le \sigma_{s}^2,
	\end{aligned}
	\end{equation*} 
	where $r_{r}\in(0,1)$ and $\sigma_{s}=s^{\frac{\kappa}{2}}\sigma_{0}$.
	
	The proof of Theorem \ref{converagerate} consists of two steps. First, following the idea of Lemmas 15-18 in the paper of Noisy Stochastic Gradient \cite{ge2015escaping} under the relaxed assumption (The Remark of Lemma 16 in \cite{ge2015escaping}), Lemmas \ref{region2}-\ref{l11} show that in each stage, if the number of stage $s$ is large enough, with probability at least $1-\frac{\zeta}{2}$, $x$ will enter the strongly convex region at least once within $\Theta(s^{2\upsilon-\kappa}\log\frac{2}{\zeta})$ steps with $\|\nabla F(x)\|\le\mathcal{O}(s^{\frac{\kappa-\upsilon}{2}})$. 
	
	Then in Lemma \ref{local minimum}, we show that after entering the strongly convex region whose local minimizer is $x^{s'}$ with $\|\nabla F(x)\|\le\mathcal{O}({s'}^{\frac{\kappa-\upsilon}{2}})$, if stage $s'$ is large enough, with probability at least $1-\frac{\zeta}{2}$, for $\forall s>{s'}$ and $\forall n\in \{1,2,\dots,T_0 s^{2\upsilon}\}$, we have $\|x_{s,n}-x^{s'}\|^{2}\le\mathcal{O}( s^{{\kappa-\upsilon}}\log \frac{1}{\eta_{s}\zeta})$.
	
	\begin{lemma}
		\label{region2}
		With any point $x_{s,n}\in\mathbb{R}^{n_{x}}$ that is $\|\nabla F\left(x_{s,n}\right)\| \geq s^{\frac{\kappa-\upsilon}{2}}\sqrt{2\eta_0 \sigma_0^{2}L}$ where $\eta_{s}<\frac{1}{L}$ and $s^{\frac{\kappa-\upsilon}{2}}\sqrt{2\eta_0 \sigma_0^{2}L}\le \epsilon$, after one round we have:
		\begin{align*}
		\mathbb{E} \big[F\left(x_{s,n+1}\right)|x_{s,n}\big]-F\left(x_{s,n}\right) \le-\frac{L\eta_s^{2} \sigma_s^{2}}{2}.
		\end{align*}
		\begin{proof}
			\begin{align*} 
			&\mathbb{E} \big[F\left(x_{s,n+1}\right)|x_{s,n}\big]-F\left(x_{s,n}\right) \\
			\leq& \nabla F\left(x_{s,n}\right)^{T} \mathbb{E}\left[x_{s,n+1}-x_{s,n}|x_{s,n}\right]+\frac{L}{2} \mathbb{E}\big[\|x_{s,n+1}-x_{s,n}\|^2 |x_{s,n}\big] \\ 
			=&\nabla F\left(x_{s,n}\right)^{T} \mathbb{E}\big[-\eta_s\left(\nabla F\left(x_{s,n}\right)+\xi_{s,n}\right)|x_{s,n}\big]+\frac{L}{2} \mathbb{E}\big[\|-\eta_s\left(\nabla F\left(x_{s,n}\right)+\xi_{s,n}\right)\|^{2}|x_{s,n}\big] \\ 
			\le&-\left(\eta_s-\frac{L \eta_{s}^{2}}{2}\right)\left\|\nabla F\left(x_{s,n}\right)\right\|^{2}+\frac{L\eta_s^{2} \sigma_s^{2} }{2} \\ 
			\leq&-\frac{\eta_s}{2}\left\|\nabla F\left(x_{s,n}\right)\right\|^{2}+\frac{L\eta_s^{2} \sigma_s^{2} }{2}\\
			\le&-\frac{L\eta_s^{2} \sigma_s^{2}}{2}.
			\end{align*}
		\end{proof}
	\end{lemma}
	\begin{lemma}
		\label{secondorderexpansion}
		Let $\tilde{F}(x)$ be the local second-order approximation of $F(x)$ around $x_{s,n}$. Let $H(x)$ be the hessian matrix of $F(x)$ in x. With $\|\nabla F\left(x_{s,n}\right)\| < s^{\frac{\kappa-\upsilon}{2}}\sqrt{2\eta_0 \sigma_{0}^{2}L}<\epsilon$
		and $\lambda_{\min }\left(H\left(x\right)\right)=-\gamma_{0}\le\gamma$, with probability at least $1-\mathcal{O}(s^{-3\upsilon})$, the following inequalities hold for all $0\le t\le\Theta(s^\upsilon)$:
		\begin{align*}
		&\|\nabla\tilde{F}(\tilde{x}_{s,n+t})\|\le\mathcal{O}(s^{\frac{\kappa-\upsilon}{2}}\log\frac{1}{\eta_s}),\\
		&\|\tilde{x}_{s,n+t}-x_{s,n}\|\le\mathcal{O}(s^{\frac{\kappa-\upsilon}{2}}\log\frac{1}{\eta_s}).
		\end{align*}
		\begin{proof}
			
			
			Supposing that with $\Delta n$, Eq. (\ref{region_2_lower_bound}) is satisfied. 
			\begin{equation}
			\label{region_2_lower_bound}
			\sum_{\tau=0}^{\Delta n-2}(1+\eta_s\gamma_0)^{2\tau}\le\frac{n_{x}}{r_r^2\gamma_0\eta_s}\le\sum_{\tau=0}^{\Delta n-1}(1+\eta_s\gamma_0)^{2\tau}.
			\end{equation}
			
			According to assumption, $F(x)$ is L-smooth. Therefore, we have $\gamma_0\le L$. Based on this conclusion and Eq. (\ref{region_2_lower_bound}), with $\eta_s\le min(\frac{\sqrt{2}-1}{L},\frac{n_{x}}{r_r^2\gamma_0})$, we have
			\begin{equation}
			\label{region_2_upper_bound}
			\begin{aligned}
			\sum_{\tau=0}^{\Delta n-1}(1+\eta_s\gamma_0)^{2\tau}=&1+(1+\eta_s\gamma_0)^2\sum_{\tau=0}^{\Delta n-2}(1+\eta_s\gamma_0)^{2\tau}\le1+\frac{2n_{x}}{r_r^2\gamma_0\eta_s}\le\frac{3n_{x}}{r_r^2\gamma_0\eta_s}.
			\end{aligned}
			\end{equation}
			Based on Eqs. (\ref{region_2_lower_bound}) ,(\ref{region_2_upper_bound}) and $\sum_{\tau=0}^{\Delta n-1}(1+\eta_s\gamma_0)^{2\tau}=\frac{(1+\eta_s\gamma_0)^{2\Delta n}-1}{\eta_s^2\gamma_0^2+2\eta_s\gamma_0}$, we have
			\begin{equation}
			\label{ub_1_s}
			\frac{n_{x}(\eta_s\gamma_0+2)}{r_r^2}+1\le(1+\eta_s\gamma_0)^{2\Delta n}\le \frac{3r_r^2n_{x}(\eta_s\gamma_0+2)}{r_r^2}+1,
			\end{equation}
			Substitute $2^{\frac{L}{(\sqrt{2}-1)\gamma_0}\eta_s\gamma_0\Delta n}\le(1+\eta_s\gamma_0)^{\frac{1}{\eta_s\gamma_0}2\eta_s\gamma_0\Delta n}\le e^{2\eta_s\gamma_0\Delta n}$ into Eq. (\ref{ub_1_s}), we get $\Delta n=\Theta(s^\upsilon)$.
			
			For ease of presentation, we refer to the hessian matrix $H(x_{s,n})$ as $H_{s,n}$. With Eqs. (\ref{region_2_lower_bound}), (\ref{region_2_upper_bound}) and $\Delta n=\Theta(s^\upsilon)$, $\|\nabla\tilde{F}(\tilde{x}_{s,n+\Delta n})\|$ and $\|\tilde{x}_{s,n+\Delta n}-\tilde{x}_{s,n}\|$ can be bounded. 
			
			For $\nabla\tilde{F}(\tilde{x}_{s,n+\Delta n})$, we have 
			\begin{equation}
			\label{1itemgf}
			\begin{aligned}
			&\|\nabla\tilde{F}(\tilde{x}_{s,n+\Delta n})\|\\
			=&\|\nabla\tilde{F}(\tilde{x}_{s,n-1+\Delta n})+H_{s,n}(\tilde{x}_{s,n+\Delta n}-\tilde{x}_{s,n-1+\Delta n})\|\\
			=&\|\nabla\tilde{F}(\tilde{x}_{s,n-1+\Delta n})-\eta_sH_{s,n}(\nabla \tilde{F}(\tilde{x}_{s,n-1+\Delta n})+\xi_{s,n-1+\Delta n})\|\\
			=&\|(I-\eta_sH_{s,n})\nabla\tilde{F}(\tilde{x}_{s,n-1+\Delta n})-\eta_sH_{s,n}\xi_{s,n-1+\Delta n}\|\\
			=&\|(I-\eta_sH_{s,n})^{\Delta n}\nabla F(x_{s,n})-\eta_s H_{s,n}\sum_{\tau=0}^{\Delta n-1}(I-\eta_sH_{s,n})^{\Delta n-1-\tau}\xi_{s,n+\tau}\|\\
			\le&\|(I-\eta_sH_{s,n})^{\Delta n}\nabla F(x_{s,n})\|+\|\eta_s H_{s,n}\sum_{\tau=0}^{\Delta n-1}(I-\eta_sH_{s,n})^{\Delta n-1-\tau}\xi_{s,n+\tau}\|.\\
			\end{aligned}
			\end{equation}
			For the first item of Eq. (\ref{1itemgf}), based on Eq. (\ref{ub_1_s}), we have 
			\begin{equation*}
			\begin{aligned}
			\|(I-\eta_sH_{s,n})^{\Delta n}\nabla F(x_{s,n})\|\le
			(1+\eta_s\gamma_0)^{\Delta n}s^{\frac{\kappa-\upsilon}{2}}\sqrt{2\eta_0 \sigma_0^2L}
			=\mathcal{O}(s^{\frac{\kappa-\upsilon}{2}}).
			\end{aligned}
			\end{equation*}
			
			For the second item, we first have
			\begin{equation}
			\label{second_item_delta_tild_F}
			\begin{aligned}
			&\sum_{\tau=0}^{\Delta n-1}\|\eta_s H_{s,n}(I-\eta_sH_{s,n})^{\Delta n-1-\tau}\xi_{s,n+\tau}\|^2\\
			\le&\sigma_s^2\eta_s^2 \sum_{\tau=0}^{\Delta n-1}L^2(1+\eta_s\gamma_0)^{2\Delta n-2-2\tau}\\
			=&L^2\sigma_s^2\eta_s^2 \sum_{\tau=0}^{\Delta n-1}(1+\eta_s\gamma_0)^{2\tau}
			\overset{Eq. (\ref{region_2_upper_bound})}{\le}\mathcal{O}(s^{\kappa-\upsilon}).
			\end{aligned}
			\end{equation}
			Therefore, based on Azuma–Hoeffding inequality, with a small enough $\eta_s$, for each dimension of $\eta_s H_{s,n}\sum_{\tau=0}^{\Delta n-1}(I-\eta_sH_{s,n})^{\Delta n-1-\tau}\xi_{s,n+\tau}$, which is marked as $(\eta_s H_{s,n}\sum_{\tau=0}^{\Delta n-1}(I-\eta_sH_{s,n})^{\Delta n-1-\tau}\xi_{s,n+\tau})_i$, we have 
			\begin{equation*}
			\begin{aligned}
			&Pr(|(\eta_s H_{s,n}\sum_{\tau=0}^{\Delta n-1}(I-\eta_sH_{s,n})^{\Delta n-1-\tau}\xi_{s,n+\tau})_i|\ge \sigma_s\eta_s^{0.5}\log\frac{1}{\eta_s})\\
			\le&2e^{\frac{-2\sigma_s^2\eta_s\log^2\frac{1}{\eta_s}}{\sum_{\tau=0}^{\Delta n-1}(\eta_s H_{s,n}(1-\eta_sH_{s,n})^{\Delta n-1-\tau}\xi_{s,n+\tau})_i^2}}\\
			\le&2e^{\frac{-2\sigma_s^2\eta_s\log^2\frac{1}{\eta_s}}{\sum_{\tau=0}^{\Delta n-1}\|\eta_s H_{s,n}(1-\eta_sH_{s,n})^{\Delta n-1-\tau}\xi_{s,n+\tau}\|^2}}\\
			\le&2\eta_s^{\Omega(\log\frac{1}{\eta_s} )}
			\le\mathcal{O}(\eta_s^{4})
			=\mathcal{O}(s^{-4\upsilon}).
			\end{aligned}
			\end{equation*}
			Therefore, we have 
			\begin{equation}
			\label{final_sum_t_1}
			\begin{aligned}
			&Pr(|\eta_s H_{s,n}\sum_{\tau=0}^{\Delta n-1}(I-\eta_sH_{s,n})^{\Delta n-1-\tau}\xi_{s,n+\tau}|\ge \sigma_s\eta_s^{0.5}\log\frac{1}{\eta_s})\le\mathcal{O}(\eta_s^{4}).
			\end{aligned}
			\end{equation}
			The next step is to bound $\|\tilde{x}_{s,n+\Delta n}-x_{s,n}\|$:
			\begin{equation}
			\label{2itemgf}
			\begin{aligned}
			&\|\tilde{x}_{s,n+\Delta n}-x_{s,n}\| \\
			=&\|-\eta_s\sum_{\tau=0}^{\Delta n-1}\big[\nabla \tilde{F}(\tilde{x}_{s,n+\tau})+\xi_{s,n+\tau}\big]\|\\
			=&\|-\eta_s\sum_{\tau=0}^{\Delta n-1}(I-\eta_sH_{s,n})^\tau\nabla F(x_{s,n})+\eta_s^2H_{s,n}\sum_{\tau=1}^{\Delta n-1}\sum_{i=0}^{\tau-1}(I-\eta_sH_{s,n})^{\tau-1-i}\xi_{s,n+i}-\eta_s\sum_{\tau=0}^{\Delta n-1} \xi_{s,n+\tau}\|\\
			=&\|-\eta_s\sum_{\tau=0}^{\Delta n-1}(I-\eta_sH_{s,n})^\tau\nabla F(x_{s,n})+\eta_s^2H_{s,n}\sum_{\tau=0}^{\Delta n-2}\sum_{i=0}^{\Delta n-2-\tau}(I-\eta_sH_{s,n})^i\xi_{s,n+\tau}-\eta_s\sum_{\tau=0}^{\Delta n-1} \xi_{s,n+\tau}\|\\
			=&\|-\eta_s\sum_{\tau=0}^{\Delta n-1}(I-\eta_sH_{s,n})^\tau\nabla F(x_{s,n})-\eta_s\sum_{\tau=0}^{\Delta n-2}((I-\eta_sH_{s,n})^{\Delta n-\tau-1}-I)\xi_{s,n+\tau}-\eta_s\sum_{\tau=0}^{\Delta n-1} \xi_{s,n+\tau}\|\\
			\le&\|\eta_s\sum_{\tau=0}^{\Delta n-1}(I-\eta_sH_{s,n})^\tau\nabla F(x_{s,n})\|+\|\eta_s\sum_{\tau=0}^{\Delta n-1}(I-\eta_sH_{s,n})^{\Delta n-\tau-1}\xi_{s,n+\tau}\|.
			\end{aligned}
			\end{equation}
			For the first item of Eq. (\ref{2itemgf}), we have 
			\begin{equation*}
			\begin{aligned}
			&\|\eta_s\sum_{\tau=0}^{\Delta n-1}(I-\eta_sH_{s,n})^\tau\nabla F(x_{s,n})\|\\
			\le&\sqrt{2\sum_{\tau=0}^{\Delta n-1}(1+\eta_s\gamma_0)^{\tau} \eta_s^2 \sigma_{s}^{2}L}\le\sqrt{2\frac{(1+\eta_s\gamma_0)^{\Delta n}-1}{\eta_s\gamma_0} \eta_s^2 \sigma_{s}^{2}L}\overset{Eq. (\ref{ub_1_s})}{\le}\mathcal{O}(s^{\frac{\kappa-\upsilon}{2}}).
			\end{aligned}
			\end{equation*}
			For the second item of Eq. (\ref{2itemgf}), following the idea of Eq. (\ref{final_sum_t_1}), we first have 
			\begin{equation*}
			\begin{aligned}
			&\sum_{\tau=0}^{\Delta n-1}\|\eta_s(I-\eta_sH_{s,n})^{\Delta n-\tau-1}\xi_{s,n+\tau}\|^2\le\mathcal{O}(s^{\kappa-\upsilon}).
			\end{aligned}
			\end{equation*}
			We then have 
			\begin{equation}
			\label{final_sum_t_2}
			\begin{aligned}
			&Pr(|\eta_s\sum_{\tau=0}^{\Delta n-1}(I-\eta_sH_{s,n})^{\Delta n-\tau-1}\xi_{s,n+\tau}|\ge \sigma_s\eta_s^{0.5}\log\frac{1}{\eta_s})\le\mathcal{O}(s^{-4\upsilon}).
			\end{aligned}
			\end{equation}
			
			With Eqs. (\ref{final_sum_t_1}) and (\ref{final_sum_t_2}), we have 
			\begin{equation*}
			\begin{aligned}
			&Pr(\forall t<=\Delta n, |\eta_s H_{s,n}\sum_{\tau=0}^{t-1}(1-\eta_sH_{s,n})^{t-1-\tau}\xi_{s,n+\tau}|\ge \sigma_s\eta_s^{0.5}\log\frac{1}{\eta_s})\le\mathcal{O}(s^{-3\upsilon}),\\
			&Pr(\forall t<=\Delta n, |\eta_s\sum_{\tau=0}^{t-1}(1-\eta_sH_{s,n})^{t-\tau-1}\xi_{s,n+\tau}|\ge \sigma_s\eta_s^{0.5}\log\frac{1}{\eta_s})\le\mathcal{O}(s^{-3\upsilon}).\\
			\end{aligned}
			\end{equation*}
			Finally, with probability larger than $1-\mathcal{O}(s^{-3\upsilon})$, for all $t<\Theta(s^\upsilon)$, we have 
			\begin{align*}
			&\|\nabla\tilde{F}(\tilde{x}_{s,n+t})\|\le\mathcal{O}(s^{\frac{\kappa-\upsilon}{2}}\log\frac{1}{\eta_s}),\\
			&\|\tilde{x}_{s,n+t}-x_{s,n}\|\le\mathcal{O}(s^{\frac{\kappa-\upsilon}{2}}\log\frac{1}{\eta_s}).
			\end{align*}
		\end{proof}
	\end{lemma}
	\begin{lemma}
		\label{exactexpansion}
		Based on the definition of Lemma \ref{secondorderexpansion},
		for all $t<\Theta(s^{\upsilon})$, with probability at least $1-\mathcal{O}(\eta_{s}^2)$, we have:
		\begin{align*}
		\left\|\nabla F(x_{s,n+t})-\nabla\tilde{F}(\tilde{x}_{s,n+t})\right\|&\le \mathcal{O}(s^{\kappa-\upsilon}\log^2\frac{1}{\eta_s}),\\
		\left\|x_{s,n+t}-\tilde{x}_{s,n+t}\right\|&\le \mathcal{O}(s^{\kappa-\upsilon}\log^2\frac{1}{\eta_s}).\\
		\end{align*}
		\begin{proof}
			With $H_{s,n}=H(x_{s,n})$ and $H'_{s,n+t}=H_{s,n+t}-H_{s,n}$ , we have 
			\begin{equation*}
			\begin{aligned}
			&\nabla F(x_{s,n+t})\\
			=&(1-\eta_sH_{s,n})\nabla F(x_{s,n+t-1})-\eta_s H_{s,n}\xi_{s,n+t-1}
			-\eta_s H_{s,n+t-1}'
			(\nabla F(x_{s,n+t-1})+\xi_{s,n+t-1})+\theta_{s,n+t-1},\\
			&\triangle_{s,n+t}\\
			=&\nabla F(x_{s,n+t})-\nabla\tilde{F}(\tilde{x}_{s,n+t})\\
			=&(1-\eta_sH_{s,n})\triangle_{s,n+t-1}-\eta_sH_{s,n+t-1}'\big[\triangle_{s,n+t-1}+\nabla\tilde{F}(\tilde{x}_{s,n+t-1})+\xi_{s,n+t-1}\big]+\theta_{s,n+t-1},\\
			\end{aligned}
			\end{equation*}
			where
			\begin{equation*}
			\theta_{s,n+t-1}= \int_{0}^{1}\left[H\left(x_{s,n+t-1}+h\left(x_{s,n+t}-x_{s,n+t-1}\right)\right)-H_{s,n+t-1}\right] \mathrm{d} h \cdot\left(x_{s,n+t}-x_{s,n+t-1}\right).   
			\end{equation*}
			Then define events $\mathfrak{K}_{s,n+t}$ and $\mathfrak{E}_{s,n+t}$ as 
			\begin{equation*}
			\begin{aligned}
			&\mathfrak{K}_{s,n+t}=\big\{\forall\tau\le t,\|\nabla\tilde{F}(\tilde{x}_{s,n+\tau})\|\le\mathcal{O}(\sigma_s\eta_s^{0.5}\log\frac{1}{\eta_s}) ,\|\tilde{x}_{s,n+\tau}-x_{s,n}\|\le\mathcal{O}(\sigma_s\eta_s^{0.5}\log\frac{1}{\eta_s}) \big\},\\
			&\mathfrak{E}_{s,n+t}=\big\{\forall\tau\le t,\|\triangle_{s,n+\tau}\|\le\mathcal{O}(\mu \sigma_s^2\eta_s\log^2\frac{1}{\eta_s}) \big\}.
			\end{aligned}    
			\end{equation*}
			Conditioned on $\mathfrak{K}_{s,n+t-1}$ and $\mathfrak{E}_{s,n+t-1}$, with $t\le \Theta(s)$, we have 
			\begin{equation*}
			\begin{aligned}
			&\|x_{s,n+t}-\tilde{x}_{s,n+t}\|\\
			=&\|\big[x_{s,n}-\eta_s\sum_{\tau=0}^{t-1} [\nabla F(x_{s,n+\tau})+\xi_{s,n+\tau}]-x_{s,n}+\eta_s\sum_{\tau=0}^{t-1} [\nabla \tilde{F}(\tilde{x}_{s,n+\tau})+\xi_{s,n+\tau}]\|\\
			=&\eta_s\|\sum_{\tau=0}^{t-1} [\nabla F(x_{s,n+\tau})-\nabla\tilde{F}(\tilde{x}_{s,n+\tau})]\|=\eta_s\|\sum_{\tau=0}^{t-1} \Delta_{s,n+\tau}\|\le\mathcal{O}(\mu \sigma_s^2\eta_s\log^2\frac{1}{\eta_s}).
			\end{aligned}
			\end{equation*}
			By Hessian smoothness and conditioned on $\mathfrak{K}_{s,n+t-1}$ and $\mathfrak{E}_{s,n+t-1}$, we have 
			\begin{equation*}
			\begin{aligned}
			&\|H_{s,n+t-1}'\|=\|H_{s,n+t-1}-H_{s,n}\|\\
			\le&\rho\|x_{s,n+t-1}-x_{s,n}\|\\
			\le&\rho(\|x_{s,n+t-1}-\tilde{x}_{s,n+t-1}\|+\|\tilde{x}_{s,n+t-1}-x_{s,n}\|)\\
			\le&\mathcal{O}(\mu \sigma_s^2\eta_s\log^2\frac{1}{\eta_s})+\mathcal{O}(s^{\frac{\kappa-\upsilon}{2}}\log\frac{1}{\eta_s})=\mathcal{O}(s^{\frac{\kappa-\upsilon}{2}}\log\frac{1}{\eta_s}),\\
			&\|\theta_{s,n+t-1}\|\le\frac{\rho}{2}\|x_{s,n+t}-x_{s,n+t-1}\|^2\\
			\le&3\eta_s^2\frac{\rho}{2}\big[\|\Delta_{s,n+t-1}\|+\|\nabla\tilde{F}(\tilde{x}_{s,n+t-1})\|+\|\xi_{s,n+t-1}\|\big]^2\\
			=&\mathcal{O}(s^{\kappa-2\upsilon}).
			\end{aligned}
			\end{equation*}
			
			In summary, we have the following bound
			\begin{equation}
			\label{fourineq}
			\begin{aligned}
			\|(1-\eta_sH_{s,n})\triangle_{s,n+t-1}\|&\le\mathcal{O}(\mu s^{\kappa-\upsilon}\log^2\frac{1}{\eta_s}),\\
			\|\eta_sH_{s,n+t-1}'(\triangle_{s,n+t-1}+\nabla\tilde{F}(\tilde{x}_{s,n+t-1}))\|&\le\mathcal{O}(s^{\kappa-2\upsilon}\log^2\frac{1}{\eta_s}),\\
			\|\eta_sH_{s,n+t-1}'\xi_{s,n+t-1}\|&\le\mathcal{O}(s^{\kappa-1.5\upsilon}\log\frac{1}{\eta_s}),\\
			\|\theta_{s,n+t-1}\|&\le\mathcal{O}(s^{\kappa-2\upsilon}).\\
			\end{aligned}
			\end{equation}
			
			Let filtration $\mathfrak{F}_{s,n+t}=\sigma\left\{\xi_{s,1}, \cdots \xi_{s,n+t}\right\}$, based on Eq. (\ref{fourineq}) and $\mathbb{E}\big[\xi_{s,n+t-1}|\mathfrak{F}_{s,n+t-2}\big]=0
			$, we have 
			\begin{equation*}
			\begin{aligned}
			&\mathbb{E}\big[\|\triangle_{s,n+t}\|^21_{\mathfrak{K}_{s,n+t-1}\cap \mathfrak{E}_{s,n+t-1}}\big|\mathfrak{F}_{s,n+t-2}]\\
			\le&\big[(1+\eta_s\gamma_0)^2\|\triangle_{s,n+t-1}\|^2+\mathcal{O}(\mu s^{2\kappa-3\upsilon}\log^4\frac{1}{\eta_s})\big]1_{\mathfrak{K}_{s,n+t-1}\cap \mathfrak{E}_{s,n+t-1}}.
			\end{aligned}
			\end{equation*}
			Define
			\begin{equation*}
			\begin{aligned}
			G_{s,n,t}=(1+\eta_s\gamma_0)^{-2t}\big[\|\triangle_{s,n+t}\|^2+a s^{2\kappa-2\upsilon}\log^4\frac{1}{\eta_s} \big].
			\end{aligned}
			\end{equation*}
			The next step is to prove $G_t$ is supermartingale. When $\eta_s$ is small enough and $\alpha$ is large enough, we have 
			\begin{equation*}
			\begin{aligned}
			&\mathbb{E}\big[G_{s,n,t}1_{\mathfrak{K}_{s,n+t-1}\cap \mathfrak{E}_{s,n+t-1}}|\mathfrak{F}_{s,n+t-2}\big]\\
			\le&(1+\eta_s\gamma_0)^{-2t}\big[(1+\eta_s\gamma_0)^2\|\triangle_{s,n+t-1}\|^2+\mathcal{O}(s^{2\kappa-3\upsilon}\log^4\frac{1}{\eta_s})+as^{2\kappa-2\upsilon}\log^4\frac{1}{\eta_s}\big]1_{\mathfrak{K}_{s,n+t-1}\cap \mathfrak{E}_{s,n+t-1}}\\
			\le&(1+\eta_s\gamma_0)^{-2t}\big[(1+\eta_s\gamma_0)^2\|\triangle_{s,n+t-1}\|^2+(1+\mathcal{O}(\frac{\eta_s}{a}))a s^{2\kappa-2\upsilon}\log^4\frac{1}{\eta_s}\big]1_{\mathfrak{K}_{s,n+t-1}\cap \mathfrak{E}_{s,n+t-1}}\\
			\le&(1+\eta_s\gamma_0)^{-2t}\big[(1+\eta_s\gamma_0)^2\|\triangle_{s,n+t-1}\|^2+(1+\eta_s\gamma_0)^2a s^{2\kappa-2\upsilon}\log^4\frac{1}{\eta_s}\big]1_{\mathfrak{K}_{s,n+t-1}\cap \mathfrak{E}_{s,n+t-1}}\\
			\le&G_{s,n,t-1}1_{\mathfrak{K}_{s,n+t-1}\cap \mathfrak{E}_{s,n+t-1}}\le G_{s,n,t-1}1_{\mathfrak{K}_{s,n+t-2}\cap \mathfrak{E}_{s,n+t-2}}.
			\end{aligned}
			\end{equation*}
			Besides, $c_{s,n+t-1}$ is used to bound the following equation: 
			\begin{equation*}
			\begin{aligned}
			&|G_{s,n,t}1_{\mathfrak{K}_{s,n+t-1}\cap \mathfrak{E}_{s,n+t-1}}-\mathbb{E}\big[G_{s,n,t}1_{\mathfrak{K}_{s,n+t-1}\cap \mathfrak{E}_{s,n+t-1}}|\mathfrak{F}_{s,n+t-2}\big]\\
			\le&|\|\triangle_{s,n+t}\|^2-\mathbb{E}\big[\|\triangle_{s,n+t}\|^2\big|\mathfrak{F}_{s,n+t-2}]|1_{\mathfrak{K}_{s,n+t-1}\cap \mathfrak{E}_{s,n+t-1}}\\
			\le&\mathcal{O}(s^{2\kappa-2.5\upsilon}\mu\log^3\frac{1}{\eta_s})=c_{s,n+t}.
			\end{aligned}
			\end{equation*}
			Let 
			\begin{equation}
			\label{AHbound}
			\begin{aligned}
			A_{s,n+t}\le G_{s,n,t}1_{\mathfrak{K}_{s,n+t-1}\cap \mathfrak{E}_{s,n+t-1}}-G_{s,n,t-1}1_{\mathfrak{K}_{s,n+t-2}\cap \mathfrak{E}_{s,n+t-2}}\le B_{s,n+t}.
			\end{aligned}
			\end{equation}
			With the definition of $c_{s,n+t}$, we have $B_{s,n+t}-A_{s,n+t}\le \mathcal{O}(s^{2\kappa-2.5\upsilon}\mu\log^3\frac{1}{\eta_s})$. Based on the definition of $\mathfrak{K}_{s,n+t}$ and $\mathfrak{E}_{s,n+t}$, we have
			\begin{align*}
			&Pr(G_{s,n,t}1_{\mathfrak{K}_{s,n+t-1}\cap \mathfrak{E}_{s,n+t-1}}-G_{s,n,0}\ge\mathcal{O}(1)\sqrt{\sum_{\tau=1}^{t}c_{s,n+\tau}^2}\log\frac{1}{\eta_s})\\
			=&Pr(G_{s,n,t}1_{\mathfrak{K}_{s,n+t-1}\cap \mathfrak{E}_{s,n+t-1}}\ge \tilde{C}s^{2\kappa-2\upsilon}\mu\log^4\frac{1}{\eta_s})\\
			\le&\mathcal{O}(\eta_s^3).
			\end{align*}
			If $\mu>\tilde{C}$, we have
			\begin{align*}
			&Pr\left(\mathfrak{K}_{s,n+t-1}\cap \mathfrak{E}_{s,n+t-1} \cap\left\{\left\|\Delta_{s,n+t}\right\|^{2} \geq \mu^2s^{2\kappa-2\upsilon}\log^4\frac{1}{\eta_s}\right\}\right)\\
			\le&Pr(G_{s,n,t}1_{\mathfrak{K}_{s,n+t-1}\cap \mathfrak{E}_{s,n+t-1}}\ge \tilde{C}s^{2\kappa-2\upsilon}\mu\log^4\frac{1}{\eta_s})\\
			\le&\mathcal{O}(\eta_s^3).
			\end{align*}
			In equation mentioned above, $Pr\left(\mathfrak{E}_{s,n+t-1} \cap\left\{\left\|\Delta_{s,n+t}\right\|^{2} \geq \mu^2s^{2\kappa-2\upsilon}\log^4\frac{1}{\eta_s}\right\}\right)$ can be expressed as follows:
			\begin{align*} 
			& Pr\left(\mathfrak{E}_{s,n+t-1} \cap\left\{\left\|\Delta_{s,n+t}\right\|^{2} \geq \mu^2s^{2\kappa-2\upsilon}\log^4\frac{1}{\eta_s}\right\}\right) \\
			=& Pr\left(\mathfrak{K}_{s,n+t-1} \cap \mathfrak{E}_{s,n+t-1} \cap\left\{\left\|\Delta_{s,n+t}\right\|^{2} \geq \mu^2s^{2\kappa-2\upsilon}\log^4\frac{1}{\eta_s}\right\}\right)\\
			&+Pr\left(\overline{\mathfrak{K}}_{s,n+t-1} \cap \mathfrak{E}_{s,n+t-1} \cap\left\{\left\|\Delta_{s,n+t}\right\|^{2} \geq \mu^2s^{2\kappa-2\upsilon}\log^4\frac{1}{\eta_s}\right\}\right) \\ 
			\leq & \mathcal{O}\left(\eta_s^{3}\right)+P\left(\overline{\mathfrak{K}}_{s,n+t-1}\right) \leq \mathcal{O}\left(\eta_s^{3}\right) .
			\end{align*}
			Finally, $Pr\left(\overline{\mathfrak{E}}_{s,n+t}\right)$ can be bounded by
			\begin{align*}
			&Pr\left(\overline{\mathfrak{E}}_{s,n+t}\right)\\
			=&Pr\left(\mathfrak{E}_{s,n+t-1} \cap\left\{\left\|\Delta_{s,n+t}\right\|^{2} \geq \mathcal{O}(s^{2\kappa-2\upsilon}\mu^2\log^4\frac{1}{\eta_s})\right\}\right)+Pr\left(\overline{\mathfrak{E}}_{s,n+t-1}\right)\\
			\leq& \mathcal{O}\left(\eta_s^{3}\right)+Pr\left(\overline{\mathfrak{E}}_{s,n+t-1}\right).
			\end{align*}
			Because $P\left(\overline{\mathfrak{E}}_{s,n}\right)=0$, we have $P\left(\overline{\mathfrak{E}}_{s,n+\Delta n}\right) \leq \mathcal{O}\left(\eta_s^{2}\right)$.
		\end{proof}
	\end{lemma}
	\begin{lemma}
		\label{l10}
		Based on Lemmas \ref{secondorderexpansion} and \ref{exactexpansion}, for $\Delta n<\Theta(s^\upsilon)$, with probability $1-\mathcal{O}(\eta_{s}^2)$, we have:
		\begin{align*}
		\mathbb{E}F(x_{s,n+\Delta n})-\mathbb{E}F(x_{s,n})\le\mathcal{O}(-s^{\kappa-\upsilon}).
		\end{align*}
		\begin{proof}
			Since the Hessian matrix of $F(x)$ is $\rho$-Lipschitz, we have:
			\begin{align*}
			F(x_{s,n+t}) \leq& F(x_{s,n})+\nabla F(x_{s,n})^{T}\left(x_{s,n+t}-x_{s,n}\right)\\
			&+\frac{1}{2}(x_{s,n+t}-x_{s,n})^{T} H_{s,n}(x_{s,n+t}-x_{s,n})+\frac{\rho}{6}\left\|x_{s,n+t}-x_{s,n}\right\|^{3}.
			\end{align*}
			Let $\tilde{\delta}=\tilde{x}_{s,n+\Delta n}-x_{s,n}$ and $\delta=x_{s,n+\Delta n}-\tilde{x}_{s,n+\Delta n}$, we have:
			\begin{align*} 
			&F(x_{s,n+\Delta n})-F(x_{s,n}) \\
			\le&\left[\nabla F(x_{s,n})^{T}(\tilde{\delta}+\delta)+\frac{1}{2}(\tilde{\delta}+\delta)^{T} H_{s,n}(\tilde{\delta}+\delta)+\frac{\rho}{6}\|\tilde{\delta}+\delta\|^{3}\right]\\
			=&\left[ \nabla F(x_{s,n})^{T} \tilde{\delta}+\frac{1}{2} \tilde{\delta}^{T} H_{s,n} \tilde{\delta}\right]+\left[\nabla F(x_{s,n})^{T} \delta+\tilde{\delta}^{T} H_{s,n} \delta+\frac{1}{2} \delta^{T} H_{s,n} \delta+\frac{\rho}{6}\|\tilde{\delta}+\delta\|^{3}\right].
			\end{align*}
			Let $\mathfrak{E}_{s,n+t}=\left\{\forall \tau \leq t, \left\|x_{s,n+\tau}-\tilde{x}_{s,n}\right\| \leq \mathcal{O}\left(s^{\frac{\kappa-\upsilon}{2}}\log\frac{1}{\eta_s}\right),\left\|\tilde{x}_{s,n+\tau}-x_{s,n+\tau}\right\|\leq \mathcal{O}\left(s^{\kappa-\upsilon}\log^2\frac{1}{\eta_s}\right)\right\}$. According to Lemma \ref{exactexpansion}, we have $Pr(\mathfrak{E}_{s,n+t})\ge 1-\mathcal{O}(\eta_s^2)$. Denote $\tilde{\Lambda}=\nabla F(x_{s,n})^{T} \tilde{\delta}+\frac{1}{2} \tilde{\delta}^{T} H_{s,n} \tilde{\delta}$ and $\Lambda=\nabla F(x_{s,n})^{T} \delta+\tilde{\delta}^{T} H_{s,n} \delta+\frac{1}{2} \delta^{T} H_{s,n} \delta+\frac{\rho}{6}\|\tilde{\delta}+\delta\|^{3}$, we have
			\begin{equation*}
			\begin{aligned} 
			&\mathbb{E} F(x_{s,n+\Delta n})-F(x_{s,n})\\
			=&\mathbb{E}\left[F(x_{s,n+\Delta n})-F(x_{s,n})\right]1_{\mathfrak{E}_{s,n+\Delta n}}+\mathbb{E}\left[F(x_{s,n+\Delta n})-F(x_{s,n})\right] 1_{\overline{\mathfrak{E}}_{s,n+\Delta n}} \\ 
			\leq &\mathbb{E} \tilde{\Lambda}1_{\mathfrak{E}_{s,n+\Delta n}}+\mathbb{E} \Lambda 1_{\mathfrak{E}_{n+\Delta n}}+\mathbb{E}[F(x_{s,n+\Delta n})-F(x_{s,n})]1_{\overline{\mathfrak{E}}_{s,n+\Delta n}} \\ 
			=&\mathbb{E} \tilde{\Lambda}+\mathbb{E} \Lambda 1_{\mathfrak{E}_{s,n+\Delta n}}+\mathbb{E}\left[F\left(x_{s,n+\Delta n}\right)-F\left(x_{s,n}\right)\right] 1_{\overline{\mathfrak{E}}_{s,n+\Delta n}}-\mathbb{E} \tilde{\Lambda} 1_{\overline{\mathfrak{E}}_{s,n+\Delta n}}.
			\end{aligned}    
			\end{equation*}
			Since 
			\begin{equation*}
			\begin{aligned}
			&\tilde{\delta}=\tilde{x}_{s,n+\Delta n}-x_{s,n}\\
			=&-\eta_s\sum_{\tau=0}^{\Delta n-1}(I-\eta_sH_{s,n})^\tau\nabla F(x_{s,n})-\eta_s\sum_{\tau=0}^{\Delta n-1}(I-\eta_sH_{s,n})^{\Delta n-\tau-1}\xi_{s,n+\tau},
			\end{aligned}
			\end{equation*}
			$\mathbb{E} \tilde{\Lambda}$ can be bounded by 
			\begin{equation}
			\label{mostimportant}
			\begin{aligned} 
			\mathbb{E} \tilde{\Lambda}=&-\eta_s\nabla F(x_{s,n})^T\sum_{\tau=0}^{\Delta n-1}(I-\eta_sH_{s,n})^\tau\nabla F(x_{s,n})\\
			&+\frac{1}{2}\eta_s^2\nabla F(x_{s,n})^T\sum_{\tau=0}^{\Delta n-1}(I-\eta_sH_{s,n})^\tau H_{s,n}\sum_{\tau=0}^{\Delta n-1}(I-\eta_sH_{s,n})^\tau\nabla F(x_{s,n})\\
			&+\mathbb{E}\big[\frac{1}{2}\eta_s^2\sum_{\tau=0}^{\Delta n-1}\xi_{s,n+\tau}^T(I-\eta_sH_{s,n})^{\Delta n-\tau-1}H_{s,n}\sum_{\tau=0}^{\Delta n-1}(I-\eta_sH_{s,n})^{\Delta n-\tau-1}\xi_{s,n+\tau}\big].\\
			\end{aligned}    
			\end{equation}
			When $\eta_0$ is small enough, the summation of the first two terms in Eq. (\ref{mostimportant}) is smaller than 0 (Section 6.5 in \cite{jain2017non}), with $\eta_s\le\frac{1}{L}$, we get
			\begin{equation*}
			\begin{aligned}
			\mathbb{E} \tilde{\Lambda}\le&\frac{1}{2} \sum_{i=1}^{n_{x}} \lambda_{i} \sum_{\tau=0}^{T-1}\left(1-\eta_s \lambda_{i}\right)^{2 \tau} \eta_s^{2} \mathbb{E}(\xi_{s,n+\tau}^T\xi_{s,n+\tau}) \\ 
			\leq & \frac{\eta_s^{2} \sigma_s^{2}}{2}\left[\sum_{\lambda_i\in\{\lambda_1,\dots,\lambda_{n_{x}}\}\backslash-\gamma_0}\lambda_i\frac{1-(1-\lambda_i\eta_s)^{2\tau}}{2\lambda_i\eta_s-\lambda_i^2\eta_s^2}-\gamma_{0}r_r^2 \sum_{\tau=0}^{T-1}\left(1+\eta_s \gamma_{0}\right)^{2 \tau}\right]\\
			\leq & \frac{\eta_s^{2} \sigma_s^{2}}{2}\left[\frac{n_{x}-1}{\eta_s}-\gamma_{0}r_r^2 \sum_{\tau=0}^{T-1}\left(1+\eta_s \gamma_{0}\right)^{2 \tau}\right]
			\overset{Eq. (\ref{region_2_lower_bound})}{\leq}-\frac{\eta_s\sigma_s^{2}}{2}.
			\end{aligned}
			\end{equation*}
			For other terms, since $F(\cdot)$ is a bounded function,    we have:
			\begin{align*}
			&\mathbb{E} \Lambda 1_{\mathfrak{E}_{s,n+\Delta n}}=\mathbb{E}\left[\nabla F(x_{s,n})^{T} \delta+\tilde{\delta}^{T} H_{s,n} \delta+\frac{1}{2} \delta^{T} H_{s,n} \delta+\frac{\rho}{6}\|\tilde{\delta}+\delta\|^{3}\right] 1_{\mathfrak{E}_{s,n+\Delta n}}\\
			\le&\mathcal{O}(s^{1.5(\kappa-\upsilon)}\log^2\frac{1}{\eta_s})+\mathcal{O}(s^{1.5(\kappa-\upsilon)}\log^3\frac{1}{\eta_s})+\mathcal{O}(s^{2\kappa-2\upsilon}\log^4\frac{1}{\eta_s})+\frac{6}{\rho}\mathcal{O}(s^{1.5(\kappa-\upsilon)}\log^3\frac{1}{\eta_s})\\
			\le&\mathcal{O}(s^{1.5(\kappa-\upsilon)}\log^3\frac{1}{\eta_s}).
			\end{align*}
			With the bounded $F(x)$ and 
			\begin{equation*}
			\begin{aligned}
			&\|\eta_s\sum_{\tau=0}^{\Delta n-1}(I-\eta_sH_{s,n})^{\Delta n-\tau-1}\xi_{s,n+\tau}\|\le\sqrt{2\eta_s\sum_{\tau=0}^{\Delta n-1}(1+\eta_s\gamma_0)^{\tau}s^{\kappa}\sigma_0^2 }
			\overset{Eq. (\ref{region_2_upper_bound})}{=}\mathcal{O}(s^{\frac{\kappa}{2}}),
			\end{aligned}
			\end{equation*}
			we have
			\begin{align*}
			& \mathbb{E}\left[F\left(x_{s,n+\Delta n}\right)-F\left(x_{0}\right)\right] 1_{\overline{\mathfrak{E}}_{s,n+\Delta n}}-\mathbb{E} \tilde{\Lambda} 1_{\overline{\mathfrak{E}}_{s,n+\Delta n}} \\
			=& \mathbb{E}\left[F\left(x_{s,n+\Delta n}\right)-F\left(x_{0}\right)\right] 1_{\overline{\mathfrak{E}}_{s,n+\Delta n}}-\mathbb{E}\left[\nabla F(x_{s,n})^{T} \tilde{\delta}+\frac{1}{2} \tilde{\delta}^{T} H_{s,n} \tilde{\delta}\right] 1_{\overline{\mathfrak{E}}_{s,n+\Delta n}}\\
			\le& \mathbb{E}\left[F\left(x_{s,n+\Delta n}\right)-F\left(x_{0}\right)\right] 1_{\overline{\mathfrak{E}}_{s,n+\Delta n}}+\mathbb{E}\big[\mathcal{O}(s^{\kappa}) 1_{\overline{\mathfrak{E}}_{s,n+\Delta n}}\big]\\
			\leq& \mathcal{O}\left(s^{\kappa-2\upsilon}\right).
			\end{align*}
			Finally, we have:
			\begin{align*}
			\mathbb{E}F(x_{s,n+\Delta n})-\mathbb{E}F(x_{s,n})\le\mathcal{O}(-s^{\kappa-\upsilon}).
			\end{align*}
		\end{proof}
	\end{lemma}
	Finally, we can finish the first step with Lemma \ref{l11}.
	\begin{lemma}
		\label{l11}
		Under Assumption \ref{non_convex_PCB}, if SSGD-SCB employs  $\tilde{g}_{I(s,n),a_I(s,n)}+\mathcal{N}_s$ to update $x_{I(s,n)}$, in stage $s\in\{1,2,\dots,\mathcal{S}\}$ of SSGD-SCB, for any $\zeta\in(0,1)$, there exists a large enough stage $s$ such that with probability at least $1-\frac{\zeta}{2}$, $x$ will enter the strongly convex region at least once within $\Theta(s^{\kappa-2\upsilon}\log\frac{2}{\zeta})$ steps. When $x$ enters the strongly convex region in round $T_{s,1}$, $\|\nabla F(x_{s,T_{s,1}})\|\le\mathcal{O}(s^{\frac{\kappa-\upsilon}{2}})$.
		\begin{proof}
			Let 
			\begin{equation*}
			\begin{aligned}
			&\mathcal{L}_{s,1}=\{x: min(\|\nabla F(x)\|,\epsilon) \geq s^{\frac{\kappa-\upsilon}{2}}\sqrt{2\eta_0 \sigma_{0}^{2}L}\},\\
			&\mathcal{L}_{s,2}=\{x: \|\nabla F(x)\| \leq s^{\frac{\kappa-\upsilon}{2}}\sqrt{2\eta_0 \sigma_{0}^{2}L}<\epsilon,\lambda_{\min }\left(H\left(x\right)\right)\le\gamma \},\\
			&\mathcal{L}_{s,3}=\{x: x\notin\mathcal{L}_{s,1},x\notin\mathcal{L}_{s,2}\}.
			\end{aligned}
			\end{equation*}
			Define a stochastic process $\{\tau_{s,i}\}$ s.t. $\tau_{s,1}=1$, and
			\begin{align*}
			\tau_{s,i+1}=\left\{\begin{array}{ll}{\tau_{s,i}+1}, & {x_{s,\tau_{s,i}} \in \mathcal{L}_{s,1} \cup \mathcal{L}_{s,3}}; \\ {\tau_{s,i}+T\left(x_{\tau_{s,i}}\right)}, & {x_{s,\tau_{s,i}} \in \mathcal{L}_{s,2},}\end{array}\right.
			\end{align*}
			where $T(x_{s,\tau_{s,i}})=\Theta(s^\upsilon)$.
			
			According to Lemmas \ref{region2} and \ref{l10}, we have:
			\begin{align*}
			&\mathbb{E}\left[F\left(x_{s,\tau_{s,i+1}}\right)-F\left(x_{s,\tau_{s,i}}\right) | x_{s,\tau_{s,i}} \in \mathcal{L}_{s,1}\right]\leq -\Omega\left(s^{\kappa-2\upsilon}\right), \\
			&\mathbb{E}\left[F\left(x_{s,\tau_{s,i+1}}\right)-F\left(x_{s,\tau_{s,i}}\right) | x_{s,\tau_{s,i}} \in \mathcal{L}_{s,2}\right]\leq-\Omega(s^{\kappa-\upsilon}).
			\end{align*}
			Therefore, we have 
			\begin{align*}
			\mathbb{E}\left[F\left(x_{s,\tau_{s,i+1}}\right)-F\left(x_{s,\tau_{s,i}}\right) | x_{s,\tau_{s,i}} \notin \mathcal{L}_{s,3}\right] \leq-\left(\tau_{s,i+1}-\tau_{s,i}\right) \Omega(s^{\kappa-2\upsilon}).
			\end{align*}
			Define event $\mathfrak{E}_{s,i}=\left\{\nexists j \leq i, x_{s,\tau_{s,j}} \in \mathcal{L}_{s,3}\right\}$. We have 
			\begin{align*} 
			&\mathbb{E} F\left(x_{s,\tau_{s,i+1}}\right) 1_{\mathfrak{E}_{s,i+1}}-\mathbb{E} F\left(x_{s,\tau_{s,i}}\right) 1_{\mathfrak{E}_{s,i}} \\
			=&\mathbb{E} F\left(x_{s,\tau_{s,i+1}}\right) 1_{\mathfrak{E}_{s,i+1}}-\mathbb{E} F\left(x_{s,\tau_{s,i+1}}\right) 1_{\mathfrak{E}_{s,i}}+\mathbb{E} F\left(x_{s,\tau_{s,i+1}}\right) 1_{\mathfrak{E}_{s,i}}-\mathbb{E} F\left(x_{s,\tau_{s,i}}\right) 1_{\mathfrak{E}_{s,i}} \\
			\leq& -B \cdot \mathbb{E}\big[1_{\mathfrak{E}_{s,i+1}}-1_{\mathfrak{E}_{s,i}}\big]+\mathbb{E}\left[F\left(x_{\tau_{s,i+1}}\right)-F\left(x_{\tau_{s,i}}\right) | \mathfrak{E}_{s,i}\right] \cdot Pr\left(\mathfrak{E}_{s,i}\right) \\ 
			\le&  -B \cdot \mathbb{E}\big[1_{\mathfrak{E}_{s,i+1}}-1_{\mathfrak{E}_{s,i}}\big]-\left(\tau_{s,i+1}-\tau_{s,i}\right) \Omega\left(s^{\kappa-2\upsilon}\right) Pr\left(\mathfrak{E}_{s,i}\right).
			\end{align*}
			Based on the definition of $\mathfrak{E}_{s,i}$, $Pr(\mathfrak{E}_{s,i})\le Pr(\mathfrak{E}_{s,i-1})$. By summing up over $i$, we have:
			\begin{align*}
			&\mathbb{E} F\left(x_{s,\tau_{i+1}}\right) 1_{\mathfrak{E}_{s,i+1}}-F\left(x_{s,1}\right)
			\leq B-(\tau_{s,i+1}-1) \Omega\left(s^{\kappa-2\upsilon}\right) Pr\left(\mathfrak{E}_{s,i}\right).
			\end{align*}
			Since $\mathbb{E} F(x_{s,\tau_{s,i}}) 1_{\mathfrak{E}_{s,i}}-F(x_{s,1})\ge-2B$, when $\tau_{s,i}=\Theta(\frac{6B}{s^{\kappa-2\upsilon}})$, $Pr(\mathfrak{E}_{s,i})$ will be smaller than $\frac{1}{2}$. By repeating this process $\log 2/\zeta$ times, the algorithm will enter $\mathcal{L}_{s,3}$ at least once with probability at least $1-\frac{\zeta}{2}$. Note that after entering $\mathcal{L}_{s,3}$ in round $T_{s,1}$, $\|\nabla F(x_{s,T_{s,1}})\|\le\mathcal{O}(s^{\frac{\kappa-\upsilon}{2}})$.
		\end{proof}
	\end{lemma}
	Based on Lemma \ref{l11}, we are able to complete the proof of Theorem \ref{converagerate} with Lemma \ref{local minimum}.
	\begin{lemma}
		\label{local minimum}
		If in the $T_{s',1}^{th}$ round of ${s'}^{th}$ stage, $x$ enters the $2\delta$ strongly convex region of $F(x)$ whose local minimizer is $x^{s'}$ and is captured by the local minimizer in terms of $\|\nabla F(x_{s',T_{s',1}})\|\le\mathcal{O}(s'^{\frac{\kappa-\upsilon}{2}})$. If the stage $s'$ is large enough, $\upsilon\ge 1$ and $\kappa\le \upsilon$, with probability at least $\frac{1}{2\prod_{i=1}^{s'-1}(1-\frac{1}{i^2})}$,we have 
		\begin{equation*}
		\forall s\ge s'+1\ and\ \forall n\in [T_0 s^2], \quad\|x_{s,n}-x^{s'}\|^{2}\le\mu^2 s^{\kappa-\upsilon} \eta_0 \sigma_0^2\log \frac{1}{\eta_{s} \zeta}<\delta^2.
		\end{equation*}
		\begin{proof}
			According to Lemma \ref{l11}, $T_{s',1}\le\mathcal{O}(s'^{\kappa-2\upsilon}\log\frac{2}{\zeta})$. Let $T_{s,2}=\lceil\frac{\eta_s\log 2}{\alpha}\rceil$. If $s'$ is large enough, we have
			\begin{itemize}
			    \item $1+T_0s'^{2\upsilon}-T_{s',1}-T_{s',2}\ge0$.
			    \item $\forall s>s',\ T_0s^{2\upsilon}-T_{s,2}\ge0$.
			    \item $\forall s>s',\ (\frac{s}{s+1})^{\kappa-\upsilon}<(\frac{s'}{s'+1})^{\kappa-\upsilon}< 1.5$.
			    \item $\forall s\ge s'$, 
			    \begin{equation}
			    \label{l13begin}
			        \|\nabla F(x_{s',T_{s',1}})\|\le\alpha\sqrt{\frac{2}{2+(\frac{s'}{s'+1})^{\kappa-\upsilon}}}\mu s^{\frac{\kappa-\upsilon}{2}}\sqrt{\eta_0}\sigma_0\log^{0.5}\frac{1}{\eta_{s}\zeta} <\alpha\delta.
		       \end{equation}
			\end{itemize}
			Based on the property of $\alpha$-strongly convex, with $\|\nabla F(x_{s',T_{s',1}})\|^{2}\le\mathcal{O}(s'^{\frac{\kappa-\upsilon}{2}})$, Eq. (\ref{l13begin}) means that
			\begin{equation}
			\label{zero_bound}
			    \|x_{s',T_{s',1}}-x^{s'}\|\le\sqrt{\frac{2}{2+(\frac{s'}{s'+1})^{\kappa-\upsilon}}}\mu s'^{\frac{\kappa-\upsilon}{2}}\sqrt{\eta_0}\sigma_0\log^{0.5}\frac{1}{\eta_{s'}\zeta} <\delta.
			\end{equation}
			Then with $b\in\{1,T_{s',1}\}$, we define event $\mathfrak{E}_{s,b,t}$ as
			\begin{itemize}
			    \item $\mbox{If }t\in[0, T_{s,2}),\forall \tau \in [b,b+t],\left\|x_{s,\tau}-x^{s'}\right\| \leq \mu s^{\frac{\kappa-\upsilon}{2}}\sqrt{\eta_0}\sigma_0\log^{0.5}\frac{1}{\eta_s\zeta} <\delta$,
			    \item $\mbox{If }t\in[T_{s,2},T_0s^{2\upsilon}-b+1], \forall \tau \in [b+T_{s,2},b+t],$
			    \begin{equation*}
			        \left\|x_{s,\tau}-x^{s'}\right\| \leq \sqrt{\frac{2}{2+(\frac{s'}{s'+1})^{\kappa-\upsilon}}}\mu (s+1)^{\frac{\kappa-\upsilon}{2}}\sqrt{\eta_0}\sigma_0\log^{0.5}\frac{1}{\eta_{s+1}\zeta}<\delta
			    \end{equation*}
			\end{itemize}
			Define $\mathfrak{F}_{s,b+t}=\{\xi_{s,1},\dots,\xi_{s,b+t}\}$. When $\eta_s\le\frac{\alpha}{L^2}$ , we have
			\begin{align}
			&\mathbb{E}\left[\left\|x_{s,b+t}-x^{s'}\right\|^{2}1_{\mathfrak{E}_{s,b,t-1}} | \mathfrak{F}_{s,b+t-2}\right]\nonumber\\
			=&\mathbb{E}\left[\left\|x_{s,b+t-1}-\eta_s\left(\nabla F\left(x_{s,b+t-1}\right)+\xi_{s,b+t-1}\right)-x^{s'}\right\|^{2} | \mathfrak{F}_{s,b+t-2}\right] 1_{\mathfrak{E}_{s,b,t-1}}\nonumber\\
			\le&\left[\left\|x_{s,b+t-1}-x^{s'}\right\|^{2}-2 \eta_s \nabla F\left(x_{s,b+t-1}\right)^{T}\left(x_{s,b+t-1}-x^{s'}\right)+\eta_s^{2}\left\|\nabla F\left(x_{s,b+t-1}\right)\right\|^{2}+\eta_s^{2} \mathbb{E}[\|\xi_{s,b+t-1}\|^2]\right] 1_{\mathfrak{E}_{s,b,t-1}}\nonumber\\
			\leq&\left[\left(1-2 \eta_s \alpha+\eta_s^{2} L^{2}\right)\left\|x_{s,b+t-1}-x^{s'}\right\|^{2}+\eta_s^{2} \mathbb{E}[\|\xi_{s,b+t-1}\|^2]\right] 1_{\mathfrak{E}_{s,b,t-1}}\nonumber\\
			\leq&[(1-\eta_s \alpha)\|x_{s,b+t-1}-x^{s'}\|^{2}+s^{\kappa-2\upsilon}\eta_0^2\sigma_0^2] 1_{\mathfrak{E}_{s,b,t-1}}\label{abc}.
			\end{align}
			
			Therefore, we have 
			\begin{align*}
			&\left[\mathbb{E}\left[\left\|x_{s,b+t}-x^{s'}\right\|^{2} | \mathfrak{F}_{s,b+t-2}\right]-\frac{\eta_s^2
				\sigma_s^2}{\alpha}\right] 1_{\mathfrak{E}_{s,b,t-1}}\\
			\leq&(1-\eta_s \alpha)\left[\left\|x_{s,b+t-1}-x^{s'}\right\|^{2}-\frac{\eta_s
				\sigma_s^2}{\alpha}\right] 1_{\mathfrak{E}_{s,b,t-1}}.
			\end{align*}
			Define $G_{s,b,t}=(1-\eta_s \alpha)^{-t}\left(\left\|x_{s,b+t}-x^{s'}\right\|^{2}-\frac{\eta_s
				\sigma_s^2}{\alpha}\right)
			$, we have 
			\begin{align*}
			\mathbb{E}\left[G_{s,b,t} 1_{\mathfrak{E}_{s,b,t-1}} | \mathfrak{F}_{s,b+t-2}\right] \leq G_{s,b,t-1} 1_{\mathfrak{E}_{s,b,t-1}} \leq G_{s,b,t-1} 1_{\mathfrak{E}_{s,b,t-2}}.
			\end{align*}
			Furthermore, conditioned on $\mathfrak{F}_{s,b+t-2}$, with a large enough $s$, we have 
			\begin{equation*}
			\begin{aligned} 
			&\left|G_{s,b,t} 1_{\mathfrak{E}_{s,b,t-1}}-\mathbb{E}\left[G_{s,b,t} 1_{\mathfrak{E}_{s,b,t-1}} | \mathfrak{F}_{s,b+t-2}\right]\right| \\
			\leq &(1-\eta_s\alpha)^{-t}\big[2\eta_s \|\xi_{s,b+t-1}\|\|x_{s,b+t-1}-x^{s'}-\eta_s\nabla F\left(x_{s,b+t-1}\right)\|+\eta_s^2\xi_{s,b+t-1}^2+\eta_s^2\sigma_s^2\big]\\
			\leq &(1-\eta_s\alpha)^{-t}\big[2s^{\frac{\kappa}{2}-\upsilon}\eta_0 \sigma_0(\|x_{s,b+t-1}-x^{s'}\|+\|s^{-\upsilon}\eta_0\nabla F\left(x_{s,b+t-1}\right)\|)+2s^{\kappa-2\upsilon}\eta_0^2 \sigma_0^2\big]\\
			\leq &(1-\eta_s\alpha)^{-t} \cdot 8\mu s^{\kappa-1.5\upsilon}\eta_0^{1.5} \sigma_0^{2}\log^{0.5}\frac{1}{\eta_s\zeta}.
			\end{aligned}    
			\end{equation*}
			Similar to Eq. (\ref{AHbound}), we define 
			\begin{equation}
			\label{escape_convex_region}
			c_{s,t}=16\mu s^{\kappa-1.5\upsilon}\eta_0^{1.5} \sigma_0^{2}\log^{0.5}\frac{1}{\eta_s\zeta} \sqrt{\sum_{\tau=1}^{t}(1-\eta_s\alpha)^{-2 \tau}}.
			\end{equation}
			
			According to Azuma's inequality, with probability smaller than $\mathcal{O}(\eta_s^5)$, we have
			\begin{equation*}
			\begin{aligned} 
			G_{s,b,t} 1_{\mathfrak{E}_{s,b,t-1}}>\sqrt{10}c_{s,t} \log ^{\frac{1}{2}}\left(\frac{1}{\eta_s \zeta}\right)+G_{s,b,0},
			\end{aligned}
			\end{equation*}
			which can be rewritten as 
			\begin{equation*}
			\begin{aligned} 
			\left\|x_{s,b+t}-x^{s'}\right\|^{2}\ge(1-\eta_s\alpha)^t(\sqrt{10} c_{s,t} \log ^{\frac{1}{2}}\left(\frac{1}{\eta_s \zeta}\right)+G_{s,b,0})+\frac{\eta_s
				\sigma_s^2}{\alpha}.
			\end{aligned}
			\end{equation*}
			If the stage number is large enough, we have
			\begin{equation*}
			\begin{aligned} 
			(1-s^{-1}\eta_0\alpha)^t c_{s,t}\le\frac{16}{\alpha}\mu s^{\kappa-\upsilon}\eta_0 \sigma_0^{2}\log^{0.5}\frac{1}{\eta_s\zeta}.
			\end{aligned}
			\end{equation*}
			Then with $1-\frac{2(\frac{s'}{s'+1})^{\kappa-\upsilon}}{2+(\frac{s'}{s'+1})^{\kappa-\upsilon}}>0$, let $\mu^2=(\frac{s'}{s'+1})^{\kappa-\upsilon}(\frac{16\sqrt{10}}{\alpha}\mu+\frac{2}{2+(\frac{s'}{s'+1})^{\kappa-\upsilon}}\mu^2+\frac{1}{\alpha})$. We further get   
			\begin{equation*}
			\begin{aligned} 
			&(1-s^{-1}\eta_0\alpha)^t(\sqrt{10} c_{s,t} \log ^{\frac{1}{2}}\left(\frac{1}{\eta_s \zeta}\right)+G_{s,b,0})+\frac{\eta_s
				\sigma_s^2}{\alpha}\\
			\le&(1-s^{-1}\eta_0\alpha)^t(\sqrt{10} c_{s,t} \log ^{\frac{1}{2}}\left(\frac{1}{\eta_s \zeta}\right)+G_{s,b,0})+\frac{\eta_s
				\sigma_s^2}{\alpha}\log\frac{1}{\eta_s\zeta}\\
			\le&(\frac{16\sqrt{10}}{\alpha} \mu s^{\kappa-\upsilon}\eta_0 \sigma_0^{2}\log \frac{1}{\eta_s \zeta}+(1-s^{-1}\eta_0\alpha)^tG_{s,b,0})+\frac{\eta_s
				\sigma_s^2}{\alpha}\log\frac{1}{\eta_s\zeta}\\
			\le&(\frac{s'}{s'+1})^{\kappa-\upsilon}(\frac{16\sqrt{10}}{\alpha}+\frac{2}{2+(\frac{s'}{s'+1})^{\kappa-\upsilon}}\mu) \mu s^{\kappa-\upsilon}\eta_0 \sigma_0^{2}\log \frac{1}{\eta_s \zeta}+\frac{(\frac{s'}{s'+1})^{\kappa-\upsilon}\eta_s
				\sigma_s^2}{\alpha}\log\frac{1}{\eta_s\zeta}\\
			=&\mu^2 s^{\kappa-\upsilon}\eta_0 \sigma_0^{2}\log\frac{1}{\eta_s\zeta}.
			\end{aligned}
			\end{equation*}
			Therefore, for $t<T_{s,2}$, we have 
			\begin{align}
			\label{stage_1_ub_p}
			Pr\left(\mathfrak{E}_{s,b,t-1} \cap\left\{\left\|x_{s,b+t}-x^{s'}\right\|^{2}>\mu^2 s^{\kappa-\upsilon}\eta_0 \sigma_0^{2}\log\frac{1}{\eta_s\zeta}\right\}\right) \leq \mathcal{O}(\eta_s^5).
			\end{align}
			Note that when $t \in \{T_{s,2},T_{s,2}+1,\dots,T_0 s^{2}-T_{s,1}+1\}$, with 
			\begin{equation*}
			    \forall x\ge1,\left(1-\frac1x\right)^x\leq e^{-1},
			\end{equation*}
			if $s$ is large enough and $T_{s,2}=\lceil\frac{\eta_s\log 2}{\alpha}\rceil$, we get 
			\begin{equation}
			\label{laststep_convex_stage}
			\begin{aligned} 
			&(1-s^{-\upsilon}\eta_0\alpha)^t(\sqrt{10} c_{s,t} \log ^{\frac{1}{2}}\left(\frac{1}{\eta_s \zeta}\right)+G_{s,n,0})+\frac{\eta_s
				\sigma_s^2}{\alpha}\\
			\le&(\frac{16\sqrt{10}}{\alpha}+(\frac{1}{e})^{\frac{T_{s,2}\alpha}{\eta_s}}\frac{2}{2+(\frac{s'}{s'+1})^{\kappa-\upsilon}}\mu) \mu s^{\kappa-\upsilon}\eta_0\sigma_0^{2}\log\frac{1}{\eta_s\zeta}+\frac{\eta_s\sigma_s^2}{\alpha}\log\frac{1}{\eta_s\zeta}\\
			\le&(\frac{16\sqrt{10}}{\alpha}+\frac{1}{2+(\frac{s'}{s'+1})^{\kappa-\upsilon}}\mu) \mu (s+1)^{\kappa-\upsilon}\frac{s^{\kappa-\upsilon}}{(s+1)^{\kappa-\upsilon}}\eta_0\sigma_0^{2}\log\frac{1}{\eta_{s+1}\zeta}+\frac{s^{\kappa-\upsilon}}{(s+1)^{\kappa-\upsilon}}\frac{\eta_{s+1}\sigma_{s+1}^2}{\alpha}\log\frac{1}{\eta_{s+1}\zeta}\\
			=&(\frac{s'}{s'+1})^{\kappa-\upsilon}(\frac{16\sqrt{10}}{\alpha}+\frac{1}{2+(\frac{s}{s+1})^{\kappa-\upsilon}}\mu) \mu (s+1)^{\kappa-\upsilon}\eta_0\sigma_0^{2}\log\frac{1}{\eta_{s+1}\zeta}+(\frac{s'}{s'+1})^{\kappa-\upsilon}\frac{\eta_{s+1}\sigma_{s+1}^2}{\alpha}\log\frac{1}{\eta_{s+1}\zeta}\\
			=&(1-\frac{2(\frac{s'}{s'+1})^{\kappa-\upsilon}}{2+(\frac{s'}{s'+1})^{\kappa-\upsilon}})\mu^2 (s+1)^{\kappa-\upsilon}\eta_0 \sigma_0^{2}\log\frac{1}{\eta_{s+1}\zeta}+\frac{(\frac{s'}{s'+1})^{\kappa-\upsilon}}{2+(\frac{s'}{s'+1})^{\kappa-\upsilon}}\mu^2 (s+1)^{\kappa-\upsilon}\eta_0 \sigma_0^{2}\log\frac{1}{\eta_{s+1}\zeta}\\
			=&(\frac{2}{2+(\frac{s'}{s'+1})^{\kappa-\upsilon}})\mu^2 (s+1)^{\kappa-\upsilon}\eta_0 \sigma_0^{2}\log\frac{1}{\eta_{s+1}\zeta}.
			\end{aligned}
			\end{equation}
			Based on Eq. (\ref{laststep_convex_stage}) and following the idea of Eq. (\ref{stage_1_ub_p}), for $T_0 s^2-b+1\ge t\ge{T_{s,2}}$
			we get 
			\begin{align}
			\label{stage_2_ub_p}
			Pr\left(\mathfrak{E}_{s,b,t-1} \cap\left\{\left\|x_{s,b+t}-x^{s'}\right\|^{2}>(\frac{2}{2+(\frac{s'}{s'+1})^{\kappa-\upsilon}})\mu^2 (s+1)^{\kappa-\upsilon}\eta_0 \sigma_0^{2}\log\frac{1}{\eta_{s+1}\zeta}\right\}\right) \leq \mathcal{O}(\eta_s^5).
			\end{align}
			If $\mathfrak{E}_{s',T_{s',1},T_0s'^{2\upsilon}+1-T_{s',1}}$ holds, at the beginning of stage $s'+1$, we get 
			\begin{align}
			\label{nextstage_begin}
			\left\|x_{s'+1,1}-x^{s'}\right\|^{2}<  \mu^2\eta_0 \sigma_0^{2}(s'+1)^{\kappa-\upsilon}\log\frac{1}{\eta_{s'+1}\zeta}\le\delta^2,
			\end{align}
			The last inequality of Eq. (\ref{nextstage_begin}) holds because of the assumption that $s'$ is large enough. Based on Eqs. (\ref{stage_1_ub_p}) and (\ref{stage_2_ub_p}), we can calculate 
			$Pr(\mathfrak{E}_{s',T_{s',1},T_0s'^{2\upsilon}-T_{s',1}+1})$:
			\begin{equation*}
			\begin{aligned}
			&Pr(\mathfrak{E}_{s',T_{s',1},T_0s'^{2\upsilon}-T_{s',1}+1})\\
			=&1-Pr(\overline{\mathfrak{E}}_{s',T_{s',1},T_0s'^{2\upsilon}-T_{s',1}+1})\\
			=&1-Pr(\mathfrak{E}_{s',T_{s',1},T_0s'^{2\upsilon}-T_{s',1}} \cap\{\|x_{s,T_0s^{2\upsilon}+1}-x^{s'}\|^2>\mu^2\eta_0 \sigma_0^{2} (s'+1)^{\kappa-\upsilon} \})-Pr\left(\overline{\mathfrak{E}}_{s',T_{s',1},T_0s'^{2\upsilon}-T_{s',1}}\right) \\
			\ge& 1-\mathcal{O}(T_0s'^{2\upsilon}\eta_{s'}^5)-Pr\left(\overline{\mathfrak{E}}_{s',T_{s',1},0}\right)\\
			\ge& 1-\mathcal{O}(T_0s'^{2\upsilon}\eta_{s'}^5).
			\end{aligned}    
			\end{equation*}
			With a large enough $s'$, we get $\mathcal{O}(T_0s'^{2\upsilon}\eta_{s'}^5)\le s'^{-2\upsilon}$. Similarly, for stage $s'+1$, we have
			\begin{equation*}
			Pr(\mathfrak{E}_{s'+1,1,T_0{(s+1)}^{2\upsilon}}|\mathfrak{E}_{s',T_{s',1},T_0s'^{2\upsilon}-T_{s',1}+1})\ge 1-(s'+1)^{-2\upsilon}.
			\end{equation*} 
			Then for any stage $s> s'+1$, we have 
			\begin{equation*}
			Pr(\mathfrak{E}_{s,1,T_0{s}^{2\upsilon}}|\mathfrak{E}_{s-1,1,T_0{(s-1)}^{2\upsilon}})\ge 1-s^{2\upsilon}.
			\end{equation*} 
			With the definition of  $Pr(\mathfrak{E}_{s',T_{s',1},T_0s'^{2\upsilon}-T_{s',1}+1})$, $Pr(\mathfrak{E}_{s'+1,1,T_0{(s+1)}^{2\upsilon}}|\mathfrak{E}_{s',T_{s',1},T_0s'^{2\upsilon}-T_{s',1}+1})$ and $Pr(\mathfrak{E}_{s,1,T_0{s}^{2\upsilon}}|\mathfrak{E}_{s-1,1,T_0{(s-1)}^{2\upsilon}})$, if $\upsilon\ge1$ we have 
			\begin{align*}
			Pr(\mathfrak{E}_{s',T_{s',1},T_0s'^{2\upsilon}-T_{s',1}+1}\cap(\bigcap_{s=s'+1}^\infty \mathfrak{E}_{s,1,T_0s^2}))\ge\frac{\prod_{i=1}^\infty (1-\frac{1}{i^2})}{\prod_{i=1}^{s'-1}(1-\frac{1}{i^2})}\ge\frac{1}{2\prod_{i=1}^{s-1}(1-\frac{1}{i^2})}.
			\end{align*}

		\end{proof}
	\end{lemma}
	\begin{remark}
	\label{without_Lsmooth}
		According to Assumption \ref{Noise_SGD}, $\left\|\nabla F\left(x\right)\right\|$ has a fixed upper bound, although the bound might be larger than that calculated based on the condition of L-smooth of $\nabla F(x)$. With a large enough $s'$, we can still get
		\begin{equation*}
		\forall s\ge s'+1\ and\ \forall n\in [T_0 s^{2\upsilon}], \quad\|x_{s,n}-x^{s'}\|^{2}\le\tilde{\mathcal{O}}(s^{\kappa-\upsilon})
		\end{equation*}
		without the assumption of L-smooth of $\nabla F(x)$.
	\end{remark}
\newpage
\section{SSGD-SCB with \texorpdfstring{$\sigma$}{Lg}-nice \texorpdfstring{$F(x)$}{Lg}}
With an additional structural assumption on $F(x)$, i.e. $\sigma$-nice, a variant of SSGD-SCB, $\sigma$-SSGD-SCB, can achieve a zero approaching expected mismatching rate, in which $X$ is a set of $\epsilon$-optimal solution ($F(x^l)-F(x^*)\le\epsilon$). 

At the end of stage $s'$, we stop increasing the stage number and visit number of actions, and then update $x$ with the $GradOpt_G$ algorithm (\cite{hazan2016graduated}). After running $GradOpt_G$ for $\tilde{\mathcal{O}}(s'^{2\upsilon-2\kappa})$ rounds, we start the $(s'+1)^{th}$ stage of SSGD-SCB. In the following stages, SSGD-SCB is implemented as in Section 3 except that $x$ was updated by
\begin{equation}
\label{sigma_update}
    x_{I(s,n)+1}=x_{I(s,n)}-\eta_s(\tilde{g}_{I(s,n)}(x_{I(s,n)}+ \delta u)),
\end{equation}
where $u$ is randomly sampled from a unit Euclidean ball/sphere in $\mathbb{R}^{n_{x}}$, $\delta=\frac{\mbox{diam}(\mathcal{
K})\alpha_0\epsilon}{2}$, $\epsilon=\alpha\frac{2}{2+(\frac{s'}{s'+1})^{\kappa-\upsilon}}\mu^2 (s'+1)^{\kappa-\upsilon}\eta_0\sigma_0^2\log\frac{1}{\eta_{s'+1}\zeta}$ and $\alpha_{0}=\min \left\{\frac{1}{2 \overline{L} \operatorname{diam}(\mathcal{K})}, \frac{2 \sqrt{2}}{\sqrt{\alpha} \operatorname{diam}(\mathcal{K})}\right\}$.

Before conducting the theoretical analysis, we first give the definition of $\sigma$-nice and the necessary assumption. In our definition, strongly convex is measured by $\alpha$ rather than $\sigma$ in (\cite{hazan2016graduated}). In the following proof, we will use $\alpha$-strongly convex instead of $\sigma$-strongly convex.
\begin{definition}[$\sigma$-nice (\cite{hazan2016graduated})]
\label{sigmanice}
Define the $\alpha$ smooth version of $F: \mathbb{R}^{n_x} \mapsto \mathbb{R}$ as  $\hat{F}_{\delta'}(\mathbf{x})=\mathbb{E}_{\mathbf{u} \sim \mathbb{B}}[F(\mathbf{x}+\delta' \mathbf{u})]$, where $\mathbb{B}$ is an unit Euclidean ball/sphere in $\mathbb{R}^{n_x}$. With the convex set $\mathcal{K}$, a function $F: \mathcal{K} \mapsto \mathbb{R}$ is said to be $\alpha$-nice if the following two conditions hold:

(1) Centering property: For every $\delta>0$, and every $\mathbf{x}_{\delta'}^*\in\arg\min_{x\in\mathcal{K}}\hat{F}_{\delta'}(\mathbf{x})$, there exists $\mathbf{x}_{{\delta'} / 2}^{*} \in \arg \min _{x \in \mathcal{K}} \hat{F}_{{\delta'} / 2}(\mathbf{x})$, such that:
\begin{equation*}
    \left\|\mathbf{x}_{{\delta'}}^{*}-\mathbf{x}_{{\delta'} / 2}^{*}\right\| \leq \frac{{\delta'}}{2}.
\end{equation*}

(2) Local strong convexity of the smoothed function: For every ${\delta'}>0$ let $r_{{\delta'}}=3 {\delta'}$,
and denote $\mathbf{x}_{{\delta'}}^{*}=\arg \min _{\mathbf{x} \in \mathcal{K}} \hat{F}_{{\delta'}}(\mathbf{x})$, then over $\mathbb{B}_{r_{{\delta'}}}\left(\mathbf{x}_{{\delta'}}^{*}\right)$, the function $\hat{F}_{{\delta'}}(\mathbf{x})$ is $\alpha$-strongly-convex, where $\mathbb{B}_{r_{{\delta'}}}\left(\mathbf{x}_{{\delta'}}^{*}\right)$ is the Euclidean $r_{{\delta'}}$-ball/sphere in $\mathbb{R}^{n_x}$ centered at $\mathbf{x}_{{\delta'}}^{*}$.
\end{definition}

\begin{assumption}
\label{convexset}
$x^*$ defined in the reward function belongs to the convex set $\mathcal{K}$ defined in Definition \ref{sigmanice}.
\end{assumption}

\begin{lemma}
\label{sigmaS_l1}
Based on Assumption \ref{convexset}, according to Lemma \ref{boundg}, if we stop increasing the number of stage, the noise added stochastic gradient ($\tilde{g}_{I(s,n)}(x_{I(s,n)}+ \delta u)$) will be upper bounded by a fixed number, which is marked as $\overline{G}$. Then according to Assumption \ref{Noise_SGD}, we know that $F(x)$ is $L_0$-Lipschitz Continuity. For the sake of simplicity, with $\overline{L}=\max(L_0,\overline{G})$, $F(x)$ is $\overline{L}$-Lipschitz Continuity and $\tilde{g}_{I(s,n)}(x_{I(s,n)}+ \delta u)$ is $\overline{L}$ bounded. Then we have 
    \begin{equation*}
        \forall \mathbf{x} \in\mathcal{K}:\left|\hat{F}_{\delta'}(\mathbf{x})-F(\mathbf{x})\right| \leq \delta' \overline{L}.
    \end{equation*}
    \begin{proof}
    \begin{equation*}
        \begin{aligned}
        \left|\hat{F}_{\delta'}(\mathbf{x})-F(\mathbf{x})\right| &=\left|\mathbf{E}_{u \sim \mathbb{B}}[F(\mathbf{x}+\delta' \mathbf{u})]-F(\mathbf{x})\right| \\
        & \leq \mathbf{E}_{\mathbf{u} \sim \mathbb{B}}[|F(\mathbf{x}+\delta' \mathbf{u})-F(\mathbf{x})|] \\
        & \leq \mathbf{E}_{\mathbf{u} \sim \mathbb{B}}[\overline{L}\|\delta' \mathbf{u}\|] \\
        & \leq \overline{L} \delta'
        \end{aligned}
    \end{equation*}
    \end{proof}
\end{lemma}
\begin{lemma}
\label{sigmaS_l2}
After running $GradOpt_G$ for $\tilde{\mathcal{O}}(s'^{2\upsilon-2\kappa})$ rounds, following Lemma 5.1 and Theorem 5.1 in (\cite{hazan2016graduated}), as $\delta'=\frac{\mbox{diam}(\mathcal{
K})\alpha_0\epsilon}{2}$, $\alpha_{0}=\min \left\{\frac{1}{2 \overline{L} \operatorname{diam}(\mathcal{K})}, \frac{2 \sqrt{2}}{\sqrt{\alpha} \operatorname{diam}(\mathcal{K})}\right\}$, and $\epsilon=\alpha\frac{2}{2+(\frac{s'}{s'+1})^{\kappa-\upsilon}}\mu^2 (s'+1)^{\kappa-\upsilon}\eta_0\sigma_0^2\log\frac{1}{\eta_{s'+1}\zeta}<1$, with probability at least $1-\frac{\zeta}{2}$, the output $x_{s'}^G$ has the following two properties:

(1) The smoothed version $\hat{F}_{\delta'}$ is $\alpha$-strongly convex over $\mathcal{K}_m=\mathcal{K} \cap B\left(x_{s'}^G, 1.5 {\delta'}\right)$, and $x_{\delta'}^*\in\mathcal{K}_m$.

(2)$\hat{F}_{\delta'}(x_{s'}^G)-\hat{F}(x^*_{\delta'})\le\alpha{\delta'}^2/8\le\frac{\alpha}{2}\frac{2}{2+(\frac{s'}{s'+1})^{\kappa-\upsilon}}\mu^2 (s'+1)^{\kappa-\upsilon}\eta_0\sigma_0^2\log\frac{1}{\eta_{s'+1}\zeta}$.
\end{lemma}

Based on the Assumption \ref{convexset}, if $F(x)$ is a $\sigma$-nice function. We have the following conclusion:
\begin{theorem}
At the end of stage $s'$, if we stop increasing the stage number and visit number of actions, and then update $x$ with the algorithm of $GradOpt_G$ (\cite{hazan2016graduated}). After running $GradOpt_G$ for $\tilde{\Theta}(s'^{2\upsilon-2\kappa})$ rounds, we stop $GradOpt_G$, start the $(s'+1)^{th}$ stage of SSGD-SCB and then update $x$ with Eq. (\ref{sigma_update}). Then with probability at least $1-\zeta$, we can achieve a zero approaching expected mismatching rate, in which $X$ is a set of $\mathcal{O}(s'^{\kappa-\upsilon})$-optimal solution ($F(x^l)-F(x^*)\le\mathcal{O}(s'^{\kappa-\upsilon})$). 

\begin{proof}
Based on Lemma \ref{sigmaS_l2} and the property of strongly convex, with probability at least $1-\frac{\zeta}{2}$, we have 
\begin{equation*}
    \frac{\alpha}{2}\|x_{s'}^G-x^*_{\delta'}\|^2\le\hat{F}_{\delta'}(x_{s'}^G)-\hat{F}(x^*_{\delta'})\le\alpha{\delta'}^2/8\le\frac{\alpha}{2}\frac{2}{2+(\frac{s'}{s'+1})^{\kappa-\upsilon}}\mu^2 (s'+1)^{\kappa-\upsilon}\eta_0\sigma_0^2\log\frac{1}{\eta_{s'+1}\zeta}.
\end{equation*}

When we start the $(s'+1)^{th}$ stage of SSGD-SCB, with the added noise in Eq. (\ref{sigma_update}), the smoothed function $\hat{F}_{\delta'}(x)$ is $\alpha$-strongly convex when $\|x_{s'}^G-x^*_{\delta'}\|^{2}\le \mathcal{O}(s'^{\kappa-\upsilon})$, where $x^*_{\delta'}$ is the corresponding local minimizer in the strongly convex region. Then following Lemma \ref{local minimum} and Remark \ref{without_Lsmooth}, conditioned on (1) 
and (2) in Lemma \ref{sigmaS_l2}, with probability at least $1-\frac{\zeta}{2}$, for any $s>s'$, we have $\|x_{s,n}-x_{\delta'}^*\|^{2}\le\frac{2}{2+(\frac{s'}{s'+1})^{\kappa-\upsilon}}\mu^2 s^{\kappa-\upsilon}\eta_0\sigma_0^2\log\frac{1}{\eta_{s}\zeta}$. Based on Assumption \ref{Noise_SGD}, the gradient of $F(x)$ is bounded, thus we have
\begin{equation}
\label{sigmaeq1}
    |\hat{F}_{\delta'}(x_{s,n})-\hat{F}(x^*_{\delta'})|\le\mathcal{O}(\sqrt{\frac{2}{2+(\frac{s'}{s'+1})^{\kappa-\upsilon}}\mu^2 s^{\kappa-\upsilon}\eta_0\sigma_0^2\log\frac{1}{\eta_{s}\zeta}}).
\end{equation}

Following the proof of Theorem \ref{pcbsgdconclusion}, with probability at least $1-\zeta$, we can get a zero approaching expected local optimal mismatching rate. 

Besides, following Lemma \ref{sigmaS_l1}, with $\delta'=\frac{\mbox{diam}(\mathcal{
K})\alpha_0\epsilon}{2}$, $\alpha_{0}=\min \left\{\frac{1}{2 \overline{L} \operatorname{diam}(\mathcal{K})}, \frac{2 \sqrt{2}}{\sqrt{\alpha} \operatorname{diam}(\mathcal{K})}\right\}$ and $\epsilon=\alpha\frac{2}{2+(\frac{s'}{s'+1})^{\kappa-\upsilon}}\mu^2 (s'+1)^{\kappa-\upsilon}\eta_0\sigma_0^2\log\frac{1}{\eta_{s'+1}\zeta}<1$, we have 
\begin{equation}
\label{sigmaeq2}
\begin{aligned}
    |\hat{F}_{\delta'}(x^*)-F(x^*)|&\le{\delta'}\overline{L}\le\frac{\alpha}{4}\frac{2}{2+(\frac{s'}{s'+1})^{\kappa-\upsilon}}\mu^2 (s'+1)^{\kappa-\upsilon}\eta_0\sigma_0^2\log\frac{1}{\eta_{s'+1}\zeta}.\\
    |\hat{F}_{\delta'}(x_{s'}^G)-F(x_{s'}^G)|&\le{\delta'}\overline{L}\le\frac{\alpha}{4}\frac{2}{2+(\frac{s'}{s'+1})^{\kappa-\upsilon}}\mu^2 (s'+1)^{\kappa-\upsilon}\eta_0\sigma_0^2\log\frac{1}{\eta_{s'+1}\zeta}.
\end{aligned}
\end{equation}

Based on Eqs. (\ref{sigmaeq1}) and (\ref{sigmaeq2}), with a large enough $s'$, for any $s>s'$, we have 
\begin{equation*}
    |F(x_{s,n})-F(x^*)|\le\tilde{\mathcal{O}}((s'+1)^{\kappa-\upsilon}).
\end{equation*}
\end{proof}
\end{theorem}
\newpage
\section{SSGD-SCB with Strongly Convex \texorpdfstring{$F(x)$}{Lg}}
When $F(x)$ is strongly convex, we propose a simplified version of SSGD-SCB, called SGD-SCB. It shares the same principles as SSGD-SCB except for three differences below. First, each stage contains only one round in SGD-SCB. That is, there are no stages but only rounds. Second, noise $\mathcal{N}_t$ is no longer needed, thus $x_{t+1}$ can be updated by
\begin{equation}\label{eq:update_x_cb_sgd}
    x_{t+1}=x_{t}-\eta_t\cdot \tilde{g}_{t}(x_{t}).
\end{equation}
Third, according to the convergence analysis of SGD-SCB we will show, we can start $t=1+\Delta s$ in the first round and all the way move up to $T+\Delta s$, where $\Delta s$ is the minimum positive integer greater than $2^{\frac{1}{\upsilon}}-1$ (for $\upsilon<1$ in SGD-SCB). 

In the following, we first summarize SGD-SCB in Algorithm \ref{SGD-SCB} and then show in Theorem \ref{convergence_cb_sgd} that the expected cumulative regret of SGD-SCB averaged over $T$ will converge to zero.  
\begin{algorithm}
\caption{SGD-SCB}
\label{SGD-SCB}
\begin{algorithmic}[1]
   \STATE {\bfseries Input:} Initial point $x_{1+\Delta s}$, the number of rounds $T_0$ in the $1^{th}$ phase, decay factor $\upsilon$, $\omega$ and $\kappa$,  exploration parameter $\beta$ and $C$, $\mathcal{N}_0=0$ and initial learning rate $\eta_0$.
   \STATE{Calculate $\Delta s=\lceil2^{\frac{1}{\upsilon}}-1\rceil$}
   \FOR{$t=1+\Delta s$ {\bfseries to} $T+\Delta s$}
   \STATE\slash\slash \textbf{Action Selection}
	\STATE Context $d_t$ is revealed
	\STATE $c_t=\operatorname*{argmax}_k f_k(d_{t};x_{t})$
	\FOR{$k=1$ {\bfseries to} $K$}
	\STATE	Calculate $\pi_{S}(k|d_{t},\mathcal{H}_{t-1})$
	\ENDFOR
	\STATE$a_{t}\sim \pi_{S}(\cdot|d_{t},\mathcal{H}_{t-1})$\\
	\STATE\slash\slash \textbf{Backpropagation}\\
	\STATE$r_{d_{t},a_{t}}$ is revealed\\
    \STATE$N_{t+1,a,c_t}=N_{t,a,c_t}+1$\\
    \STATE Update $x_{t+1}$ with Eq. (\ref{eq:update_x_cb_sgd})
   \ENDFOR
\end{algorithmic}
\end{algorithm}

	\begin{lemma} 
		\label{convex_convergence}
		Based on Lemma \ref{unbiased_g}, Assumption \ref{Noise_SGD} and the strong-convexity of $F(x)$,
		with a suitable $\eta_0$, we have $\mathbb{E}\|x_{t+\Delta s}-x^*\|^2\leq \mathcal{O}(t^{\kappa-\upsilon})$ in SGD-SCB, where $x^*$ is the globally optimal solution of $F(x)$.
	\end{lemma}
	\begin{proof}
		By the definition of $\alpha$-strongly convex, we have 
		\begin{align*}
		&F(x^*)-F(x_{t+\Delta s})\ge \langle \nabla F(x_{t+\Delta s}) \,, x^*-x_{t+\Delta s}\rangle +\frac{\alpha}{2}\|x^*-x_{t+\Delta s}^2\|^2,\\
		&F(x_{t+\Delta s})-F(x^*)\ge \langle \nabla F(x^*) \,, x_{t+\Delta s}-x^*\rangle +\frac{\alpha}{2}\|x^*-x_{t+\Delta s}^2\|^2.
		\end{align*}
		We then have 
		\begin{equation}
		\label{strongconvex_key_ineq}
		\langle \nabla F(x_{t+\Delta s})-\nabla F(x^*) \,, x_{t+\Delta s}-x^*\rangle=\langle \nabla F(x_{t+\Delta s})\,, x_{t+\Delta s}-x^*\rangle \ge \alpha\|x_{t+\Delta s}-x^*\|^2.
		\end{equation}
		According to Lemma \ref{unbiased_g} and Assumption \ref{Noise_SGD}, let $G=\frac{KM}{0.05}$. 
		We have 
		\begin{equation*}
		\begin{aligned}
		&\mathbb{E}\big[\tilde{g}_{{t+\Delta s},a_{t+\Delta s}}(x_{t+\Delta s})]=\nabla F(x_{t+\Delta s}),\\
		&\mathbb{E}\|\tilde{g}_{{t+\Delta s},a_{t+\Delta s}}(x_{t+\Delta s})\|^2\le(t+\Delta s)^{\kappa}G^2.
		\end{aligned}
		\end{equation*} 
		Then we have
		\begin{align*}
		&\mathbb{E}( \| x_{t+\Delta s+1}-x^*\|^2)\\
		=&\mathbb{E}( \| x_{t+\Delta s}-\eta_{t+\Delta s}\tilde{g}_{t,a_t}(x_{t+\Delta s})-x^*\|^2)\\
		\le&\mathbb{E}( \| x_{t+\Delta s}-x^*\|^2)-2\eta_{t+\Delta s}\mathbb{E}\langle \nabla F(x_{t+\Delta s})\,, x_{t+\Delta s}-x^*\rangle+\eta_{t+\Delta s}^2 \mathbb{E}(\|\tilde{g}_{{t+\Delta s},a_{t+\Delta s}}(x_{t+\Delta s})\|^2)\\
		\le&(1-2\alpha\eta_{t+\Delta s})\mathbb{E}( \| x_{t+\Delta s}-x^*\|^2)+\eta_{t+\Delta s}^2(t+\Delta s)^{\kappa} G^2.
		\end{align*}
		Let 
		\begin{equation*}
		C_g=max\{(\Delta s+1)^{\upsilon-\kappa}\|x_{1+\Delta s}-x^*\|^2,G^2/\alpha^2\}.
		\end{equation*}
		We have
		\begin{equation*}
		\mathbb{E}(\|x_{1+\Delta s}-x^*\|^2)\le\frac{C_g}{(\Delta s+1)^{\upsilon-\kappa}}.
		\end{equation*}
		With $\eta_0=\frac{1}{\alpha}$, under SGD-SCB, it is easy to see that 
		\begin{align*}
		&\mathbb{E}(\|x_{t+\Delta s+1}-x^*\|^2)\\
		\le&(1-\frac{2}{(t+\Delta s)^\upsilon})\mathbb{E}(\|x_{t+\Delta s}-x^*\|^2)+\frac{(t+\Delta s)^{\kappa}}{\alpha^2(t+\Delta s)^{2\upsilon}}G^2\\
		\le&(1-\frac{2}{(t+\Delta s)^{\upsilon}})\frac{C_g}{(t+\Delta s)^{\upsilon-\kappa}}+\frac{(t+\Delta s)^{\kappa}}{(t+\Delta s)^{2\upsilon}}C_g\\
		\le&(1-\frac{1}{(t+\Delta s)^{\upsilon}})\frac{C_g}{(t+\Delta s)^{\upsilon-\kappa}}\\
		\le&(1-\frac{1}{t+\Delta s})\frac{C_g}{(t+\Delta s)^{\upsilon-\kappa}}\\
		\le&(1-\frac{1}{t+\Delta s+1})\frac{C_g}{(t+\Delta s)^{\upsilon-\kappa}}\\
		\le&(1-\frac{1}{t+\Delta s+1})^{\upsilon-\kappa}\frac{C_g}{(t+\Delta s)^{\upsilon-\kappa}}\\
		=&\frac{C_g}{(t+\Delta s+1)^{\upsilon-\kappa}}.
		\end{align*}
		Therefore, we have 
		\begin{equation*}
		\mathbb{E}(\|x_{t+\Delta s}-x^*\|^2)\le\frac{C_g}{(t+\Delta s)^{\upsilon-\kappa}}\le\frac{C_g}{t^{\upsilon-\kappa}}.
		\end{equation*}
	\end{proof}
	\begin{lemma} 
		\label{boundoutput_CB_SGD}
		In the $t^{th}$ phase of SGD-SCB, based on Assumption \ref{boundedslope}, if 
		\begin{equation*}
		\mathbb{E}(\|x_{t+\Delta s}-x^*\|^2)\leq\mathcal{O}(t^{\kappa-\upsilon}),
		\end{equation*}
		with context $d$, for any action $a$, the expectation and variance of $f_a(x_{t+\Delta s};d)$ can be bounded by
		\begin{equation*}
		\begin{aligned}
		&|\mathbb{E}\left[f_a(x_{t+\Delta s};d)-f_a(x^*;d)\,|\,d\right]|
		\le \mathcal{O}(t^{-0.25}),\\
		&\mbox{VAR}\left[f_a(x_{t+\Delta s};d)\,|\,d\right]\le \mathcal{O}(t^{-0.5}).
		\end{aligned}
		\end{equation*}
		\begin{proof}
			The proof is similar to that of Lemma \ref{boundoutput}.
		\end{proof}
	\end{lemma}
	\begin{theorem} 
		\label{convergence_cb_sgd}
		Let $T$ be the total number of rounds in SGD-SCB and $a_t$ be the action selected by SGD-SCB at round $t$. Under Assumptions \ref{Noise_SGD}, \ref{banditproof}, \ref{boundedslope}, and the strong-convexity of $F(x)$, if
		$\Delta s=\lceil2^{\frac{1}{\upsilon}}-1\rceil$, the expected $x^*$-mismatching rate of SGD-SCB is bounded by
		\begin{equation*}
		\begin{aligned}
		&\frac{1}{T}\sum_{t=1}^{T}{\mathbb{E}[r_{d_{t+\Delta s},a^*_{d_{t+\Delta s}}}-r_{d_{t+\Delta s},a_{t+\Delta s}}]}\\
		\le&\mathcal{O}(T^{max(-\frac{\kappa}{2},-2\beta+\kappa-\upsilon+1,2\beta-1,-\omega)}),
		\end{aligned}
		\end{equation*}
		which goes to $0$ when $T \rightarrow \infty$.
		\begin{proof}
			
			For the sake of simplicity, let $\Delta_{t,k}=r_{d_{t+\Delta s},a^*_{d_{t+\Delta s}}}-r_{d_{t+\Delta s},k}$.
			$\forall a\in\{1,2,\dots,K\}$ and $\forall t\in\{1,2,\dots,T\}$, we have $\Delta_{t,a}\le1$ and $\Delta_{t,a^*_{d_{t+\Delta s}}}=0$. Let $K^-_t=\{1,2,\dots,K\}\backslash a_{d_{t+\Delta s}}^*$ and define the policy of SGD-SCB as $\pi_{C}(\cdot\,|\,d_t,t)$. We have 
			\begin{equation*}
			\begin{aligned}
			&\sum_{t=1}^{T}\frac{\mathbb{E}[r_{d_{t+\Delta s},a^*_{d_{t+\Delta s}}}-r_{d_{t+\Delta s},a_{t+\Delta s}}]}{T}\\
			=&\frac{1}{T}\mathbb{E}\big[\sum_{t=1}^T\sum_{k=1}^K\Delta_{t,k}\mathbf {1}(a_{t+\Delta s}= k)\big]\\
			\le&\frac{1}{T}\sum_{t=1}^T\sum_{k\in K^-_t}Pr(a_{t+\Delta s}=k\,|\,k\ne a^*_{d_{t+\Delta s}}).
			\end{aligned}
			\end{equation*}
			We further have 
			\begin{equation*}
			\begin{aligned}
			&\Pr(a_{t+\Delta s}=k\,|\,k\ne a_{d_{t+\Delta s}}^*)\\
			=&\Pr(a_{t+\Delta s}=k,a_{t+\Delta s}\ne \arg\max_i U_{t+\Delta s,1,i}\,|\,k\ne a_{d_{t+\Delta s}}^*)\\
			&+\Pr(a_{t+\Delta s}=k,a_{t+\Delta s}= \arg\max_i U_{t+\Delta s,1,i}\,|\,k\ne a_{d_{t+\Delta s}}^*).\\
			\end{aligned}
			\end{equation*}
			Similar to the proof of Eq. (\ref{final_second_part}) in Theorem \ref{converagerate}, we have 
			\begin{equation*}
			\Pr(a_{t+\Delta s}=k,a_{t+\Delta s}\ne \arg\max_i U_{t+\Delta s,1,i}\,|\,k\ne a_{d_{t+\Delta s}}^*)\le\mathcal{O}((t+\Delta s)^{max(-\frac{\kappa}{2},-\omega)})\le\mathcal{O}(t^{max(-\frac{\kappa}{2},-\omega)}).
			\end{equation*}
			For the second term, we have 
			\begin{equation}
			\label{CB_SGD_3items}
			\begin{aligned}
			&\Pr(a_{t+\Delta s}=k,a_{t+\Delta s}= \arg\max_i U_{t+\Delta s,1,i}\,|\,k\ne a_{d_{t+\Delta s}}^*)\\
			\le&Pr(f_k(d_{t+\Delta s};x_{t+\Delta s})+e_{{t+\Delta s},1,k}> f_{a^*_{d_{t+\Delta s}}}(d_{t+\Delta s};x_{t+\Delta s})+e_{t+\Delta s,1,a^*_{d_{t+\Delta s}}})\\
			\le&Pr(f_k(d_{t+\Delta s};x_{t+\Delta s})-\mathbb{E}f_k(d_{t+\Delta s};x_{t+\Delta s})> e_{t+\Delta s,1,k})\\
			&+Pr(2e_{t+\Delta s,1,k}> \mathbb{E}f_{a^*_{d_{t+\Delta s}}}(d_{t+\Delta s};x_{t+\Delta s})-\mathbb{E}f_k(d_{t+\Delta s};x_{t+\Delta s}))\\
			&+Pr(f_{a^*_{d_{t+\Delta s}}}(d_{t+\Delta s};x_{t+\Delta s})-\mathbb{E}f_{a^*_{d_{t+\Delta s}}}(d_{t+\Delta s};x_{t+\Delta s})< -e_{t+\Delta s,1,a^*_{d_{t+\Delta s}}}).
			\end{aligned}
			\end{equation}
			For the first and third terms in Eq. (\ref{CB_SGD_3items}),
			\begin{equation*}
			\begin{aligned}
			&Pr(f_k(d_{t+\Delta s};x_{t+\Delta s})-\mathbb{E}f_k(d_{t+\Delta s};x_{t+\Delta s})> e_{t+\Delta s,1,k})\\
			=&\mathbb{E}_{d_{t+\Delta s}\sim\mathcal{D}}Pr(f_k(d_{t+\Delta s};x_{t+\Delta s})-\mathbb{E}f_k(d_{t+\Delta s};x_{t+\Delta s})\ge e_{t+\Delta s,1,k}\,|\,d_{t+\Delta s})\\
			\le&\frac{\mathcal{O}(t^{\kappa-\upsilon})N_{t+\Delta s,k,c_{t+\Delta s,1}}}{(\sum_{k'=1}^KN_{t+\Delta s,k',c_{t+\Delta s,1}})^{2\beta}}\le\mathcal{O}(t^{\kappa-\upsilon})T^{1-2\beta}\\
			\le&\mathcal{O}(T^{1-2\beta}t^{\kappa-\upsilon}).
			\end{aligned}    
			\end{equation*}    
			
			For the second term in Eq. (\ref{CB_SGD_3items}), with $\Delta$ defined in Eq. (\ref{Delta}),  
			we have 
			\begin{equation*}
			\begin{aligned}
			&Pr(2e_{t+\Delta s,1,k}> \mathbb{E}f_{a^*_{t+\Delta s}}(d_{t+\Delta s};x_{t+\Delta s})-\mathbb{E}f_k(d_{t+\Delta s};x_{t+\Delta s}))\\
			\le&\mathbb{E}_{d_{t+\Delta s}\sim\mathcal{D}}Pr(2e_{t+\Delta s,1,k}> \mathbb{E}f_{a^*_{t+\Delta s}}(d_{t+\Delta s};x_{t+\Delta s})-\mathbb{E}f_k(d_{t+\Delta s};x_{t+\Delta s})\,|\,d_{t+\Delta s})\\
			\le&\mathbb{E}_{d_{t+\Delta s}\sim\mathcal{D}}Pr(2e_{t+\Delta s,1,k}> \mathbb{E}f_{a^*_{t+\Delta s}}(d_{t+\Delta s};x^*)-\mathbb{E}f_k(d_{t+\Delta s};x^*)-\mathcal{O}(t^{\frac{\kappa-\upsilon}{2}})\,|\,d_{t+\Delta s})\\
			\le&\mathbb{E}_{d_{t+\Delta s}\sim\mathcal{D}}Pr(2C\frac{(\sum_{k'=1}^KN_{t+\Delta s,k',c_{t+\Delta s,1}})^\beta}{\sqrt{N_{t+\Delta s,k,c_{t+\Delta s,1}}}}> \Delta-\mathcal{O}(t^{\frac{\kappa-\upsilon}{2}})\,|\,d_{t+\Delta s})\\
			=&\mathbb{E}_{d_{t+\Delta s}\sim\mathcal{D}}Pr(N_{t+\Delta s,a,c_{t+\Delta s,1}}<\frac{4C^2(\sum_{k'=1}^KN_{t+\Delta s,k',c_{t+\Delta s,1}})^{2\beta}}{(\Delta-\mathcal{O}(t^{\frac{\kappa-\upsilon}{2}}))^2}\,|\,d_{t+\Delta s})\\
			\le&\mathbb{E}_{d_{t+\Delta s}\sim\mathcal{D}}Pr(N_{t+\Delta s,a,c_{t+\Delta s,1}}<\frac{\mathcal{O}(T^{2\beta})}{(\Delta-\mathcal{O}(t^{\frac{\kappa-\upsilon}{2}}))^2}\,|\,d_t).
			\end{aligned}
			\end{equation*}
			With a large enough $t$, we have $\Delta-\mathcal{O}(t^{\frac{\kappa-\upsilon}{2}})=\frac{\Delta}{2}$. Therefore, when $N_{t+\Delta s,a,c_{t+\Delta s,1}}>\mathcal{O}(T^{2\beta})$,
			\begin{equation*}
			\label{bounda}
			Pr(2e_{t+\Delta s,1,k}> \mathbb{E}f_{a^*_{t+\Delta s}}(d_{t+\Delta s};x_{t+\Delta s})-\mathbb{E}f_k(d_{t+\Delta s};x_{t+\Delta s}))=0,
			\end{equation*}
			which yields
			\begin{equation*}
			Pr(a_{t+\Delta s}=k,a_{t+\Delta s}= \arg\max_i U_{t+\Delta s,1,i}|k\ne a_{t+\Delta s}^*)\le\mathcal{O}(T^{1-2\beta}t^{\kappa-\upsilon}).
			\end{equation*}
			Therefore, we have 
			\begin{equation*}
			\sum_{t=1}^{T}\frac{\mathbb{E}[r_{d_{t+\Delta s},a^*_{d_{t+\Delta s}}}-r_{d_{t+\Delta s},a_{t+\Delta s}}]}{T}\\
			\le\mathcal{O}(T^{max(-\frac{\kappa}{2},-2\beta+\kappa-\upsilon+1,2\beta-1,-\omega)}).
			\end{equation*}
		\end{proof}
	\end{theorem} 	
	
	
	\newpage
	\bibliography{appendix}